\documentclass[onefignum,onetabnum]{siamart190516} 
\usepackage{subcaption}
\usepackage{forest}
\usepackage{bbm}
\usetikzlibrary{shapes,arrows,positioning,automata}

\tikzset{
  treenode/.style = {shape=rectangle, rounded corners,
                     draw, align=center,
                     top color=white,
                     bottom color=white},
  root/.style     = {treenode, font=\Large,
                     bottom color=red!30},
  env/.style      = {treenode, font=\ttfamily\normalsize},
  dummy/.style    = {circle,draw}
}
\usepackage{caption}
\usepackage{threeparttable}
\usepackage[usestackEOL]{stackengine}
\usepackage{xcolor}
\usepackage{enumitem}

\usepackage{lipsum}
\usepackage{amsfonts}
\usepackage{graphicx}
\usepackage{epstopdf}
\usepackage{algorithmic}
\ifpdf
  \DeclareGraphicsExtensions{.eps,.pdf,.png,.jpg}
\else
  \DeclareGraphicsExtensions{.eps}
\fi

\newsiamremark{remark}{Remark}
\newsiamremark{hypothesis}{Hypothesis}
\crefname{hypothesis}{Hypothesis}{Hypotheses}
\newsiamthm{claim}{Claim}

\headers{Bayesian Inference with L-HNNs}{S. L. N. Dhulipala et al.} 

\title{Bayesian Inference with Latent Hamiltonian Neural Networks\thanks{Code available at: \href{https://github.com/IdahoLabResearch/BIhNNs}{https://github.com/IdahoLabResearch/BIhNNs}. \funding{This research is supported through the INL Laboratory Directed Research \& Development (LDRD) Program under DOE Idaho Operations Office Contract DE-AC07-05ID14517.}}}

\author{Somayajulu L. N. Dhulipala\thanks{Idaho National Laboratory, Idaho Falls, ID 83415, USA (\email{Som.Dhulipala@inl.gov}, \email{Yifeng.Che@inl.gov}).}
\and Yifeng Che\footnotemark[2]
\and Michael D. Shields\thanks{Johns Hopkins University, Baltimore, MD 21218, USA (\email{michael.shields@jhu.edu}).}}

\usepackage{amsopn}

\ifpdf
\hypersetup{
  pdftitle={Bayesian Inference with Latent Hamiltonian Neural Networks},
  pdfauthor={S. L. N. Dhulipala et al.}
}
\fi

\begin{document}

\maketitle

\begin{abstract}
  When sampling for Bayesian inference, one popular approach is to use Hamiltonian Monte Carlo (HMC) and specifically the No-U-Turn Sampler (NUTS) which automatically decides the end time of the Hamiltonian trajectory. However, HMC and NUTS can require numerous numerical gradients of the target density, and can prove slow in practice. We propose Hamiltonian neural networks (HNNs) with HMC and NUTS for solving Bayesian inference problems. Once trained, HNNs do not require numerical gradients of the target density during sampling. Moreover, they satisfy important properties such as perfect time reversibility and Hamiltonian conservation, making them well-suited for use within HMC and NUTS because stationarity can be shown. We also propose an HNN extension called latent HNNs (L-HNNs), which are capable of predicting latent variable outputs. Compared to HNNs, L-HNNs offer improved expressivity and reduced integration errors. Finally, we employ L-HNNs in NUTS with an online error monitoring scheme to prevent sample degeneracy in regions of low probability density. We demonstrate L-HNNs in NUTS with online error monitoring on several examples involving complex, heavy-tailed, and high-local-curvature probability densities. Overall, L-HNNs in NUTS with online error monitoring satisfactorily inferred these probability densities. Compared to traditional NUTS, L-HNNs in NUTS with online error monitoring required 1--2 orders of magnitude fewer numerical gradients of the target density and improved the effective sample size (ESS) per gradient by an order of magnitude.
\end{abstract}



\section{Introduction}

Bayesian inference is a principal method for parameter calibration and uncertainty quantification in many scientific and engineering fields. Since closed-form solutions usually do not exist for practical Bayesian inference problems, Markov chain Monte Carlo (MCMC) algorithms are employed to sample from the target probability density. While random-walk MCMC algorithms are popular, they have poor scalability with the number of dimensions, and can feature large serial correlations between the samples. Therefore, algorithms under the Hamiltonian Monte Carlo (HMC) \cite{Neal2011a} framework (including the No-U-Turn Sampler [NUTS] \cite{Hoffman2014}) are often employed. HMC methods simulate Hamiltonian dynamics to propose new samples, instead of drawing them randomly, which can drastically improve the sampling acceptance rate. However, the HMC class of algorithms can require numerous numerical gradients of the target density, which are computationally expensive, especially if the likelihood function involves iterating over a large dataset or calling a costly computational model. This paper proposes latent Hamiltonian neural networks (HNNs) \cite{Greydanus2019a} with HMC and NUTS for solving Bayesian inference problems. Once trained, latent HNNs do not require evaluating numerical gradients of the target density while sampling, and can therefore lead to computational gains over traditional HMC and NUTS.

\subsection{Literature review}

Owing to the high computational cost of Bayesian inference, machine learning have garnered considerable interest to accelerate this process. For example, Che et al. \cite{Che2021a} and Dhulipala et al. \cite{Dhulipala2022a} utilized Gaussian processes to accelerate random-walk MCMC by replacing computationally expensive models for the likelihood evaluation. Gradient-based inference methods (e.g., Langevin dynamics sampling and HMC) have better scalability than random-walk MCMC, and machine learning techniques are being employed to accelerate these methods as well. For example, Gu et al. \cite{Gu2020a} trained neural networks that take as inputs the current position and gradient of the log-likelihood and predict the new position in the Metropolis-adjusted Langevin algorithm (MALA). Li et al. \cite{Li2019a} proposed neural networks that learn the log-likelihood, and applied neural network gradients instead of the target density gradients in an HMC. Levy et al. \cite{Levy2017a} devised neural networks that take the log-likelihood as the input and predict the transition kernel in an HMC, which then showed improved performance over traditional HMC for complex cases such as multi-modal densities. One thing all these efforts for accelerating Bayesian inference have in common is the use of purely data-driven machine learning models. Moreover, although neural-network-aided HMC has shown improved performance over traditional HMC, Hoffman and Gelman \cite{Hoffman2014} demonstrated NUTS to be at least an order of magnitude more efficient than conventional HMC (in terms of effective sample size [ESS] \cite{Vehtari2021a} per gradient), even without relying on machine learning. The efficiency of NUTS over HMC, MALA, and random-walk MCMC has also been demonstrated in many other studies (e.g., \cite{Nishio2019a}).


To efficiently predict the behavior of dynamical systems, recent studies have investigated physics-based neural networks. Greydanus et al. \cite{Greydanus2019a} proposed HNNs for learning Hamiltonian systems in an unsupervised fashion through a loss function that accounts for Hamilton's equations. HNNs have interesting properties such as perfect time reversibility and Hamiltonian conservation, and thus proved superior in performance compared to data-driven neural networks and black-box ordinary differential equation (ODE) solvers such as neural ODEs \cite{Tong2021a}. More recently, newer architectures to learn Hamiltonian systems have been proposed. Physics-informed HNNs \cite{Mattheakis2022a}, Symplectic ODE-Net \cite{Zhong2019a}, and symplectic neural networks \cite{Jin2020a} are examples of such architectures. Efforts along these lines, however, were all driven from the standpoint of efficiently identifying and learning dynamical systems, rather than solving Bayesian inference problems. As such, the example applications that these studies have demonstrated correspond to classical mechanics problems such as the behavior of a n-pendulum, the gravitational attraction between bodies (e.g., planets), and the dynamics of predator-prey systems, etc.  

\subsection{Contributions of this paper}

We propose to use HNNs for Bayesian inference problems. As summarized in Figure \ref{HNNs_HMC_NUTS}, once trained, HNNs do not require numerical gradients of the target density. As discussed earlier, HNNs satisfy important properties such as perfect time reversibility and Hamiltonian conservation, which make them amenable for HMC sampling, since stationarity, an important condition satisfied by MCMC samplers, can be shown. We also propose an HNN extension called latent HNNs (L-HNNs), which use latent variable outputs to improve the expressivity of HNNs. L-HNNs have reduced integration errors when simulating the Hamiltonian dynamics, and show better performance than HNNs when sampling for Bayesian inference. Moreover, given that the performance of NUTS has generally proven superior to that of HMC, MALA, and random-walk MCMC, we propose to use L-HNNs in NUTS with an online error monitoring scheme. This error monitoring scheme reverts to using numerical gradients of the target density for a few samples whenever L-HNN integration errors are large, which may be the result of inadequate training data in certain regions of the uncertainty space. Overall, the online error monitoring scheme prevents sampling degeneracy when L-HNNs are used in NUTS. In summary, the novel contributions of this paper are:
\begin{itemize}[leftmargin=*]
    \item Use of HNNs for Bayesian inference problems without requiring numerous numerical gradients of the target density.
    \item Development of latent variable outputs in HNNs (L-HNNs) for improved expressivity and reduced integration errors when simulating Hamiltonian dynamics.
    \item Use of L-HNNs in NUTS with an online error monitoring scheme to prevent sampling degeneracy.
\end{itemize}

\begin{figure}[htbp]
\centering
\includegraphics[scale=0.4]{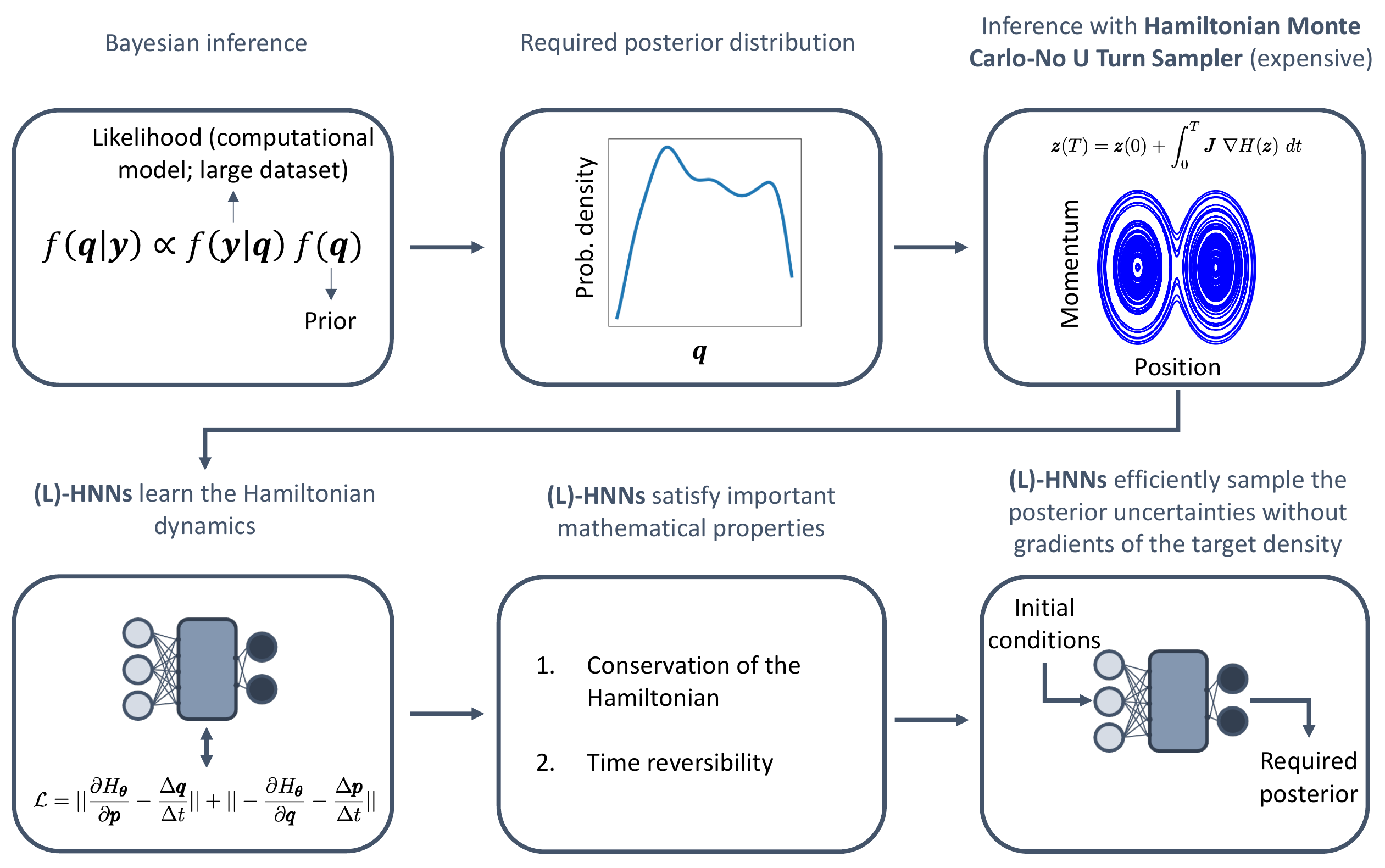} 
\caption{Schematic of the use of (L)-HNNs for efficient Bayesian inference. Once trained, (L)-HNNs do not require numerical gradients of the posterior distribution\textemdash which are expensive to compute\textemdash for sampling.}
\label{HNNs_HMC_NUTS}
\end{figure}

\section{Preliminaries}

This section provides a brief background on Hamiltonian dynamics simulation and HMC sampling.

\subsection{Simulating Hamiltonian dynamics}

The basics of Hamiltonian dynamics and its numerical simulation using time integration schemes have been discussed thoroughly (e.g., \cite{Neal2011a, Mann2018}). If $\pmb{q}$ and $\pmb{p}$ represent the position and momentum vectors, respectively, the Hamiltonian of a system is defined as:
\begin{equation}
    \label{eqn:HD_1}
    H(\pmb{z} = \{\pmb{q},\pmb{p}\}) = U(\pmb{q}) + K(\pmb{p})
\end{equation}
where $\pmb{z}$ is a concatenated vector of the vectors $\pmb{q}$ and $\pmb{p}$, $U(.)$ is the potential energy, and $K(.)$ is the kinetic energy. Hamilton's equations are given by:
\begin{equation}
    \label{eqn:HD_2}
    \frac{d\pmb{z}}{dt} = \pmb{J}~\nabla H(\pmb{z})
\end{equation}
where $\nabla$ is the gradient operator, and if $d$ is the dimensionality of the system, $\pmb{J}$ is a $2d \times 2d$ matrix given by:
\begin{equation}
    \label{eqn:HD_3}
    \pmb{J} = \begin{bmatrix}
\pmb{0}_{d \times d} & \pmb{I}_{d \times d}\\
-\pmb{I}_{d \times d} & \pmb{0}_{d \times d}
\end{bmatrix}
\end{equation}
where $\pmb{I}$ is an identity matrix and $\pmb{0}$ is a zero matrix. The equations in Equation \eqref{eqn:HD_2} can be expressed explicitly in each of the $d$ dimensions as:
\begin{equation}
    \label{eqn:HD_4}
    \begin{aligned}
    \frac{dq_i}{dt} &= \frac{\partial H}{\partial p_i} \\ 
    \frac{dp_i}{dt} &= -\frac{\partial H}{\partial q_i} \\ 
    \end{aligned}~~~~\forall~~~~i \in [1,\dots, d].
\end{equation}
Equation \eqref{eqn:HD_2} is a coupled first-order ODE with general solution given by:
\begin{equation}
    \label{eqn:HD_5}
    \pmb{z}(T) = \pmb{z}(0) + \int_{0}^{T} \pmb{J}~\nabla H(\pmb{z})~dt
\end{equation}


Hamiltonian dynamics is numerically simulated using a symplectic time integrator \cite{Blanes2014a, Calvo1993a}. A symplectic time integrator is better suited than other integrators (e.g., Runge-Kutta) because it satisfies two key properties: (1) time reversibility and (2) volume preservation in the $\{\pmb{q},\pmb{p}\}$ phase space. Both properties are critical to accurately simulating Hamiltonian dynamics over long trajectories. This paper uses the basic symplectic integrator known as the synchronized leapfrog integration scheme. It is first assumed that the mass matrix associated with the kinetic energy $K(\pmb{p})$ is a diagonal matrix and thus $K(\pmb{p})$ can be expressed as:
\begin{equation}
    \label{eqn:HMC_4}
    K(\pmb{p}) \propto \sum_{i=1}^d\frac{p_i^2}{2m_i}~~~~\forall~~~~i \in [1,\dots, d]
\end{equation}
where $m_i$ is the mass associated with dimension $i$. In the synchronized leapfrog integration scheme, the position in dimension $i$ $(q_i)$ is first updated by:
\begin{equation}
    \label{eqn:LF_1}
    q_i(t+\Delta t) = q_i(t) + \frac{\Delta t}{m_i}~p_i(t) + \frac{\Delta t^2}{2m_i}~\dot{p}_i(t) = q_i(t) + \frac{\Delta t}{m_i}~p_i(t) - \frac{\Delta t^2}{2m_i}~\frac{\partial H}{\partial q_i(t)}
\end{equation}
where $\Delta t$ is the integration time step and $\dot{p}_i(t)$ is replaced with $-\frac{\partial H}{\partial q_i(t)}$ from Equation \eqref{eqn:HD_4}. Next, the momentum in dimension $i$ $(p_i)$ is updated by:
\begin{equation}
    \label{eqn:LF_2}
    p_i(t+\Delta t) = p_i(t) + \frac{\Delta t}{2}~\bigg(\dot{p}_i(t) + \dot{p}_i(t+\Delta t)\bigg) = p_i(t) - \frac{\Delta t}{2}~\bigg(\frac{\partial H}{\partial q_i(t)} + \frac{\partial H}{\partial q_i(t+\Delta t)}\bigg)
\end{equation}
where $\dot{p}_i(t)$ is again replaced with $-\frac{\partial H}{\partial q_i(t)}$ from Equation \eqref{eqn:HD_4}. The positions and momenta across all the dimensions are iteratively updated at each time step until a specified end time $T$ is reached.

\subsection{Hamiltonian Monte Carlo}

The fundamentals of HMC are discussed thoroughly in several resources (e.g. \cite{Neal2011a,Betancourt2017a,Vishnoi2021a}). The HMC algorithm aims to sample from a joint probability distribution of the following form:
\begin{equation}
    \label{eqn:HMC_1}
    f(\pmb{z} = \{\pmb{q},\pmb{p}\}) \propto \exp\big(-H(\pmb{z} = \{\pmb{q},\pmb{p}\})\big)
\end{equation}
Using the definition of the Hamiltonian from Equation \eqref{eqn:HD_1}, the joint distribution can be explicitly expressed to include contributions from the potential and kinetic energies as:
\begin{equation}
    \label{eqn:HMC_2}
    f(\{\pmb{q},\pmb{p}\}) \propto \exp\big(-U(\pmb{q})\big)~\exp\big(-K(\pmb{p})\big)
\end{equation}
The potential energy $U(\pmb{q})$ is defined as the negative logarithm of the likelihood function times the prior distribution:
\begin{equation}
    \label{eqn:HMC_3}
    U(\pmb{q}) \propto -\log\big[\mathcal{L}(\pmb{q}|D)~g(\pmb{q})\big]
\end{equation}
where $\mathcal{L}(\pmb{q}|D)$ is the likelihood function of the data $D$, and $g(\pmb{q})$ is the prior distribution. $K(\pmb{p})$ is also defined to simulate Hamiltonian dynamics. In the basic HMC algorithm, $\pmb{p}$ is assumed to follow a multivariate Gaussian with a diagonal covariance matrix $\pmb{M}$. $\pmb{M}$ has $d$ diagonal components: $m_i~\forall~i \in [1,\dots, d]$. $K(\pmb{p})$ is then defined as the negative logarithm of a multivariate Gaussian, and is given by Equation \eqref{eqn:HMC_4}. This definition of $\pmb{M}$ implies that the Hamiltonian dynamics operates in a Euclidean space. In contrast, if $\pmb{M}$ is non-diagonal but symmetric positive definite and depends on $\pmb{q}$, then the Hamiltonian dynamics operates in a Riemannian space \cite{Girolami2011a}, and the associated Monte Carlo algorithm is called Riemannian HMC. 

In the basic HMC algorithm, $\pmb{p}$ is drawn from a multivariate Gaussian with a diagonal $\pmb{M}$, Hamiltonian dynamics is simulated given $T$ and the starting position vector $\pmb{q}^0$, and the $\{\pmb{q},\pmb{p}\}$ vector at the end of the time integration serves as the Markov chain proposal $\{\pmb{q}^*,\pmb{p}^*\}$. This proposed sample is accepted as the next sample with the following probability:
\begin{equation}
    \label{eqn:HMC_5}
    \alpha = \textrm{min}\big[1,~\exp \big(H(\{\pmb{q}^{i-1},\pmb{p}^{i-1}\})-H(\{\pmb{q}^*,\pmb{p}^*\})\big)\big]
\end{equation}
where $\{\pmb{q}^{i-1},\pmb{p}^{i-1}\}$ is an accepted sample from the previous iteration of HMC.







\section{Sampling with L-HNNs}


Greydanus et al. \cite{Greydanus2019a} introduced HNNs as unsupervised machine learning architectures to learn Hamiltonian systems. Specifically, HNNs learn the gradients of the Hamiltonian which can be further used in an time integration scheme to predict the Hamiltonian trajectories. Given training data $\{\pmb{q}, \pmb{p}\}$, numerical gradients are first computed to give the numerical time derivatives $\{\frac{\Delta \pmb{q}}{\Delta t}, \frac{\Delta \pmb{p}}{\Delta t}\}$, per Hamilton's equations in \eqref{eqn:HD_4}. For some network parameters $\pmb{\theta}$, the HNN-predicted Hamiltonians (i.e., $H_{\pmb{\theta}}$) and their corresponding gradients (i.e., $\frac{\partial H_{\pmb{\theta}}}{\partial \pmb{p}}$ and $\frac{\partial H_{\pmb{\theta}}}{\partial \pmb{q}}$) are also computed. HNNs optimize the network parameters $\pmb{\theta}$ to minimize the following loss function:
\begin{equation}
    \label{eqn:HNN_1}
    L = \Big\lVert\frac{\partial H_{\pmb{\theta}}}{\partial \pmb{p}} - \frac{\Delta \pmb{q}}{\Delta t}\Big\lVert + \Big\lVert-\frac{\partial H_{\pmb{\theta}}}{\partial \pmb{q}} - \frac{\Delta \pmb{p}}{\Delta t}\Big\lVert
\end{equation}
Algorithm \ref{alg:HNN_train} summarizes the procedure for training HNNs. During the evaluation phase, HNNs predict the numerical gradients of the Hamiltonian $\{\frac{\partial H_{\pmb{\theta}}}{\partial \pmb{p}},\frac{\partial H_{\pmb{\theta}}}{\partial \pmb{q}}\}$, without explicit calls to the joint density $f(\{\pmb{q},\pmb{p}\}) \propto \exp\big(-H(\{\pmb{q},\pmb{p}\})\big)$. Therefore, during evaluation, HNNs require neither the numerical gradients of the potential energy function $U(\pmb{q}) \propto -\log\big[\mathcal{L}(\pmb{q}|D)~g(\pmb{q})\big]$ nor explicit calls to it. This by itself leads to computational gains since these gradients only need to be evaluated during HNN training, which requires comparatively small data sets. Moreover, evaluating the likelihood $\mathcal{L}(.)$ in $U(\pmb{q})$ can require calling an expensive computational model or a large dataset whose numerical gradients are very expensive to evaluate. Thus, avoiding numerical gradients of $U(\pmb{q})$ can lead to even more computational gains. Finally, gradients predicted by HNNs can be used to solve the following initial value problem to predict Hamiltonian dynamics:
\begin{equation}
    \label{eqn:HNN_2}
    \pmb{z}(T) = \pmb{z}(0) + \int_{0}^{T} \pmb{J}~\nabla H_{\pmb{\theta}}(\pmb{z})~dt
\end{equation}
where $\nabla H_{\pmb{\theta}}(\pmb{z}) = \{\frac{\partial H_{\pmb{\theta}}}{\partial \pmb{p}},\frac{\partial H_{\pmb{\theta}}}{\partial \pmb{q}}\}$ are predicted by HNNs. Equation \eqref{eqn:HNN_2} can be numerically solved using the synchronized leapfrog symplectic integrator, as presented in Algorithm \ref{alg:HNN_eval}. The gradients of the Hamiltonian are provided by HNNs (i.e., line 5 in Algorithm \ref{alg:HNN_eval}) rather than relying on numerical gradient evaluations of the Hamiltonian function.

\begin{algorithm}
\caption{Hamiltonian neural network training}
\label{alg:HNN_train}
\begin{algorithmic}[1]
\STATE{HNN parameters: $\pmb{\theta}$; Training data: $\{\pmb{q},~\pmb{p}\}$}
\STATE{Evaluate gradients of the training data $\frac{\Delta \pmb{q}}{\Delta t}$ and $\frac{\Delta \pmb{p}}{\Delta t}$}
\STATE{Initialize HNN parameters $\pmb{\theta}$}
\STATE{Compute HNN Hamiltonian $H_{\pmb{\theta}}$}
\STATE{Compute $H_{\pmb{\theta}}$ gradients $\frac{\partial H_{\pmb{\theta}}}{\partial \pmb{p}}$ and $-\frac{\partial H_{\pmb{\theta}}}{\partial \pmb{q}}$}
\STATE{Compute loss $\mathcal{L} = ||\frac{\partial H_{\pmb{\theta}}}{\partial \pmb{p}} - \frac{\Delta \pmb{q}}{\Delta t}|| + ||-\frac{\partial H_{\pmb{\theta}}}{\partial \pmb{q}} - \frac{\Delta \pmb{p}}{\Delta t}||$}
\STATE{Minimize $\mathcal{L}$ with respect to $\pmb{\theta}$}
\end{algorithmic}
\end{algorithm}
\normalsize


\begin{algorithm}
\caption{Hamiltonian neural networks evaluation in leapfrog integration}
\label{alg:HNN_eval}
\begin{algorithmic}[1]
\STATE{Hamiltonian: $H$; Initial conditions: $\pmb{z}(0)=\{\pmb{q}(0),~\pmb{p}(0)\}$; Dimensions: $d$; Steps: $N$; End time: $T$}
    \STATE{$\Delta t = \frac{T}{N}$}
    \FOR{$j = 0:N-1$}
        \STATE{$t = j~\Delta t$}
        \STATE{Compute HNN output gradient $\frac{\partial H_{\pmb{\theta}}}{\partial \pmb{q}(t)}$}
        \FOR{$i = 1:d$}
            \STATE{$q_i(t+\Delta t) = q_i(t) + \frac{\Delta t}{m_i}~p_i(t) - \frac{\Delta t^2}{2m_i}~\frac{\partial H_{\pmb{\theta}}}{\partial q_i(t)}$}
        \ENDFOR
        \STATE{Compute HNN output gradient $\frac{\partial H_{\pmb{\theta}}}{\partial \pmb{q}(t+\Delta t)}$}
        \FOR{$i = 1:d$}
            \STATE{$p_i(t+\Delta t) = p_i(t) - \frac{\Delta t}{2}~\bigg(\frac{\partial H_{\pmb{\theta}}}{\partial q_i(t)} + \frac{\partial H_{\pmb{\theta}}}{\partial q_i(t+\Delta t)}\bigg)$}
        \ENDFOR
    \ENDFOR
\end{algorithmic}
\end{algorithm}
\normalsize

{To generate the training data for HNNs and to further use them to solve Equation \eqref{eqn:HNN_2}, Greydanus et al. \cite{Greydanus2019a} relied on the Runge-Kutta scheme. A noteworthy difference in our approach is that we rely on the synchronized leapfrog integrator.
Being a symplectic integrator, synchronized leapfrog leads to stable Hamiltonian dynamics and better training data quality under relatively larger time steps. In addition, for the numerical evaluation of Equation \eqref{eqn:HNN_2} using HNNs, our initial studies indicated that synchronized leapfrog improves Hamiltonian conservation, which is key to an effective sampling scheme, as discussed in Section \ref{sec:Sta_Erg}.}

\subsection{Latent outputs in HNNs}\label{sec:lhnns}

The HNNs proposed by Greydanus et al. \cite{Greydanus2019a} take as inputs the $\{\pmb{q},~\pmb{p}\}$ coordinates during the forward pass, and return a scalar output value (i.e., the ``Hamiltonian'' [$H_{\pmb{\theta}}$]). Gradients are computed with respect to this Hamiltonian value and the loss function in Equation \eqref{eqn:HNN_1} is minimized. However, for Bayesian inference, it may be more beneficial for the HNNs to predict $d$ latent variables $\lambda_i$ $\{i \in 1, \dots, d\}$ denoted by the vector $\pmb{\lambda}$, where $d$ is the dimensionality of the uncertainty space. The architecture of L-HNNs is formally defined as:
\begin{equation}
    \label{eqn:LHNN_1}
    \begin{aligned}
    &\pmb{u}_p = \phi(\pmb{w}_{p-1}~\pmb{u}_{p-1} + \pmb{b}_{p-1}),~~\{p \in 1, \dots, P\}\\
    &\pmb{\lambda} = \pmb{w}_{P}~\pmb{u}_{P} + \pmb{b}_{P}
    \end{aligned}
\end{equation}
where $P$ is the number of hidden layers indexed by $p$, $\pmb{u}_p$ are the outputs of the hidden layer $p$, $\pmb{w}_p$ and $\pmb{b}_p$ are the weights and biases, respectively, $\phi(.)$ is the nonlinear activation function, and $\pmb{u}_{0}$ are the inputs $\pmb{z} = \{\pmb{q},~\pmb{p}\}$. The Hamiltonian computed by L-HNNs is defined as:
\begin{equation}
    \label{eqn:LHNN_2}
    H_{\pmb{\theta}} = \sum_{i=1}^d \lambda_i
\end{equation}
Gradients are then computed with respect to the Hamiltonian in Equation \eqref{eqn:LHNN_2} and the loss function in Equation \eqref{eqn:HNN_1} is minimized. Through these latent variables, L-HNNs have more degrees of freedom and expressivity than traditional HNNs, potentially leading to reduced integration errors when predicting the Hamiltonian dynamics. As a result, L-HNNs can show improved sampling efficiency relative to HNNs. A comparison between the performance of L-HNNs and HNNs is presented in Section \ref{sec:Rosen3D} for a 3-D Rosenbrock density function.

\subsection{Sampling}\label{sec:HMC_LHNNs}

HMC sampling using L-HNNs proceeds in similar fashion to traditional HMC. The key difference is that, given the initial conditions $\{\pmb{q}(0),\pmb{p}(0)\}$, the Hamiltonian dynamics is approximated using L-HNNs via Algorithm \ref{alg:HNN_eval} rather than through direct integration. Once the new proposal $\{\pmb{q}^*,\pmb{p}^*\}$ is simulated, the Metropolis acceptance criterion is used to decide whether to accept this proposal, based on the conservation of the Hamiltonian:
\begin{equation}
    \label{eqn:HMC_HNN_1}
    \begin{aligned}
    &\{\pmb{q}^{i},\pmb{p}^{i}\} \gets \{\pmb{q}^*,\pmb{p}^*\}\textrm{ with probability }\alpha\\
    &\textrm{where } \alpha = \textrm{min}\{1,~\exp \big(H(\{\pmb{q}^{i-1},\pmb{p}^{i-1}\})-H(\{\pmb{q}^*,\pmb{p}^*\})\big)\}\\
    \end{aligned}
\end{equation}
This procedure is summarized in Algorithm \ref{alg:HNNs_HMC}.

\begin{algorithm}
\caption{L-HNNs in Hamiltonian Monte Carlo}
\label{alg:HNNs_HMC}
\begin{algorithmic}[1]
\STATE{Hamiltonian: $H = U(\pmb{q}) + K(\pmb{p})$; Samples: M; Starting sample: $\{\pmb{q}^0,~\pmb{p}^0\}$; End time for trajectory: $T$; Steps: $N$; Dimensions: $d$}
\FOR{$i = 1:M$}
\STATE{$\pmb{p}(0) \sim \mathcal{N}(\pmb{0},~\pmb{I}_d)$}
\STATE{$\pmb{q}(0) = \pmb{q}^{i-1}$}
\STATE{Compute $\{\pmb{q}^*,~\pmb{p}^*\} = \{\pmb{q}(T),~\pmb{p}(T)\}$ with Algorithm \ref{alg:HNN_eval}}
\STATE{$\alpha = \textrm{min}\{1,~\exp \big(H(\{\pmb{q}^{i-1},\pmb{p}^{i-1}\})-H(\{\pmb{q}^*,\pmb{p}^*\})\big)\}$}
\STATE{With probability $\alpha$, set $\{\pmb{q}^{i},\pmb{p}^{i}\} \gets \{\pmb{q}^*,\pmb{p}^*\}$}
\ENDFOR
\end{algorithmic}
\end{algorithm}
\normalsize

For a simple demonstration of using L-HNNs in HMC, we consider the following 1-D Gaussian mixture density, which has modes at $+1$ and $-1$, both having a standard deviation of $0.35$:
\begin{equation}
    \label{eqn:1D_Gauss_Mix}
    f(q) \propto 0.5~\exp\Bigg(\frac{(q-1)^2}{2 \times 0.35^2}\Bigg) + 0.5~\exp\Bigg(\frac{(q+1)^2}{2 \times 0.35^2}\Bigg)
\end{equation}
\noindent The L-HNNs have three hidden layers with 100 neurons each. To generate the training data, we simulated the Hamiltonian dynamics for 20 samples, each with $T=20$ units. The number of leapfrog steps was 20 per unit, meaning that $\Delta t = 0.05$. The initial condition for the first sample was set to zeros. The initial conditions for the rest of the samples were defined as follows: (1) position $q(0)$ was the final position from the previous sample, and (2) momentum $p(0)$ was drawn randomly from a standard Gaussian distribution. In total, the training of L-HNNs took $8,000$ numerical gradient evaluations. 

Figure \ref{Sampling_Demo_HMC_33} compares the $\{q,p\}$ phase space between L-HNNs and numerical gradients for several initial momentum values $p(0)$ and identical initial position $q(0) = 0$. Note that the phase space plots predicted using numerical gradients and L-HNNs compare satisfactorily. We further performed HMC using numerical gradients and L-HNNs to generate 5,000 samples from the distribution in Equation \eqref{eqn:1D_Gauss_Mix}, with the first 1,000 considered as burn-in. Each sample had an end time of $T=5$ units, with each unit being discretized into 20 steps (i.e., $\Delta t = 0.05$). The histogram plots compare well between traditional HMC (Figure \ref{Sampling_Demo_HMC_3}) and HMC using HNNs (Figure \ref{Sampling_Demo_HMC_4}). Moreover, for this simple example, using L-HNNs for HMC requires only the 8000 numerical gradients of the target density from the training, whereas HMC required 500,000 numerical gradients for the same number of samples (a 98.4\% reduction).

\begin{figure}[htbp]
\begin{subfigure}{0.32\textwidth}
\centering  
\includegraphics[width=\textwidth]{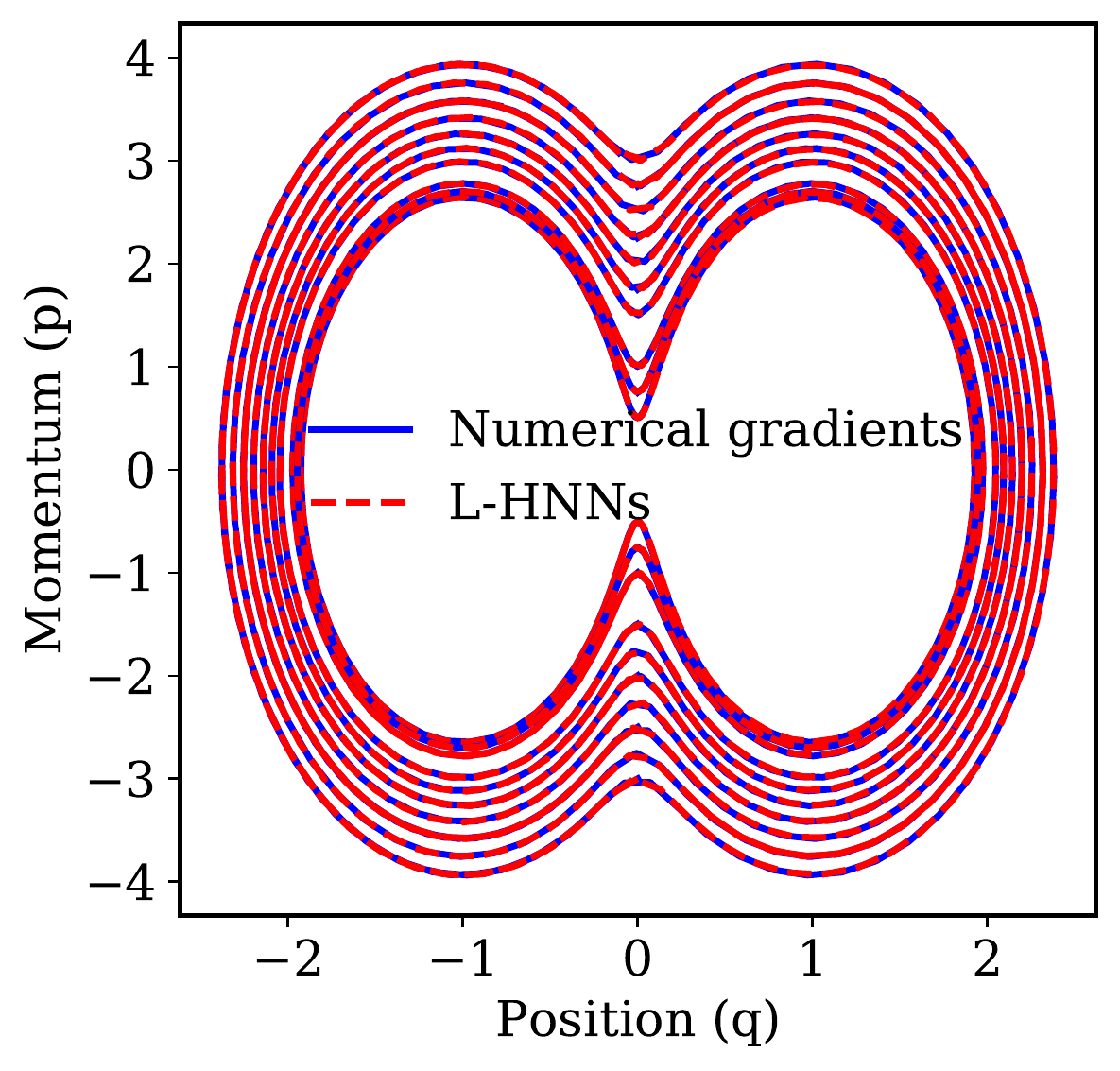} 
\caption{}
\label{Sampling_Demo_HMC_33}
\end{subfigure}
\begin{subfigure}{0.32\textwidth}
\centering  
\includegraphics[width=\textwidth]{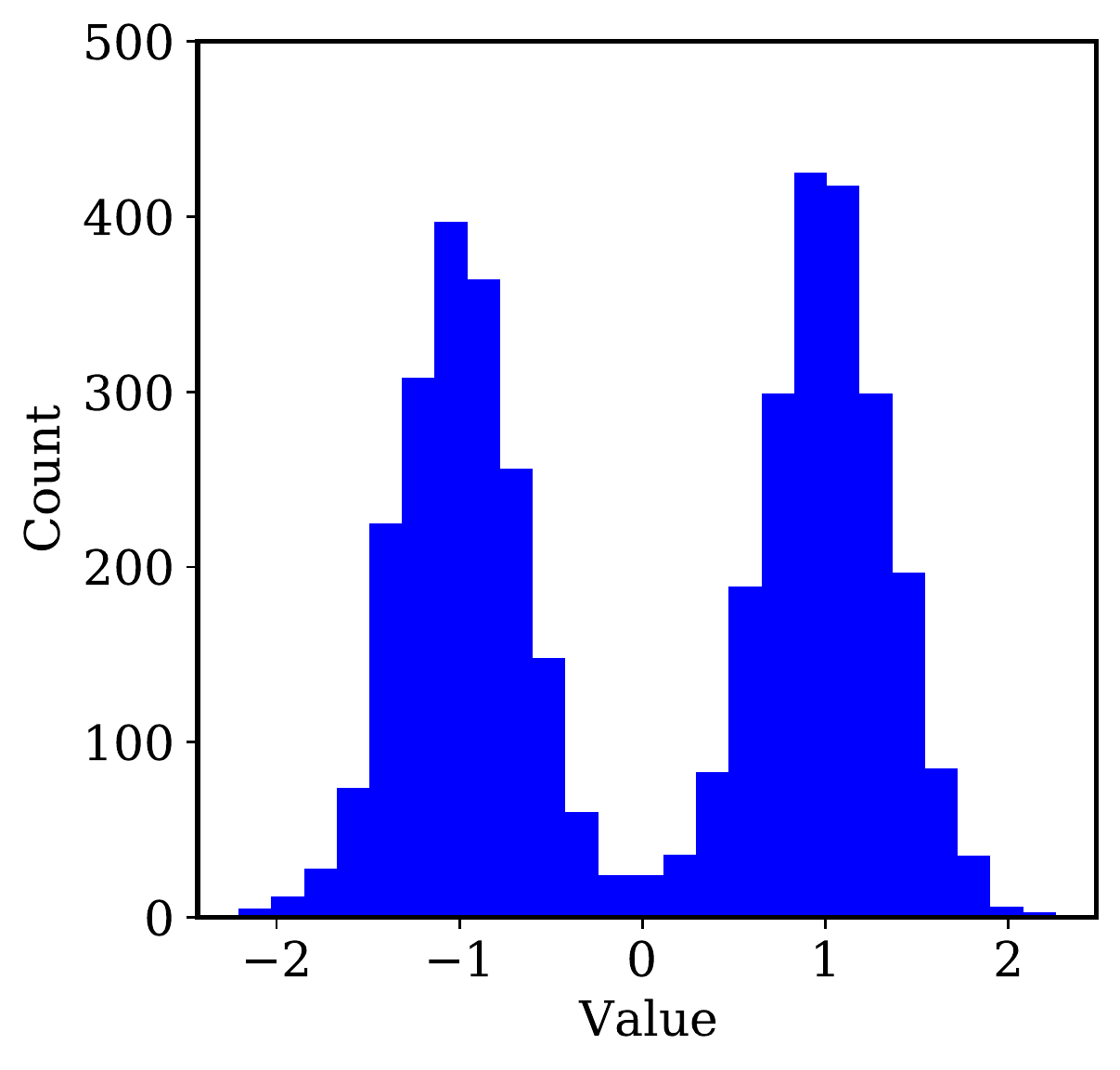}
\caption{}
\label{Sampling_Demo_HMC_3}
\end{subfigure}
\begin{subfigure}{0.32\textwidth}
\centering  
\includegraphics[width=\textwidth]{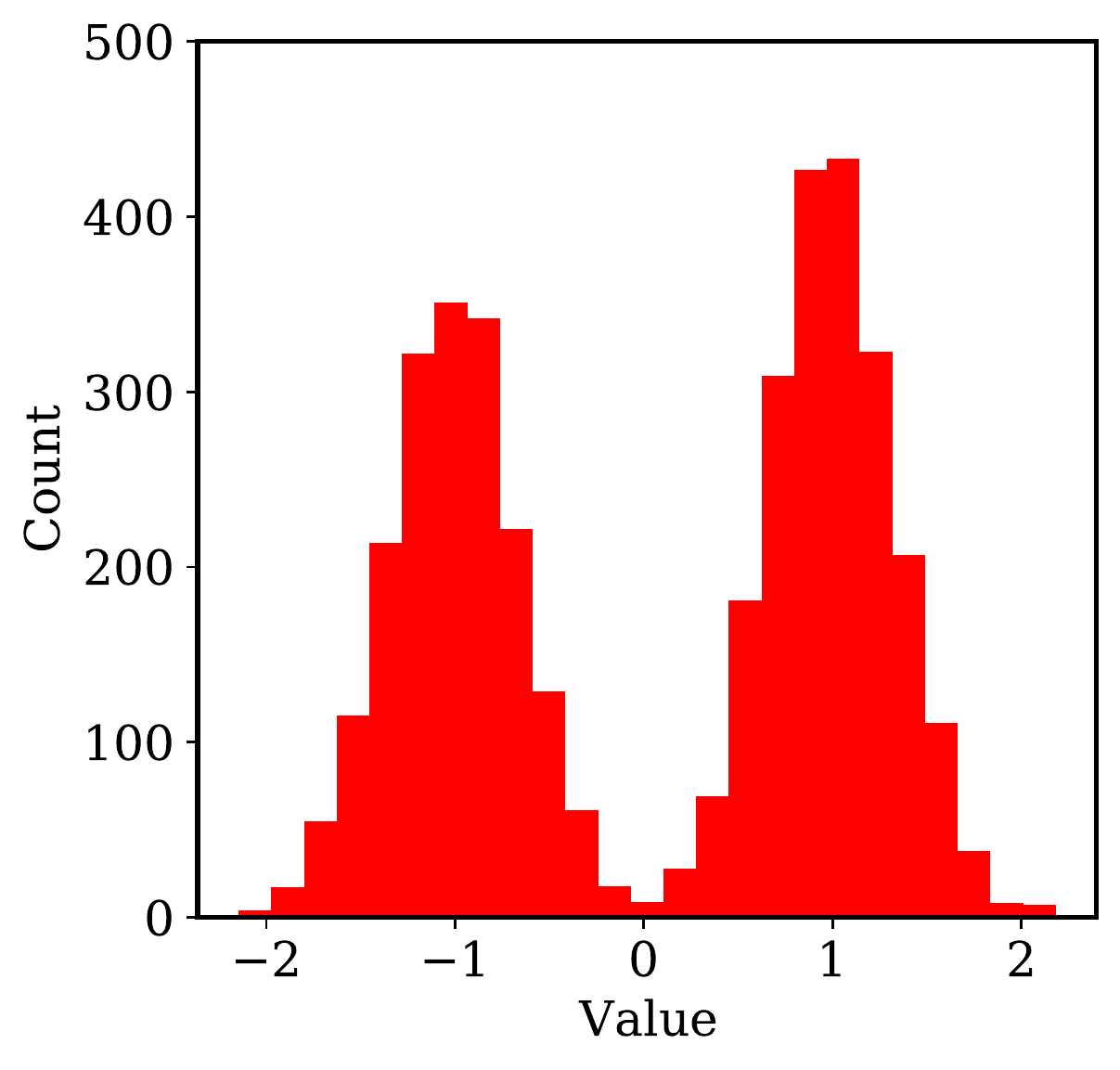}
\caption{}
\label{Sampling_Demo_HMC_4}
\end{subfigure}
\caption{{(a) $\{q, p\}$ phase space plot corresponding to different energy levels. (b) and (c) Histogram of samples from a 1-D Gaussian mixture density using  traditional HMC and L-HNNs in HMC, respectively. L-HNNs in HMC requires only $1.6\%$ of the numerical gradients of target density required by traditional HMC.}}
\label{Sampling_Demo_HMC_Hist}
\end{figure}

\subsection{Discussion of stationarity}\label{sec:Sta_Erg}

MCMC samplers should be stationary with respect to the target density $\pi$. The formal theorem on the stationarity of an idealized HMC sampler, in which no numerical errors are introduced when simulating the Hamiltonian dynamics, is presented in Theorem \ref{thm:station}. 

\begin{theorem}[Stationarity \cite{Vishnoi2021a}]\label{thm:station}
  Let $f : \mathbb{R}^d \to \mathbb{R}$ be a differentiable function. Let $\Delta t > 0$ be the step size of the idealized HMC. Let $\psi_{\Delta t}(\pmb{q},\pmb{p})$ denote the position in the phase space at step $\Delta t$. Suppose $\{\pmb{q},\pmb{p}\}$ is a sample from the density
     \begin{displaymath}
    \pi(\pmb{q},\pmb{p}) \propto \exp\big(-U(\pmb{q})\big)~\exp\big(-\frac{1}{2}\lVert\pmb{p}\lVert^2)\big)
  \end{displaymath} Then, the density of $\psi_{\Delta t}(\pmb{q},\pmb{p})$ is $\pi$ for any $\Delta t > 0$. Moreover, the density of $\psi_{\Delta t}(\pmb{q},\pmb{\xi})$, where $\pmb{\xi} \sim \mathcal{N}(\pmb{0},\pmb{I}_d)$, is also $\pi$. Thus, the idealized HMC preserves $\pi$.
\end{theorem}
\begin{proof}
  The proof of the theorem relies on the Hamiltonian dynamics being time-reversible and the Hamiltonian being conserved. The formal proof is presented in Vishnoi \cite{Vishnoi2021a}.
\end{proof}

As discussed in Greydanus et al.\ \cite{Greydanus2019a}, the Hamiltonian dynamics simulated by HNNs and L-HNNs are perfectly time-reversible as a result of Liouville’s theorem. This feature can also be seen in Figure \ref{SaE_1}, wherein two Hamiltonian trajectories for the 1-D Gaussian mixture density are simulated in forward and reverse time using L-HNNs. Both trajectories have the same initial position in forward time, but different initial momenta. To simulate the reverse-time trajectories, the final $\{q,p\}$ values at the end time during forward integration were used as the initial conditions. Backward time integration was then performed using L-HNNs. 

By virtue of their loss function (Equation \eqref{eqn:HNN_1}), L-HNNs can also conserve the Hamiltonian in an approximate sense. The degree to which the Hamiltonian is conserved may depend on the target distribution of interest and the training data. For example, Figure \ref{SaE_2} presents the Hamiltonian evolution with time for the 1-D Gaussian mixture density, considering several initial conditions $\{q(0),p(0)\}$ using $\Delta t=0.05$. Overall, the Hamiltonian values are seen to be approximately constant with time. For comparison, the Hamiltonian values obtained using traditional numerical gradients are also presented, and are also approximately constant with time. Owing to some numerical errors in the Hamiltonian conservation, and in line with the traditional HMC procedure, the use of a Metropolis acceptance criterion in Equation \eqref{eqn:HMC_HNN_1} is important to ensure stationarity when using HNNs in HMC. The Metropolis acceptance criterion leaves the target density $\pi$ invariant, as discussed in Roberts and Rosenthal \cite{Roberts2004a}. 


\begin{figure}[!htbp]
\begin{subfigure}{0.48\textwidth}
\centering  
\includegraphics[scale=0.42]{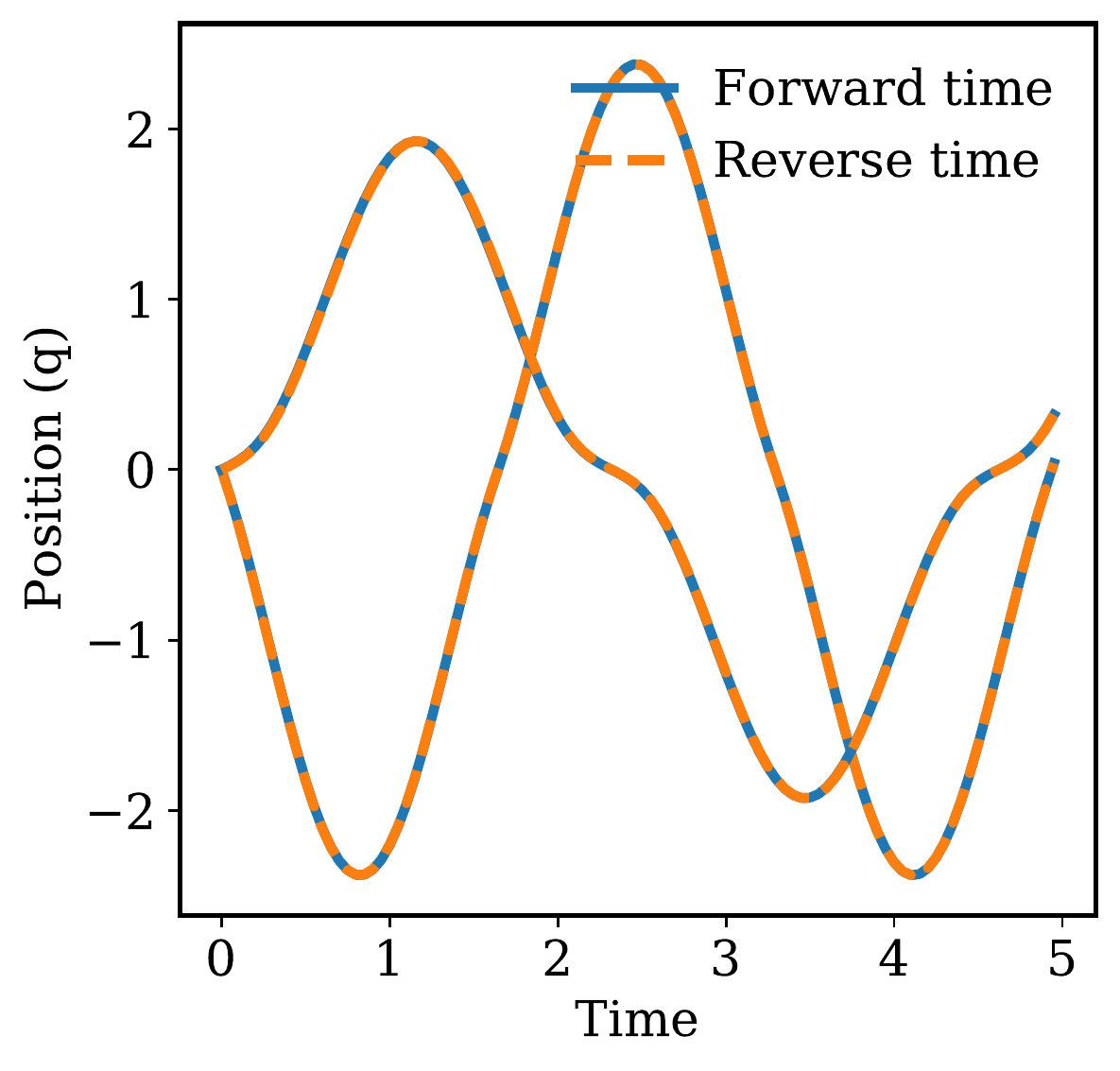} 
\caption{}
\label{SaE_1}
\end{subfigure}
\begin{subfigure}{0.48\textwidth}
\centering  
\includegraphics[scale=0.42]{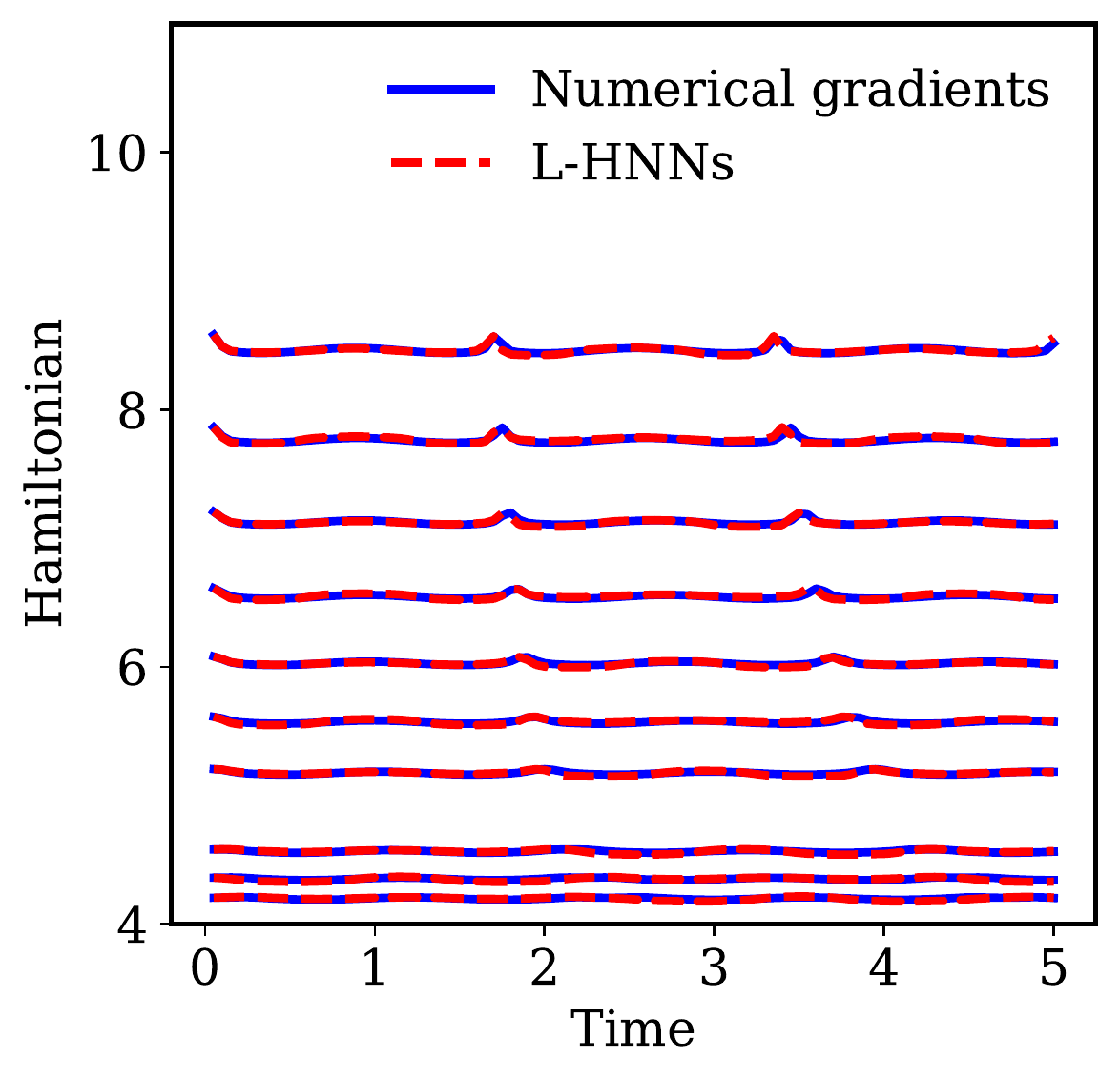} 
\caption{}
\label{SaE_2}
\end{subfigure}
\caption{{(a) Time reversibility of the L-HNNs in predicting the Hamiltonian dynamics under two different sets of initial conditions, given a 1-D Gaussian mixture density. (b) Conservation of the Hamiltonian for different energy levels given a 1-D Gaussian mixture density.}}
\label{SaE}
\end{figure}

\section{No-U-Turn Sampling (NUTS) with L-HNNs}

Traditional HMC requires the end time $T$ to be specified. For all but simple problems, specifying $T$ is not only challenging but also has considerable impact on HMC sampler performance diagnostics (e.g., the serial auto-correlations between the samples). Hoffman and Gelman \cite{Hoffman2014} introduced the NUTS algorithm to automatically decide $T$ for each sample in HMC. In summary, NUTS works by building the set of candidate $\{\pmb{q},~\pmb{p}\}$ states $\mathcal{B}$, using a binary tree. At each iteration, the tree length is doubled by integrating randomly in either the forward or reverse direction. Out of the states $\mathcal{B}$, a subset of states $\mathcal{C}$ is selected so as not to violate the detailed balance condition. A slice variable $u$ is often used in NUTS to facilitate the selection of subset $\mathcal{C}$. The doubling procedure for building the binary tree is stopped when a u-turn is made, as defined by the criterion:
\begin{equation}
    \label{eqn:NUTS_1}
    (\pmb{q}^{+}-\pmb{q}^{-}) \cdot \pmb{p}^{-} < 0~~\textrm{or}~~(\pmb{q}^{+}-\pmb{q}^{-}) \cdot \pmb{p}^{+} < 0
\end{equation}
where $\{\pmb{q}^{+},~\pmb{p}^{+}\}$ and $\{\pmb{q}^{-},~\pmb{p}^{-}\}$ are, respectively, the leftmost and rightmost states of a sub-tree in the binary tree. The doubling procedure may also be stopped when the integration error while simulating the Hamiltonian dynamics exceeds a threshold, as defined by:
\begin{equation}
    \label{eqn:NUTS_2}
    \varepsilon \equiv H\big(\pmb{q},~\pmb{p}\big)+\ln{u} > \Delta_{max}
\end{equation}
where $\{\pmb{q},~\pmb{p}\}$ is the current candidate state and $\Delta_{max}$ is a threshold recommended by Hoffman and Gelman to be set to 1,000 \cite{Hoffman2014}. From the subset of states $\mathcal{C}$, the next state is randomly selected with a uniform probability. This version of the NUTS algorithm is referred to as naive NUTS, since it requires storing $O(2^j)$ ($j$ is the number of doubling iterations in the binary tree) candidate states, leading to large memory requirements. {Naive NUTS also requires simulating the full tree---even when the stopping criteria are satisfied in the middle of the doubling iteration---leading to an expenditure of computational resources.} Hoffman and Gelman \cite{Hoffman2014} introduced an efficient version of NUTS to alleviate these issues by using a sophisticated transition kernel for building the binary tree which only requires storing $O(j)$ candidate states in memory, but generating an extra $O(2^j)$ random numbers (an insignificant cost). The efficient version of NUTS is used in this paper.

\subsection{L-HNNs in NUTS}

Here, we use L-HNNs in place of the leapfrog integrator in the efficient NUTS algorithm to build the binary tree. As a simple demonstration, consider the 1-D Gaussian mixture density discussed in Section \ref{sec:HMC_LHNNs}. Details of the L-HNN training remain the same. Using L-HNNs in efficient NUTS, 5,000 samples were simulated, with the first 1,000 samples considered as burn-in and $\Delta t = 0.05$. Table \ref{Table_HMC_ESS} compares the Effective Sample Size (ESS) of traditional NUTS with that of L-HNNs in NUTS, where the ESS is normalized by the total number of gradient evaluations (either from the HMC or the L-HNN training). Also presented are the normalized ESS values for traditional HMC and L-HNNs in HMC. Overall, the ESS per gradient of the target density is considerably better for L-HNNs in HMC and L-HNNs in NUTS compared to traditional HMC and traditional NUTS. Although the ESS per gradient of the target density is similar between L-HNNs in HMC and L-HNNs in NUTS, for achieving a similar value of the ESS, while L-HNNs in HMC required $500,000$ neural network gradients, L-HNNs in NUTS required only around $91,000$. This difference demonstrates the efficacy of NUTS over HMC with L-HNNs for this simple problem. For more complex high-dimensional problems, the efficacy of L-HNNs in NUTS is expected to be even since it is difficult to know the optimal end time for the Hamiltonian trajectory a priori.

\begin{table}[htbp]
\centering
\caption{Comparison of the performance of traditional HMC, traditional NUTS, HNN-HMC, and HNN-NUTS, considering a 1-D Gaussian mixture density. ESS per gradient is used as the performance metric.}
\label{Table_HMC_ESS}
\small
\begin{tabular}{ |c|c|c| }
\hline
 & \textbf{\Centerstack[c]{ESS per gradient \\ of target density}}   \\
\hline
\textbf{L-HNNs in HMC} & $4.59E-3$ \\
\hline
\textbf{Traditional HMC} & $8.42E-5$ \\
\hline
\textbf{L-HNNs in NUTS} & $4.83E-3$ \\
\hline
\textbf{Traditional NUTS} & $4.4E-4$ \\
\hline
\end{tabular}
\end{table}
\normalsize


Hoffman and Gelman \cite{Hoffman2014} also proposed a dual averaging scheme to automatically tune the step size $\Delta t$ given a target acceptance ratio. While this automatic step size determination will be interesting to explore for L-HNNs in NUTS, in this paper, we fix the step size to the same small value for both L-HNNs in NUTS and traditional NUTS. This helps facilitate a fair comparison between these algorithms, and step size tuning will be covered in a future study.

\subsection{Online error monitoring}

While direct use of L-HNNs in NUTS works for a simple 1-D Gaussian mixture density, for more complex distributions care must be taken to avoid large integration errors when building the binary tree. Large integration errors can arise when sampling in the tails of a distribution, since limited L-HNN training data may be available in these regions of low probability density. While building the binary tree, when such a region of low probability density is reached and integration errors are large, NUTS prematurely terminates the tree building procedure. As a result, the next sample will likely be in the same vicinity as the previous one. This cycle of generating a new sample near the previous sample continues, leading to degenerate clusters of samples in regions of low probability density. Figure \ref{Naive_NUTS} presents an example of this sampling degeneracy when using L-HNNs in NUTS for a 3-D Rosenbrock density with heavy tails \cite{Pagani2019a}. Note that the sampler is unable to rapidly move to regions of high probability density after visiting a region of low probability density, as a result of large integration errors when building the binary tree. This sort of degeneracy in sampling may be less concerning when using L-HNNs in HMC, by virtue of the Metropolis acceptance criterion in Equation \eqref{eqn:HMC_HNN_1}, which rejects a proposed sample when the difference in the Hamiltonian values is very large. At the same time, the Metropolis criterion can also prevent L-HNNs in HMC from efficiently exploring the uncertainty space.

\begin{figure}[htbp]
\begin{subfigure}{0.32\textwidth}
\centering  
\includegraphics[width=\textwidth]{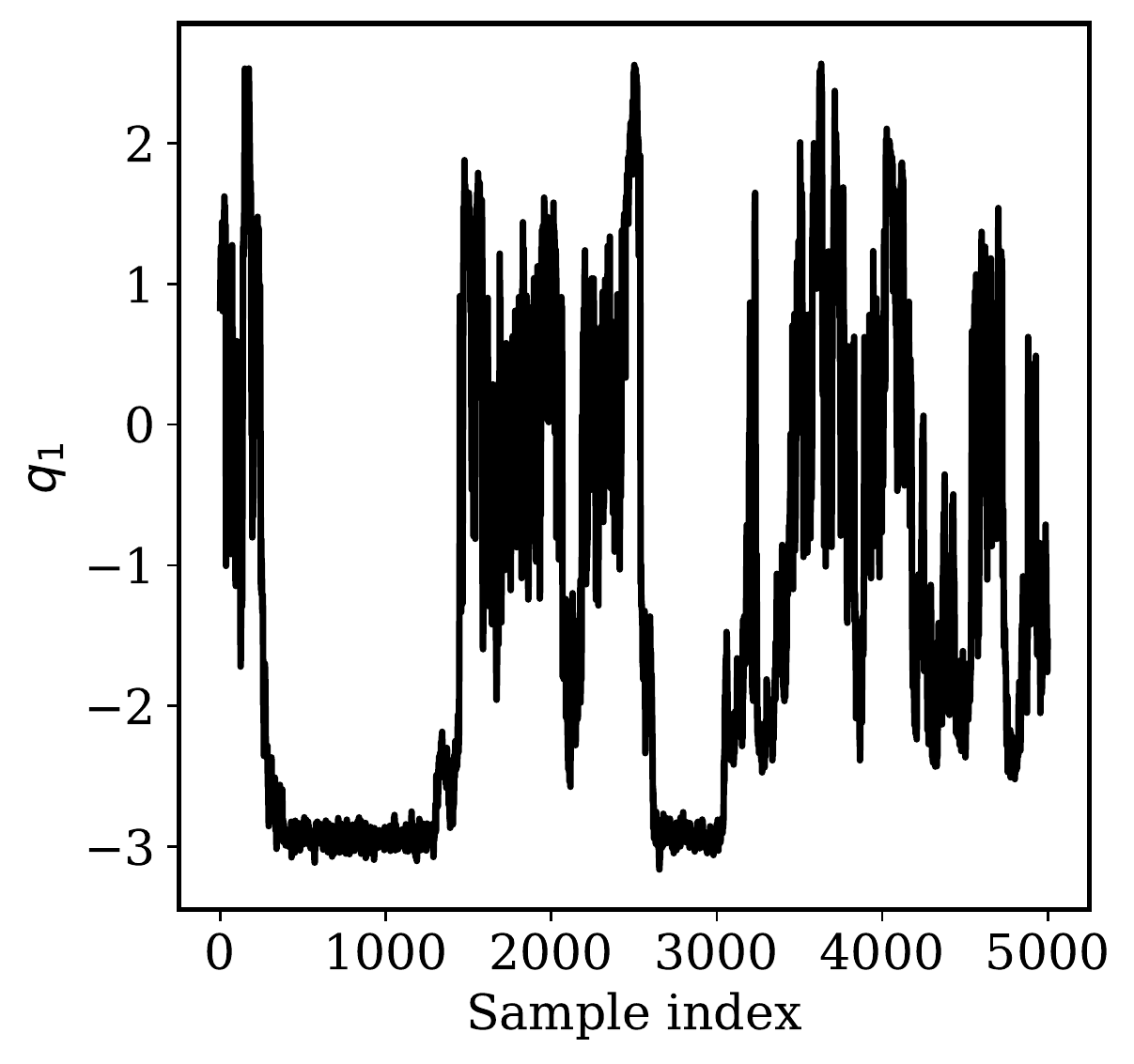} 
\caption{}
\label{Naive_NUTS_1}
\end{subfigure}
\begin{subfigure}{0.32\textwidth}
\centering  
\includegraphics[width=\textwidth]{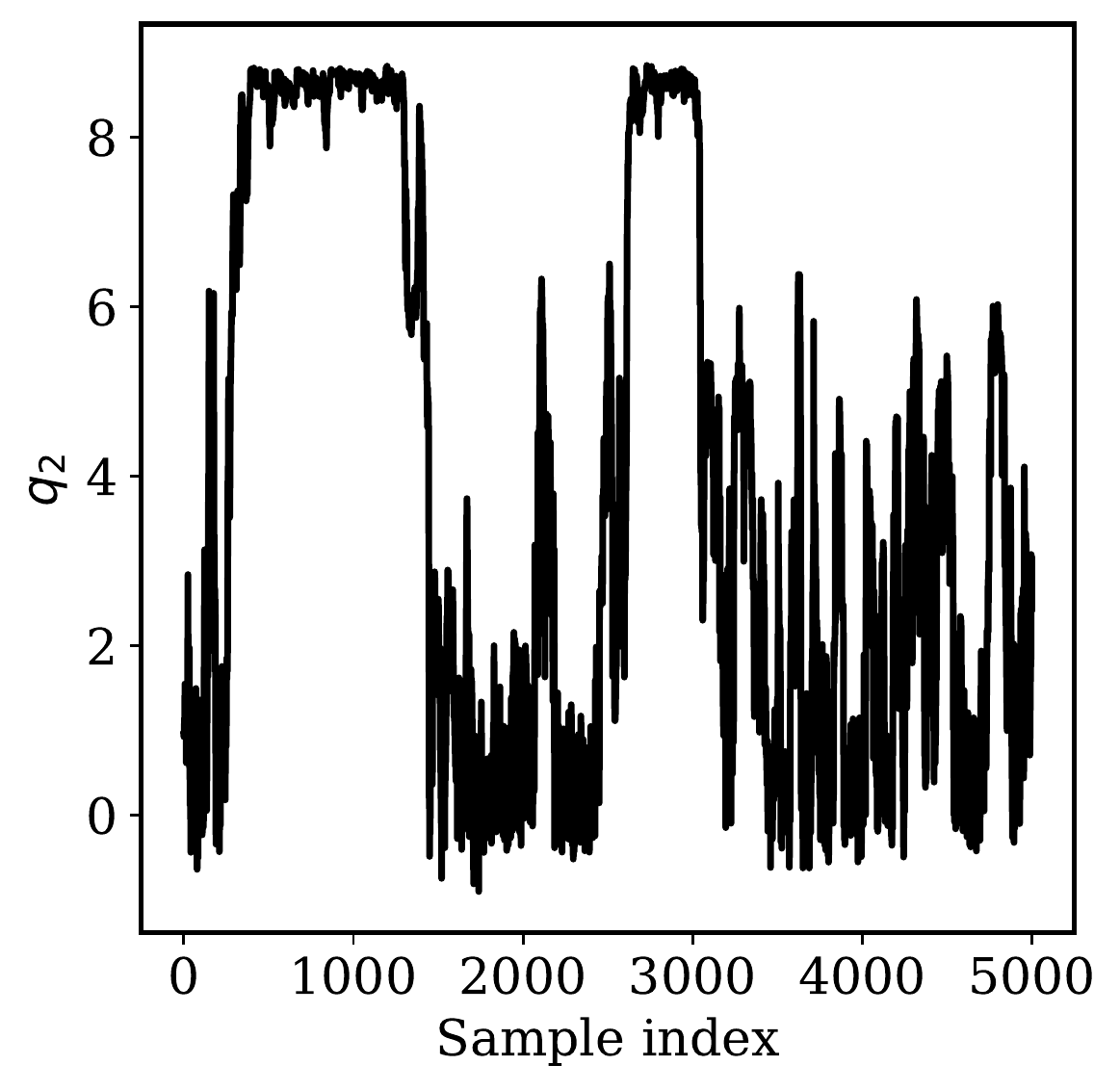} 
\caption{}
\label{Naive_NUTS_2}
\end{subfigure}
\begin{subfigure}{0.32\textwidth}
\centering  
\includegraphics[width=\textwidth]{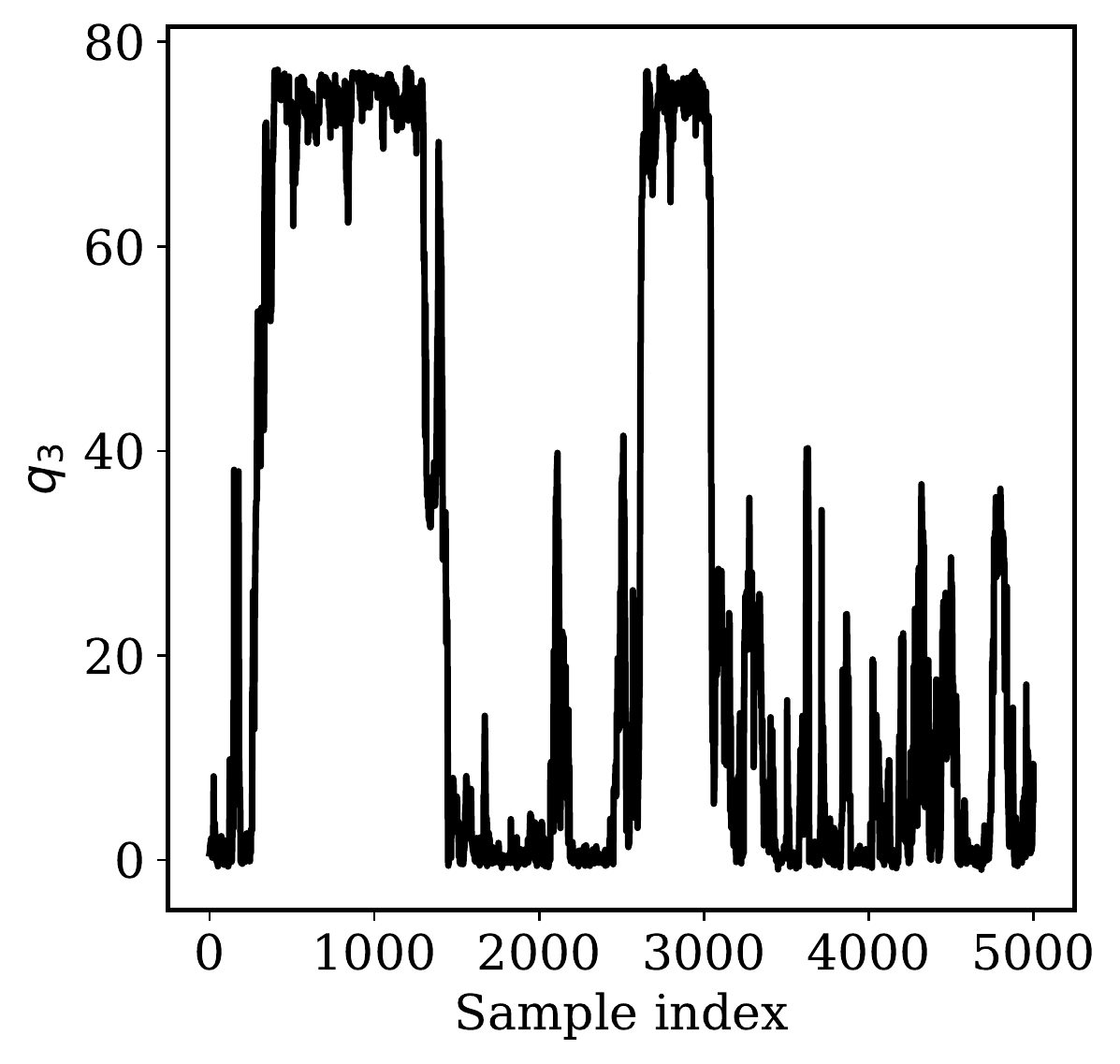} 
\caption{}
\label{Naive_NUTS_3}
\end{subfigure}
\caption{Degeneracy in sampling when using L-HNNs in NUTS for a 3-D Rosenbrock density. These trace plots correspond to the dimensions (a) $q_1$, (b) $q_2$, and (c) $q_3$.}
\label{Naive_NUTS}
\end{figure}

To alleviate the above-mentioned issues caused by large integration errors, we propose an online error monitoring scheme for L-HNNs in NUTS. This error monitoring scheme relies on the stopping criterion in Equation \ref{eqn:NUTS_2}, as introduced by Hoffman and Gelman \cite{Hoffman2014}. This criterion monitors the integration errors when simulating the Hamiltonian trajectories and terminates the tree building procedure when they are large. In theory, if the integration errors are non-existent, then the quantity $\varepsilon$ in Equation \ref{eqn:NUTS_2} will always be less than zero. But for practical applications with non-negligible integration errors, the error metric ($\varepsilon \equiv H(\pmb{q},~\pmb{p})+\ln{u}$) can be greater than zero during a doubling iteration. We define a threshold $\Delta_{max}^{hnn}$ for the error metric $\varepsilon$ while using L-HNNs. When this quantity exceeds $\Delta_{max}^{hnn}$, the L-HNN is abandoned and standard leapfrog integration using numerical gradients of the target density takes over. We also introduce a parameter $N_{lf}$, which sets the number of leapfrog samples to take before reintroducing the L-HNN. $N_{lf}$ can be viewed as the number of ``cooldown'' samples needed to return the sampler to regions where the L-HNN is sufficiently trained. In the traditional version of efficient NUTS, Hoffman and Gelman \cite{Hoffman2014} recommend that the error threshold for leapfrog integration be set to a large value (e.g., $\Delta_{max}^{lf}=1,000$). In setting the threshold for L-HNNs, we recommend a more conservative value of around $\Delta_{max}^{hnn}=10$ to prevent sampling degeneracy in regions of low probability density. $N_{lf}$ can be set to a low value of 5--20 samples in order to essentially return the sampler to regions of high probability density. These settings will contribute to the sampler exploring the uncertainty space well while ensuring that L-HNNs integration errors in certain regions of the space do not cause sampling degeneracy.



Algorithms \ref{alg:HNN_NUTS1} and \ref{alg:HNN_NUTS2} present the pseudocode for using L-HNNs in efficient NUTS with online error monitoring. These algorithms differ somewhat from the traditional version of efficient NUTS proposed by Hoffman and Gelman \cite{Hoffman2014}. The most notable difference is the use of L-HNNs for predicting the Hamiltonian dynamics in the \textbf{BuildTree(.)} function of Algorithm \ref{alg:HNN_NUTS2}. In addition, the \textbf{BuildTree(.)} function takes and returns an indicator variable $\mathbbm{1}_{lf}$ to determine whether L-HNNs or the leapfrog scheme with numerical gradients of the target density should be used to simulate the Hamiltonian dynamics. The main loop (Algorithm \ref{alg:HNN_NUTS1}) monitors the number of samples to which the leapfrog scheme with numerical gradients of the target density is applied, and facilitates reversion to L-HNNs once this number equals $N_{lf}$. There are two error thresholds (i.e., $\Delta_{max}^{lf}$ and $\Delta_{max}^{hnn}$) in the pseudocode. $\Delta_{max}^{lf}$ is the error threshold for the leapfrog scheme using numerical gradients of the target density. $\Delta_{max}^{hnn}$ is the error threshold for L-HNNs with $\Delta_{max}^{hnn} << \Delta_{max}^{lf}$. With respect to the convergence between L-HNNs in NUTS with online error monitoring and traditional NUTS, Proposition \ref{prop_1} can be shown.

\begin{algorithm}
\caption{L-HNNs in efficient NUTS with online error monitoring (main loop)}
\label{alg:HNN_NUTS1}
\begin{algorithmic}[1]
\STATE{Hamiltonian: $H = U(\pmb{q}) + K(\pmb{p})$, Samples: M; Starting sample: $\{\pmb{q}^0,~\pmb{p}^0\}$; Step size: $\Delta t$; Threshold for leapfrog: $\Delta_{max}^{lf}$; Threshold for L-HNNs: $\Delta_{max}^{hnn}$; Number of leapfrog samples: $N_{lf}$}
\STATE{Initialize $\mathbbm{1}_{lf} = 0,~n_{lf} = 0$}
\FOR{$i = 1:M$}
\STATE{$\pmb{p}(0) \sim \mathcal{N}(\pmb{0},~\pmb{I}_d)$}
\STATE{$\pmb{q}(0) = \pmb{q}^{i-1}$}
\STATE{$u \sim Uniform\Big(\Big[0,~\exp{\{-H\big(\pmb{q}(0),~\pmb{p}(0)\big)\}}\Big]\Big)$}
\STATE{Initialize $\pmb{q}^- = \pmb{q}(0),\pmb{q}^+ = \pmb{q}(0),\pmb{p}^- = \pmb{p}(0),\pmb{p}^+ = \pmb{p}(0),j=0,\pmb{q}^* = \pmb{q}^{i-1},n=1,s=1$}
\IF{$\mathbbm{1}_{lf} = 1$}
\STATE{$n_{lf} \gets n_{lf} + 1$}
\ENDIF
\IF{$n_{lf} = N_{lf}$}
\STATE{$\mathbbm{1}_{lf} = 0,~n_{lf} = 0$}
\ENDIF
\WHILE{$s=1$}
\STATE{Choose direction $\nu_j \sim Uniform(\{-1,~1\})$}
\IF{$j=-1$}
\STATE{$\pmb{q}^-,\pmb{p}^-,\_,\_,\pmb{q}^\prime,\pmb{p}^\prime,n^\prime,s^\prime,\mathbbm{1}_{lf} = \textrm{\textbf{BuildTree}}(\pmb{q}^-,\pmb{p}^-,u,\nu_j,j,\Delta t,\mathbbm{1}_{lf})$}
\ELSE 
\STATE{$\_,\_,\pmb{q}^+,\pmb{p}^+,\pmb{q}^\prime,\pmb{p}^\prime,n^\prime,s^\prime,\mathbbm{1}_{lf} = \textrm{\textbf{BuildTree}}(\pmb{q}^+,\pmb{p}^+,u,\nu_j,j,\Delta t,\mathbbm{1}_{lf})$}
\ENDIF
\IF{$s^\prime = 1$}
\STATE{With probability $\textrm{min}\{1,~\frac{n^\prime}{n}\}$}, set $\{\pmb{q}^{i},\pmb{p}^{i}\} \gets \{\pmb{q}^\prime,\pmb{p}^\prime\}$ 
\ENDIF
\STATE{$n \gets n + n^\prime$}
\STATE{$s \gets s^\prime \mathbbm{1}\big[(\pmb{q}^+ - \pmb{q}^-) \cdot \pmb{p}^- \geq 0\big] \mathbbm{1}\big[(\pmb{q}^+ - \pmb{q}^-) \cdot \pmb{p}^+ \geq 0\big]$}
\STATE{$j \gets j + 1$}
\ENDWHILE
\ENDFOR
\end{algorithmic}
\end{algorithm}
\normalsize

\begin{algorithm}
\caption{L-HNNs in efficient NUTS with online error monitoring (build tree function)}
\label{alg:HNN_NUTS2}
\begin{algorithmic}[1]
\STATE{function $\textrm{\textbf{BuildTree}}(\pmb{q},\pmb{p},u,\nu,j,\Delta t,\mathbbm{1}_{lf})$}
\IF{$j=0$}
\STATE{Base case taking one leapfrog step}
\STATE{$\pmb{q}^\prime,\pmb{p}^\prime \gets$ Algorithm \ref{alg:HNN_eval} with initial conditions: $\pmb{z}(0)=\{\pmb{q},~\pmb{p}\}$, Steps: $1$; End Time: $\Delta t$}
\STATE{$\mathbbm{1}_{lf} \gets \mathbbm{1}_{lf}~\textrm{\textbf{or}}~\mathbbm{1}\big[H\big(\pmb{q}^\prime,~\pmb{p}^\prime\big)+\ln{u} > \Delta_{max}^{hnn}\big]$}
\STATE{$s^\prime \gets \mathbbm{1}\big[H\big(\pmb{q}^\prime,~\pmb{p}^\prime\big)+\ln{u} \leq \Delta_{max}^{hnn}\big]$}
\IF{$\mathbbm{1}_{lf} = 1$}
\STATE{$\pmb{q}^\prime,\pmb{p}^\prime \gets$ Leapfrog integration with initial conditions: $\pmb{z}(0)=\{\pmb{q},~\pmb{p}\}$, Steps: $1$; End Time: $\Delta t$}
\STATE{$s^\prime \gets \mathbbm{1}\big[H\big(\pmb{q}^\prime,~\pmb{p}^\prime\big)+\ln{u} \leq \Delta_{max}^{lf}\big]$}
\ENDIF
\STATE{$n^\prime \gets \mathbbm{1}\big[u \leq \exp{\{-H\big(\pmb{q}^\prime,~\pmb{p}^\prime\big)\}}\big]$}
\STATE{\textbf{return} $\pmb{q}^\prime,~\pmb{p}^\prime,\pmb{q}^\prime,~\pmb{p}^\prime,\pmb{q}^\prime,\pmb{p}^\prime,n^\prime,s^\prime,\mathbbm{1}_{lf}$}
\ELSE 
\STATE{Recursion to build left and right sub-trees (follows from Algorithm 3 in \cite{Hoffman2014}, with $\mathbbm{1}_{lf}$ additionally passed to and retrieved from every \textrm{\textbf{BuildTree}} evaluation)}
\STATE{\textbf{return} $\pmb{q}^-,~\pmb{p}^-,\pmb{q}^+,~\pmb{p}^+,\pmb{q}^\prime,\pmb{p}^\prime,n^\prime,s^\prime,\mathbbm{1}_{lf}$}
\ENDIF
\end{algorithmic}
\end{algorithm}
\normalsize

\begin{proposition}\label{prop_1}
  With $\Delta_{max}^{hnn} \to -\infty$ and $N_{lf} = C$, where $C$ is an integer greater than 0, L-HNNs in NUTS with online error monitoring converges to traditional NUTS.
\end{proposition}
\begin{proof}
  It is straightforward to show this by virtue of lines 8--13 in Algorithm \ref{alg:HNN_NUTS1} and lines 7--10 in Algorithm \ref{alg:HNN_NUTS2}. As $\Delta_{max}^{hnn} \to -\infty$, $\mathbbm{1}_{lf} = 1$ and Algorithms \ref{alg:HNN_NUTS1} and \ref{alg:HNN_NUTS2} simply never call L-HNNs for predicting the Hamiltonian dynamics, but instead always rely on the leapfrog integration with numerical gradients, thus converging to traditional NUTS.
\end{proof}

\section{Sampling from a 3-D Rosenbrock density}\label{sec:Rosen3D}

A 3-D Rosenbrock density \cite{Pagani2019a} is considered to demonstrate the performance differences between HNNs and L-HNNs, and to investigate the performance of L-HNNs considering different amounts of training data. A 3-D Rosenbrock density makes for a particularly interesting case study, due to its heavy tails. The density function is defined by \cite{Pagani2019a}:
\begin{equation}
    \label{eqn:Rosen_Dens}
    f(\pmb{q}) \propto \exp{\Big(-\sum_{i=1}^{N-1}\big[100 (q_{i+1}-q_i^2)^2+(1-q_i)^2\big]/20\Big)}
\end{equation}

We trained HNNs and L-HNNs with the same training data, which contained 40 samples from HMC, each with an end time of $T=40$ using a $\Delta t = 0.025$. Both the HNNs and L-HNNs had three hidden layers with 100 neurons each. A total of $100,000$ steps at a learning rate of $5E-4$ was used for training the HNNs and L-HNNs. A sine activation function was used, per the recommendations of Mattheakis et al. \cite{Mattheakis2022a}. Efficient NUTS with online error monitoring (i.e., Algorithms \ref{alg:HNN_NUTS1} and \ref{alg:HNN_NUTS2}) was employed to simulate 35,000 samples from the target density. HNNs and L-HNNs were independently used to simulate the Hamiltonian dynamics within the \textbf{BuildTree(.)} function. $\Delta_{max}^{hnn}$ was set to 10---and when exceeded, the leapfrog scheme with numerical gradients of the target density was set to take over---and $N_{lf}$ was set to 20.

Figures \ref{3D_Rosen_q3_1} and \ref{3D_Rosen_q3_2} present trace plots for the dimension $q_1$, obtained using HNNs and L-HNNs, respectively. Note that, whether using HNNs or L-HNNs, efficient NUTS with online error monitoring mitigates degeneracy in sampling, compared to Figure \ref{Naive_NUTS_1}. Perhaps more important are the integration errors when predicting the Hamiltonian dynamics. Figures \ref{3D_Rosen_error_1} and \ref{3D_Rosen_error_2} present these errors---defined as $[\varepsilon \equiv H(\pmb{q}^\prime, \pmb{p}^\prime) + \ln{u}]$ in Algorithm \ref{alg:HNN_NUTS2}---as obtained using HNNs and L-HNNs, respectively. The errors when using HNNs are substantially larger across the sample indices, compared to using L-HNNs. Large integration errors lead to an increase in the number of samples for which the leapfrog scheme with numerical gradients of the target density is called. While using HNNs, the number of samples for which traditional leapfrog is used is 4,706; using L-HNNs, this number reduces to only 180. In total, HNNs and L-HNNs, require an additional 1,437,900 and 63,598 numerical gradients of the target density, respectively (above what is required to generate the training data)---a substantial difference.

\begin{figure}[htbp]
\begin{subfigure}{0.48\textwidth}
\centering  
\includegraphics[scale=0.42]{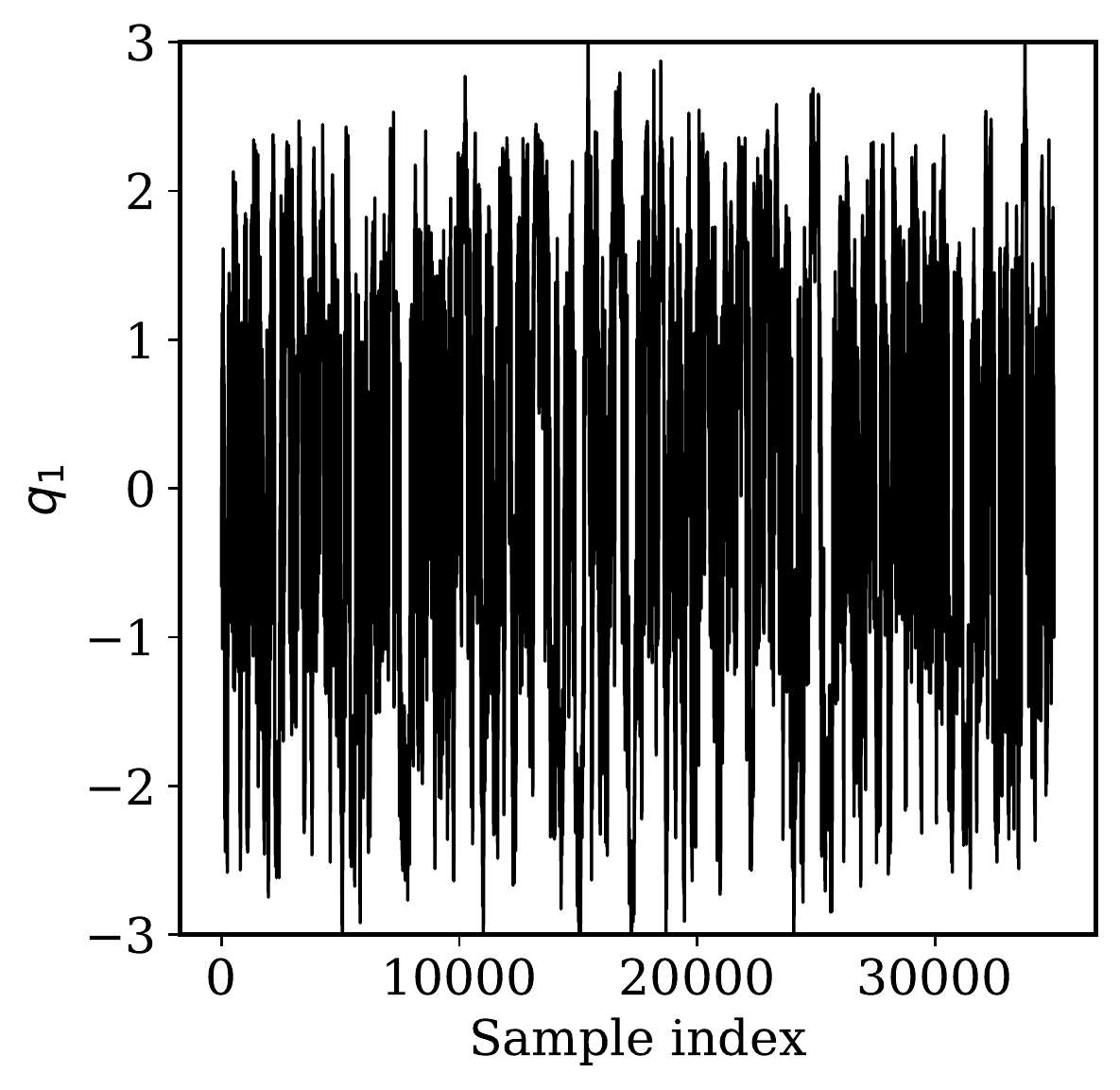}
\caption{}
\label{3D_Rosen_q3_1}
\end{subfigure}
\begin{subfigure}{0.48\textwidth}
\centering  
\includegraphics[scale=0.42]{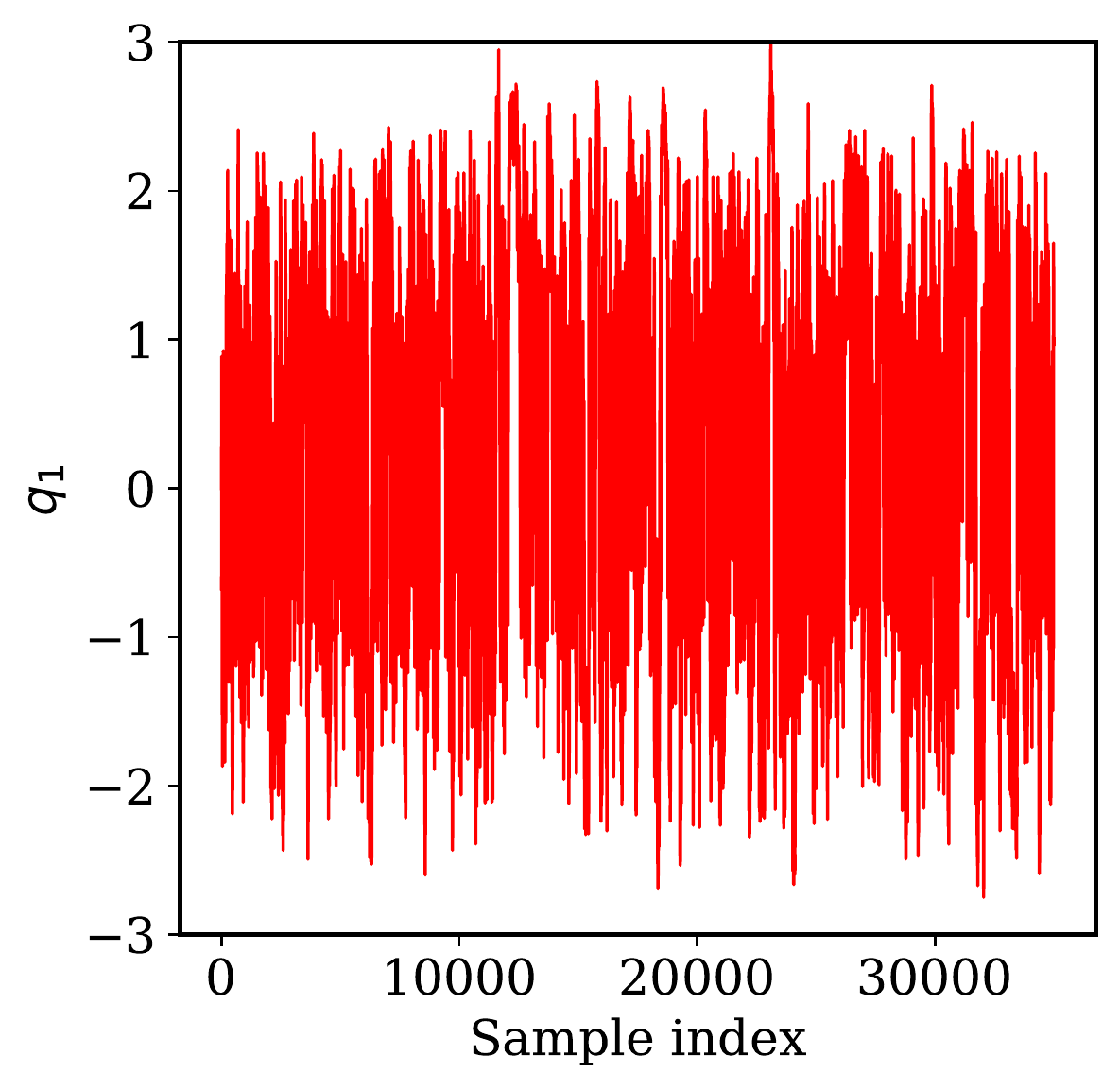} 
\caption{}
\label{3D_Rosen_q3_2}
\end{subfigure}
\caption{Trace plots of dimension $q_1$ for a 3-D Rosenbrock density simulated via efficient NUTS with online error monitoring using (a) HNNs and (b) L-HNNs. Note that the sort of degeneracy presented in Figure \ref{Naive_NUTS} is absent due to the online error monitoring scheme.}
\label{3D_Rosen_q3}
\end{figure}

\begin{figure}[htbp]
\begin{subfigure}{0.48\textwidth}
\centering  
\includegraphics[scale=0.42]{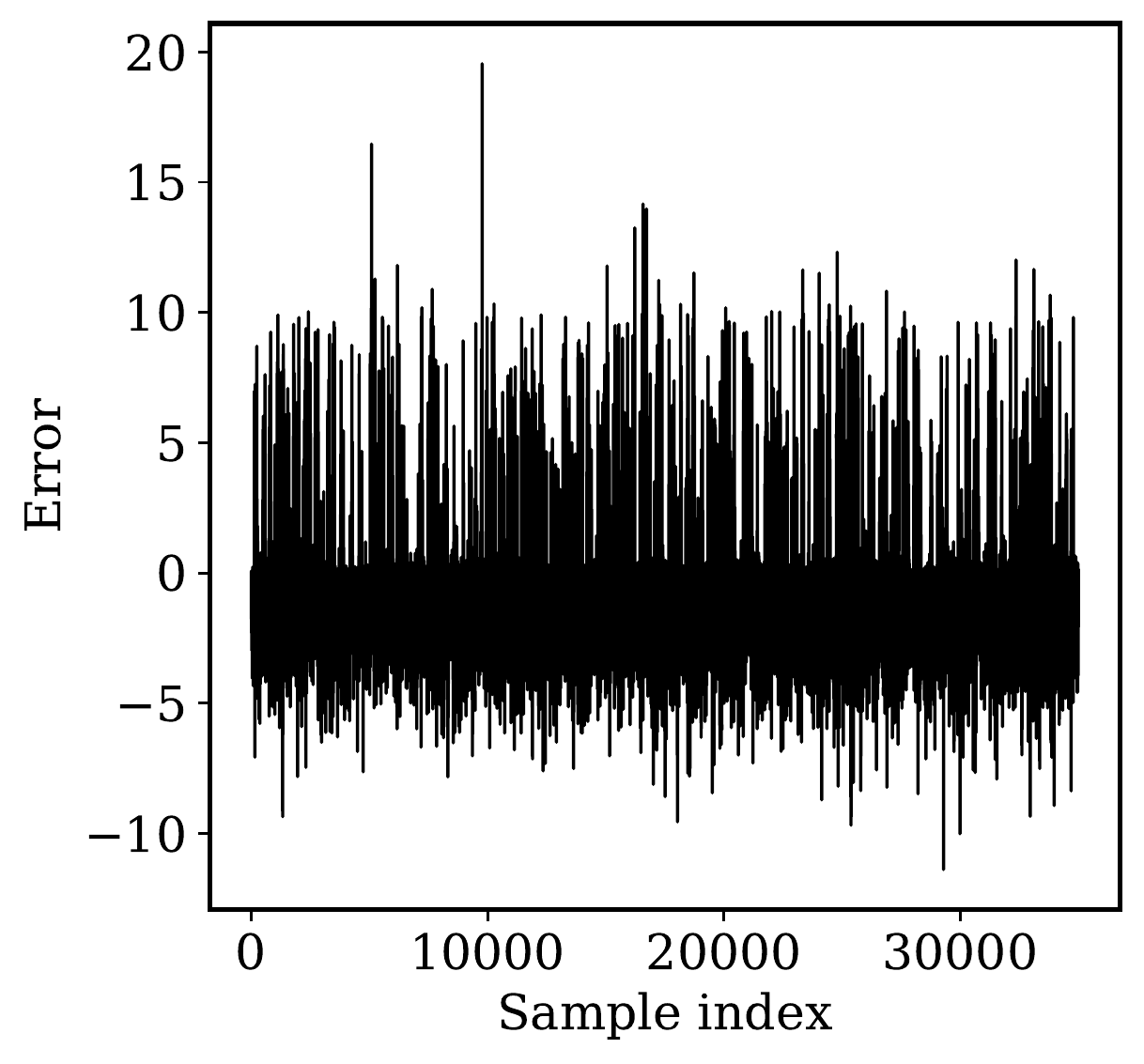}
\caption{}
\label{3D_Rosen_error_1}
\end{subfigure}
\begin{subfigure}{0.48\textwidth}
\centering  
\includegraphics[scale=0.42]{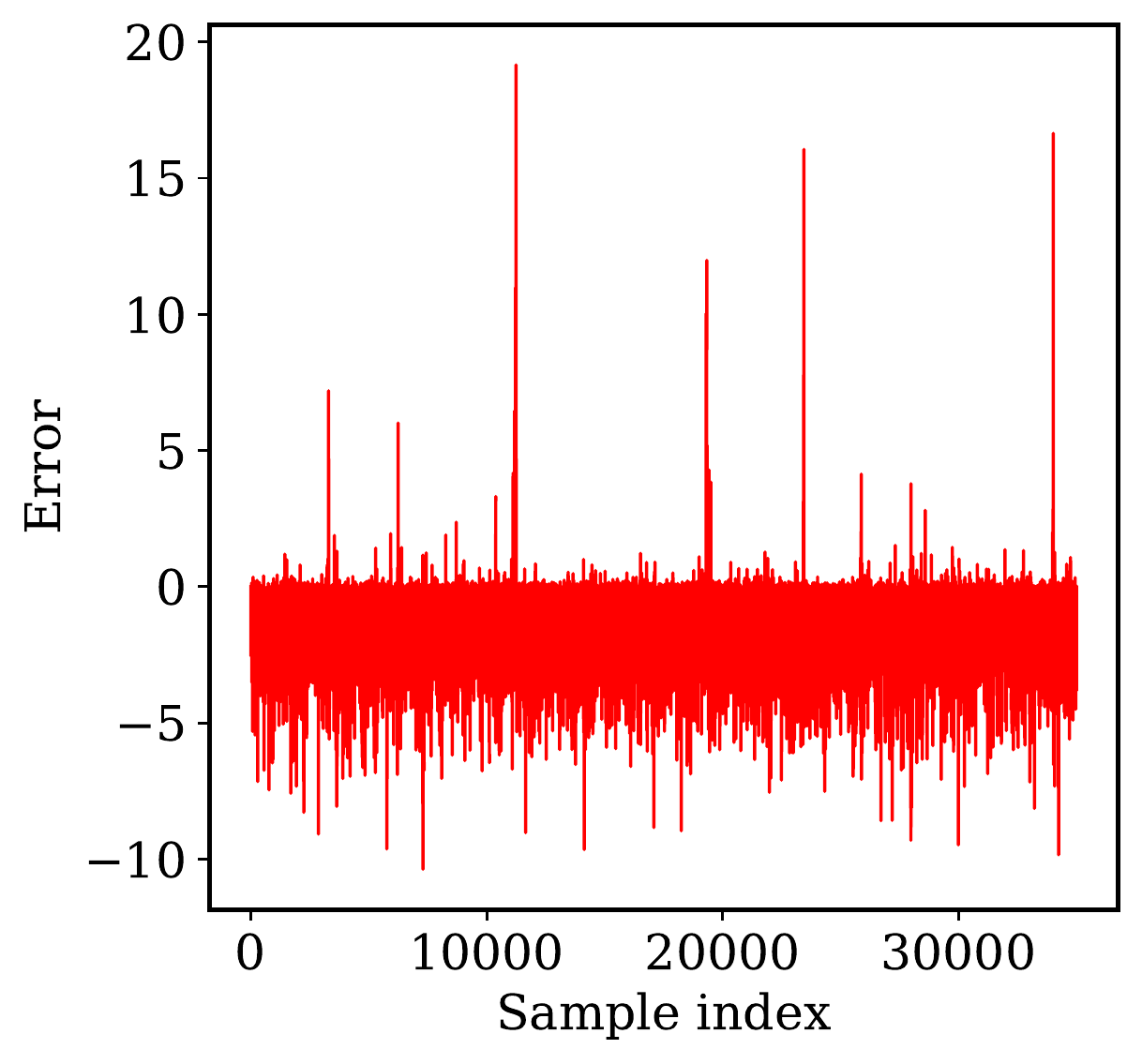} 
\caption{}
\label{3D_Rosen_error_2}
\end{subfigure}
\caption{L-HNNs error metric, as defined by the quantity $[\varepsilon \equiv H\big(\pmb{q}^\prime,~\pmb{p}^\prime\big)+\ln{u}]$ in Algorithm \ref{alg:HNN_NUTS2}, while generating samples from a 3-D Rosenbrock density using (a) HNNs and (b) L-HNNs in efficient NUTS with online error monitoring. These error values are considerably lower when using L-HNNs than HNNs.}
\label{3D_Rosen_error}
\end{figure}

To investigate the influence of training data set size, we consider three training datasets for L-HNNs. All three sets include $M_t=40$ samples from the 3-D Rosenbrock density simulated using traditional HMC, but with different end times --- $T=100$, 150, and 250 units --- and $\Delta t = 0.025$. We simulated $125,000$ samples from the 3-D Rosenbrock density, with the first $5,000$ samples considered as burn-in. For efficient NUTS with online error monitoring, we set $\Delta_{max}^{hnn} = 10$ and $N_{lf} = 20$. Figure \ref{3D_Rosen_ecdf} presents empirical cumulative distribution function (eCDF) plots for the three dimensions. In this figure, LHNN-NUTS 1, LHNN-NUTS 2, and LHNN-NUTS 3 correspond to end times of $T=100$, 150, and 250 units, respectively. Note that irrespective of the amount of training data, in general, the eCDFs of LHNN-NUTS with online error monitoring closely match the eCDFs simulated using traditional NUTS with numerical gradients of the target density. Moreover, LHNN-NUTS 2 and LHNN-NUTS 3, which were trained on HMC samples with longer Hamiltonian trajectories, almost exactly match the reference eCDFs, as compared to LHNN-NUTS 1, which shows slight deviations in dimensions $q_2$ (Figure \ref{3D_Rosen_ecdf_2}) and $q_3$ (Figure \ref{3D_Rosen_ecdf_3}).

\begin{figure}[htbp]
\begin{subfigure}{0.32\textwidth}
\centering  
\includegraphics[width=\textwidth]{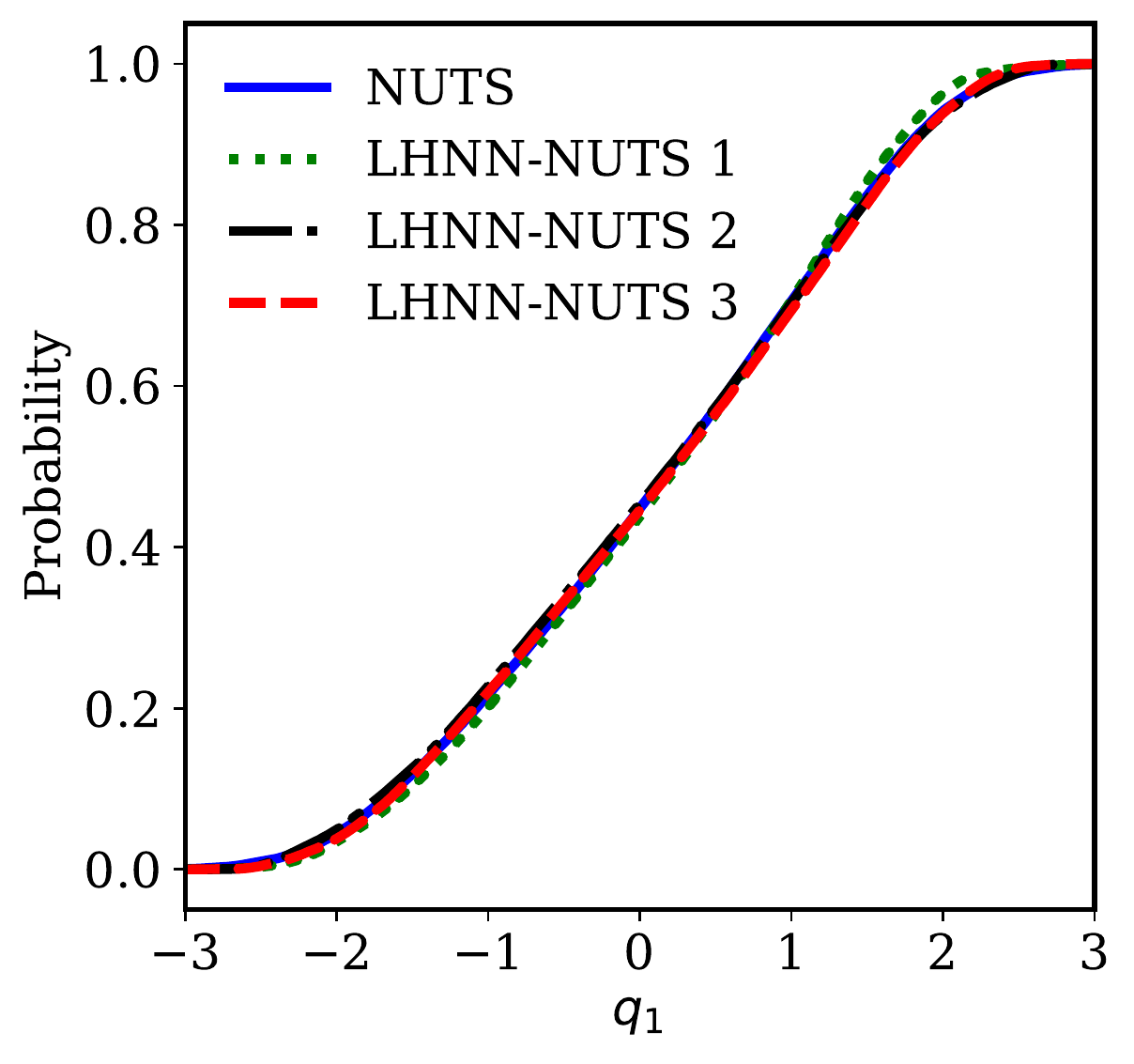} 
\caption{}
\label{3D_Rosen_ecdf_1}
\end{subfigure}
\begin{subfigure}{0.32\textwidth}
\centering  
\includegraphics[width=\textwidth]{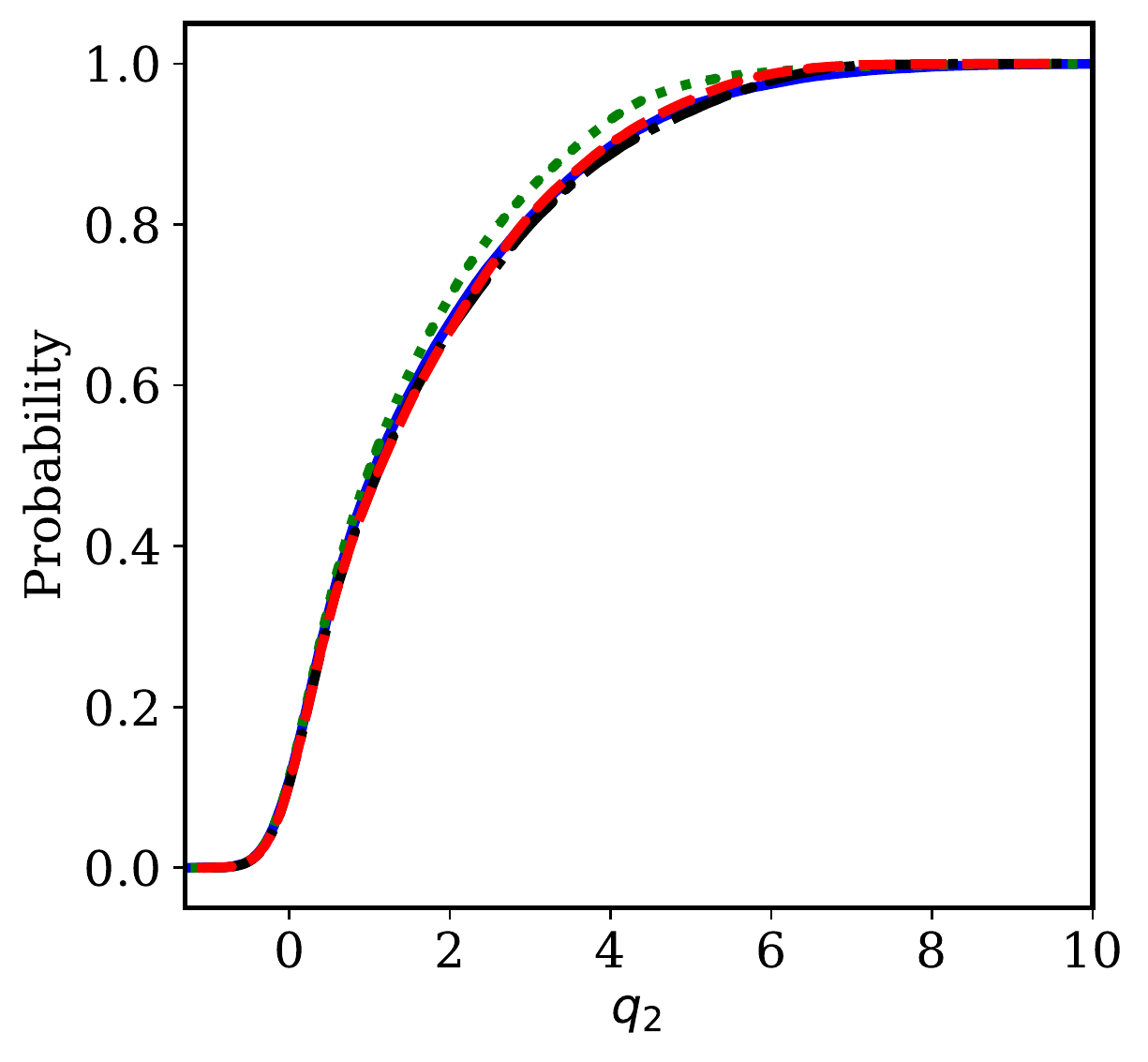} 
\caption{}
\label{3D_Rosen_ecdf_2}
\end{subfigure}
\begin{subfigure}{0.32\textwidth}
\centering  
\includegraphics[width=\textwidth]{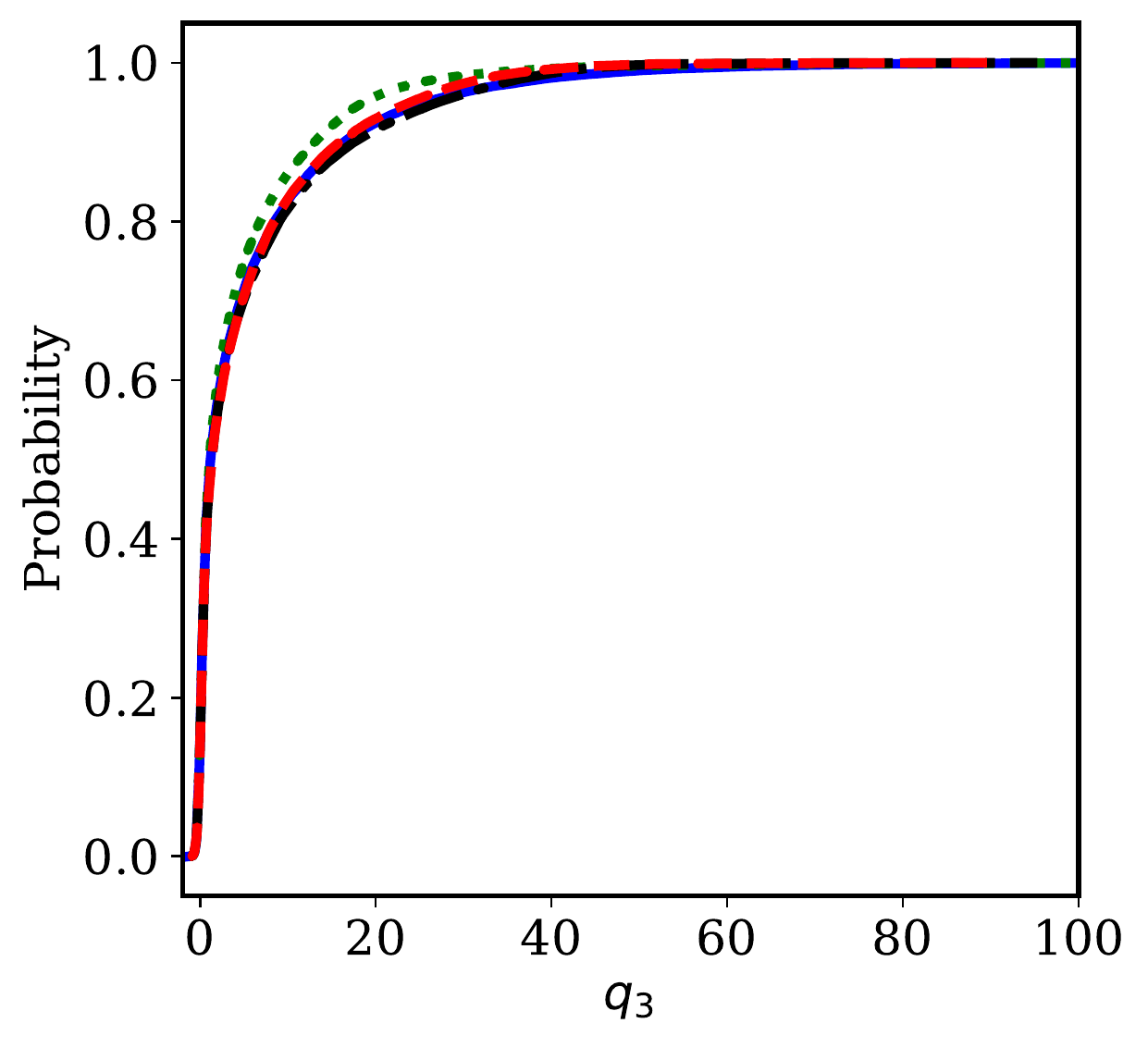} 
\caption{}
\label{3D_Rosen_ecdf_3}
\end{subfigure}
\caption{Empirical CDF plots of samples from a 3-D Rosenbrock density simulated via NUTS and L-HNNs in NUTS with online error monitoring, using different training datasets. eCDF plots are presented for the dimensions (a) $q_1$, (b) $q_2$, and (c) $q_3$. Details of the L-HNN training data are given in Table \ref{Table_Rosen_Data}.}
\label{3D_Rosen_ecdf}
\end{figure}

Table \ref{Table_Rosen_Data} compares the performance of the traditional NUTS with the LHNN-NUTS with online error monitoring in terms of ESS, number of gradients of the target density, and average ESS across all the dimensions per gradient of the target density. In short, all LHNN-NUTS results show significant improvement over the traditional NUTS method in terms of required gradient evaluations, while accurately matching the target eCDFs (Figure \ref{3D_Rosen_ecdf}) and producing samples with comparable ESS values. Even for short training trajectories with $T=100$, the total number of gradient evaluations is reduced by 83\% and the average ESS per gradient evaluation increases by 5 times. For longer training trajectories, the number of gradient evaluations reduces from approximately 16 million to only approximately 600,000 while the ESS per gradient evaluation improves by more than an order of magnitude relative to traditional NUTS. It is worth noting, however, that there is no noticeable improvement in performance by increasing $T$ from 150 to 250 for the training data generation, which implies that there is a critical length $T$ for the training data. 

\begin{table}[h]
\centering
\caption{Performance comparison between traditional NUTS and HNN-NUTS with different training data characteristics, considering a 3-D degenerate Rosenbrock density.}
\label{Table_Rosen_Data}
\small
\begin{tabular}{ |c|c|c|c| }
\hline
 & \textbf{ESS}  & \textbf{\Centerstack[c]{$\#$ gradients of \\ target density}} & \textbf{\Centerstack[c]{Avg. ESS per  \\ gradient of \\ target density}} \\
\hline
\Centerstack[c]{NUTS \\ \cite{Hoffman2014}} & \Centerstack[c]{(1396.59, 2338.99, \\ 2411.40)}  & $15,899,976$ & 0.000128\\
\hline
\Centerstack[c]{$M_t = 40; T = 100$ \\ [LHNN-NUTS 1]} & \Centerstack[c]{(1339.13, 1665.76, \\ 1648.85)}  & \Centerstack[c]{$2,711,240$ \\ Training: 160,000 \\ Evaluation: 2,551,240} & 0.000572\\
\hline
\Centerstack[c]{$M_t = 40; T = 150$ \\ [LHNN-NUTS 2]} & \Centerstack[c]{(1217.63, 1710.97, \\ 1654.11)}  & \Centerstack[c]{$607,298$ \\ Training: 240,000 \\ Evaluation: 367,298} & 0.00251\\
\hline
\Centerstack[c]{$M_t = 40; T = 250$ \\ [LHNN-NUTS 3]} &  \Centerstack[c]{(1445.03, 1729.98, \\ 1387.91)} & \Centerstack[c]{$642,414$ \\ Training: 400,000 \\ Evaluation: 242,414} & 0.00236\\
\hline
\end{tabular}
\begin{tablenotes}
\small
\item[-] {- \textbf{Notations.} $M_t$: Number of training samples; $T$: End time for trajectory}
\item[-] {- Avg. ESS represents the ESS values averaged across the three dimensions.}
\item[-] {- Evaluation in the table implies the use of L-HNNs in NUTS with online error monitoring}
\end{tablenotes}
\end{table}
\normalsize

For a small value of $M_t$, as expected, L-HNNs in NUTS would require more numerical gradients of the target density during evaluation because the training data does not adequately capture the randomization of the momenta. But for a large value of $M_t$, although the additional numerical gradients of the target density may become small, the upfront cost for generating the training data may be high. We found from the 3-D Rosenbrock example that fixing $M_t=40$ provides a balance between the upfront costs for generating the training data and additional numerical gradients required during the evaluation. Therefore, we use $M_t=40$ for the case study examples to be discussed in Section \ref{sec:results}.

\section{Results}\label{sec:results}

We considered five case studies to demonstrate the LHNN-NUTS with online error monitoring and compared its performance to the traditional NUTS. Table \ref{Table_Summary} summarizes the main performance metrics for the five case studies. In the following sections, the case studies are defined along with the training details for L-HNNs, results from Table \ref{Table_Summary} are discussed, and plots of the samples from the target densities are presented. 

\begin{table}[h!]
\centering
\caption{Summary of the performance of LHNN-NUTS with online error monitoring across the different case studies considered.}
\label{Table_Summary}
\small
\begin{tabular}{ |c|c|c|c| }
\hline
 & \textbf{ESS}  & \Centerstack[c]{\textbf{\# gradients} \\ \textbf{of target} \\ \textbf{density}} & \textbf{\Centerstack[c]{Avg. ESS per  \\ gradient of \\ target density}}\\
\hline
\multicolumn{4}{|c|}{\textbf{3-D degenerate Rosenbrock}}\\
\hline
LHNN-NUTS & \Centerstack[c]{$(1445.03, 1729.98,$ \\ $1387.91$)} & \Centerstack[c]{$642,414$ \\ Training: $400,000$\\ Evaluation: $242,414$} & $0.00236$\\
\hline
NUTS & \Centerstack[c]{$(1396.59, 2338.99,$ \\ $2411.40$)} & $15,899,976 $ & $0.000128$\\
\hline
\multicolumn{4}{|c|}{\textbf{5-D ill-conditioned Gaussian}}\\
\hline
LHNN-NUTS & \Centerstack[c]{$(20000, 17741.21,$ \\ $13906.42,  8382.5,$ \\ $2763.38)$} & \Centerstack[c]{$408,800$ \\ Training: $400,000$\\ Evaluation: $8,800$} & $0.0307$\\
\hline
NUTS & \Centerstack[c]{$(20000, 17877.90,$ \\ $14645.54,  7931.68,$ \\ $2657.19)$} & $11,817,469$ & $0.00106$\\
\hline
\multicolumn{4}{|c|}{\textbf{10-D degenerate Rosenbrock}}\\
\hline
LHNN-NUTS & \Centerstack[c]{$(15535.69, 28217.40,$ \\ $28965.00,  27280.67,$ \\ $ 21268.00, 13532.28,$ \\ $8418.80, 5731.09,$ \\ $3299.92, 2045.02)$} &  \Centerstack[c]{$418,936$ \\ Training: $400,000$\\ Evaluation: $18,936$} & $0.0368$\\
\hline
NUTS & \Centerstack[c]{$(15771.46, 27125.5,$ \\ $27825.9,  25107.02,$ \\ $ 21963.9, 13553.21,$ \\ $8548.44, 5631.39,$ \\ $2921.78, 1698.69)$} &  $7,015,025$ & $0.00219$\\
\hline
\multicolumn{4}{|c|}{\textbf{24-D Bayesian logistic regression}}\\
\hline
LHNN-NUTS & $97254.18$ &  \Centerstack[c]{$400,000$ \\ Training: $400,000$\\ Evaluation: --} & $0.243$ \\
\hline
NUTS & $97702.96$ & $1,256,229$ & $0.0777$\\
\hline
\multicolumn{4}{|c|}{\textbf{100-D rough well}}\\
\hline
HNN-NUTS & $2744.97$ & \Centerstack[c]{ $422,272$\\ Training: $400,000$\\ Evaluation: $22,272$} & $0.0065$\\
\hline
NUTS &  $4775.33$ & $1,279,873$ & $0.00373$\\
\hline
\multicolumn{4}{|c|}{\textbf{2-D Neal's funnel}}\\
\hline
LHNN-NUTS & $(1019.87,666.73)$ &  \Centerstack[c]{$3,596,688$ \\ Training: $400,000$\\ Evaluation: $3,196,688$} & $0.000234$\\
\hline
NUTS & $(1012.45,718.07)$ &  $16,639,072$ & $0.000052$\\
\hline
\end{tabular}
\begin{tablenotes}
\small
\item[-] - For dimension greater than 10, only the average ESS across all dimensions is reported.
\item[-] {- Evaluation in the table implies the use of L-HNNs in NUTS with online error monitoring.}
\end{tablenotes}
\end{table}
\normalsize

\subsection{Case studies}

We considered the following five case studies. Corresponding results are presented herein. For all cases, except where otherwise noted, the L-HNNs were trained using $M_t=40$ HMC samples each, with an end time of $T=250$ units and a step size of $\Delta t=0.025$. The architecture of the L-HNNs was composed of three hidden layers, each with 100 neurons, and a sine activation function \cite{Mattheakis2022a}. For training, a total of 100,000 steps were used at a learning rate of $5E-4$ and $\Delta_{max}^{hnn}$ and $N_{lf}$ in LHNN-NUTS with online error monitoring were set to 10 and 20, respectively. 

\subsubsection{5-D ill-conditioned Gaussian}\label{sec:5dGauss_example}

A 5-D ill-conditioned Gaussian distribution from Levy et al. \cite{Levy2017a} with a zero mean vector and a diagonal covariance matrix was considered. The diagonal elements of the covariance matrix scale in log-linear fashion, as presented below: 
\begin{equation}
    \label{eqn:5D_Gauss1}
    \textrm{diag}(\pmb{\Sigma}) = [0.01,0.1,1.0,10.0,100.0]
\end{equation}

In this example, the L-HNNs had 10 inputs (five position variables and five momenta) and predicted five latent variables, whose sum was treated as the Hamiltonian, as discussed in Section \ref{sec:lhnns}. A total of $25,000$ samples were simulated from the target density, with the first $5,000$ samples treated as burn-in. 

From Table \ref{Table_Summary}, we see that using LHNN-NUTS reduces the number of gradient evaluations by a factor of nearly 30, while improving the ESS per gradient evaluation by more than an order of magnitude. Moreover, as we can see from Figure \ref{5D_IllGauss}, which shows scatter plots of the samples from the 5-D ill-conditioned Gaussian using LHNN-NUTS and traditional NUTS, the LHNN-NUTS with online error monitoring satisfactorily samples from the target density. This is further verified by the ESS values in Table \ref{Table_Summary}, which are consistent between the two methods.
\begin{figure}[htbp]
\centering
\includegraphics[width=0.9\textwidth]{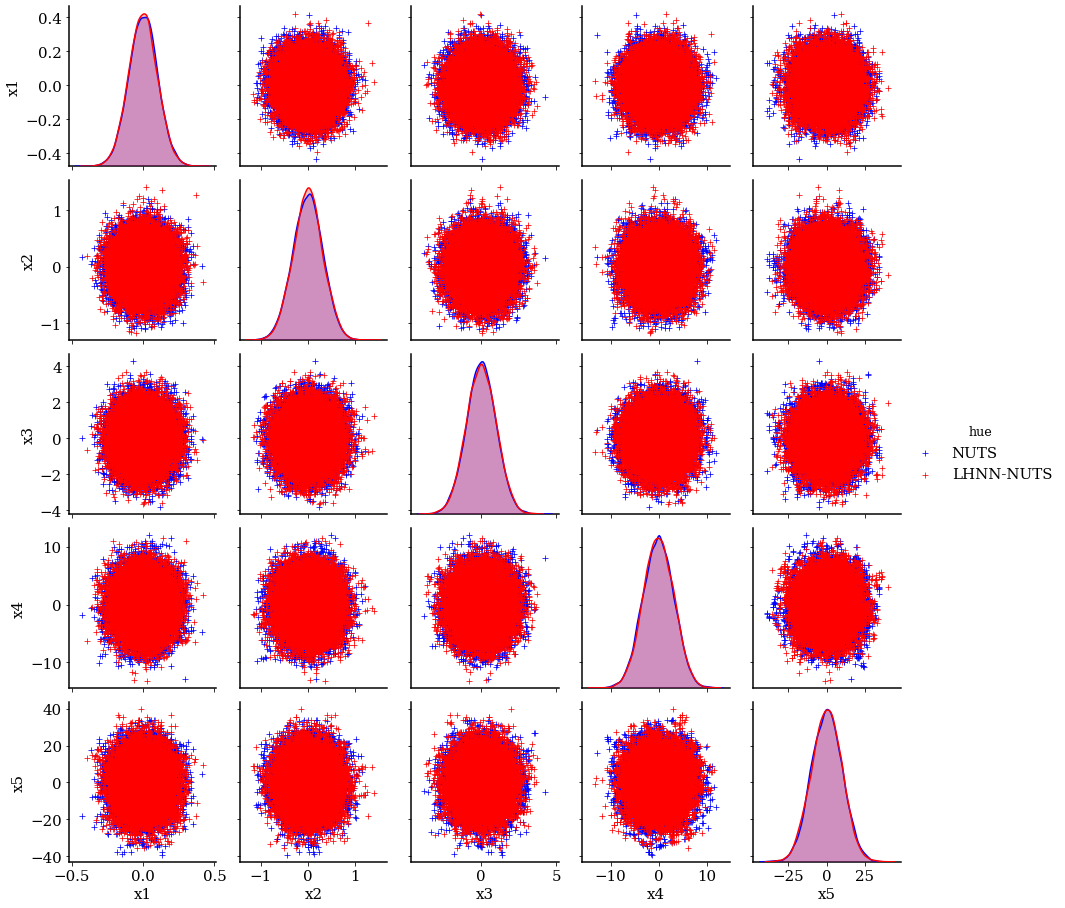} 
\caption{Scatter plot comparison between NUTS and LHNN-NUTS with online error monitoring for a 5-D ill-conditioned Gaussian distribution. (Blue dots: traditional NUTS; Red dots: LHNN-NUTS with online error monitoring)}
\label{5D_IllGauss}
\end{figure}

\subsubsection{10-D degenerate Rosenbrock density}

Here, A 10-D version of the Rosenbrock density in Equation \eqref{eqn:Rosen_Dens} was considered. The L-HNNs had 20 inputs (10 position variables and 10 momenta) and predicted 10 latent variables, whose sum was treated as the Hamiltonian. Approximately $125,000$ samples were simulated from the target density, with the first $5,000$ samples treated as burn-in. 

From Table \ref{Table_Summary}, we again see drastic reduction in the number of gradient evaluations (approximately a factor of 17) and more than an order of magnitude reduction in ESS per gradient evaluation. Figure \ref{10D_Rosen} presents scatter plots of the samples generated with LHNN-NUTS and the traditional NUTS. Although for dimensions $q_9$ and $q_{10}$, traditional NUTS included a few more samples in the tails than LHNN-NUTS, the scatter plots overall match satisfactorily. Moreover, the eCDF plots for dimensions $q_9$ and $q_{10}$ from LHNN-NUTS closely matched the reference eCDFs.
\begin{figure}[htbp]
\centering
\includegraphics[width=0.9\textwidth]{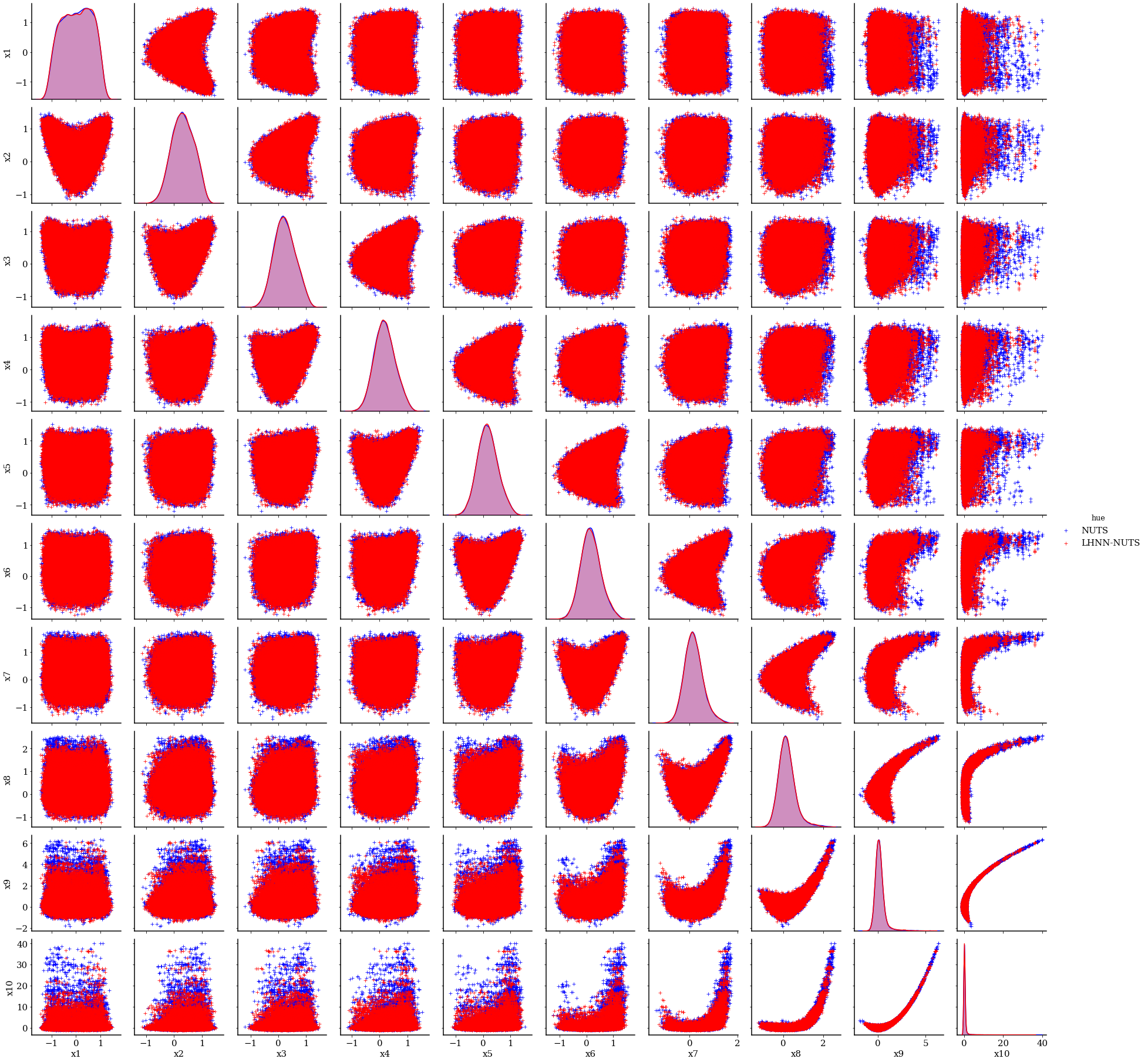} 
\caption{Scatter plot comparison between NUTS and LHNN-NUTS with online error monitoring, considering a 10-D degenerate Rosenbrock distribution. (Blue dots: traditional NUTS; Red dots: LHNN-NUTS with online error monitoring)}
\label{10D_Rosen}
\end{figure}



\subsubsection{24-D Bayesian logistic regression}

Next, a 24-D Bayesian logistic regression on the German credit dataset \cite{Dua2019a} was performed. The likelihood function is given by:
\begin{equation}
    \label{eqn:BLR_1}
    \mathcal{L}(\pmb{q}|y_i,~\pmb{z}_i) \propto 1/\big(1+\exp{(-y_i \pmb{q}^{\top} \pmb{z}_i)}\big)
\end{equation}
where $\pmb{z}_i$ are the predictors, $y_i$ is the response for the data point $i$ and $\pmb{q}$ are the parameters to be inferred. All the predictor variables were normalized to have zero mean and unit variance. The prior distribution of the parameters was assumed to follow a multivariate Gaussian distribution with identity covariance:
\begin{equation}
    \label{eqn:BLR_2}
    g(\pmb{q}) = \mathcal{N}(\pmb{0}, \pmb{I})
\end{equation}
Given the likelihood and prior, the negative log posterior (i.e., the potential energy function) is given by:
\begin{equation}
    \label{eqn:BLR_3}
    U(\pmb{q}) = -\ln f(\pmb{q}|\{\pmb{z}_i,~y_i\}_{i=1}^K) = \sum_{i=1}^K \ln \Big(1/\big(1+\exp{(-y_i \pmb{q}^{\top} \pmb{z}_i)}\big)\Big) + ||\pmb{q}||^2/(2K)
\end{equation}
where $K$ is the number of data points.  

The L-HNNs had 48 inputs (24 positions and 24 momenta) and predicted 24 latent variables, whose sum was treated as the Hamiltonian. A total of $125,000$ samples were simulated from the target density, with the first $5,000$ samples treated as burn-in.

In this case, Table \ref{Table_Summary} shows that no additional gradient evaluations were necessary due to the error monitoring step; hence all gradient evaluations occurred during training. Again, a huge savings was achieved in the total number of gradient evaluations and the ESS per gradient evaluation although the traditional NUTS required far fewer evaluations than in the previous examples. 

Owing to the high dimensionality, the posterior mean and standard deviations of the 24 parameters are presented in Figures \ref{24D_logistic_1} and \ref{24D_logistic_2}, respectively. The statistics predicted by LHNN-NUTS satisfactorily match the reference values obtained using traditional NUTS.
\begin{figure}[htbp]
\begin{subfigure}{0.48\textwidth}
\centering  
\includegraphics[width=\textwidth]{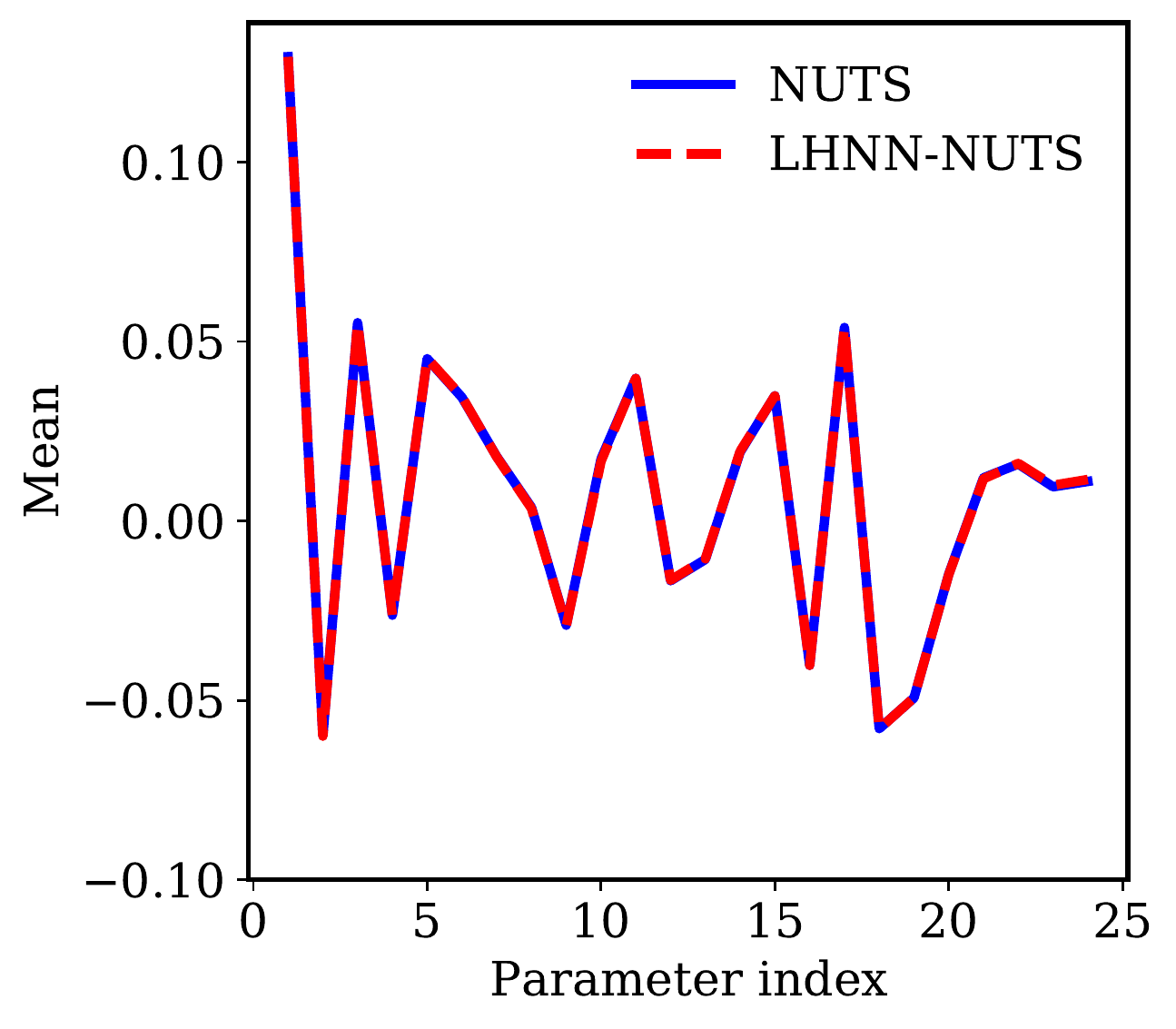} 
\caption{}
\label{24D_logistic_1}
\end{subfigure}
\begin{subfigure}{0.48\textwidth}
\centering  
\includegraphics[width=\textwidth]{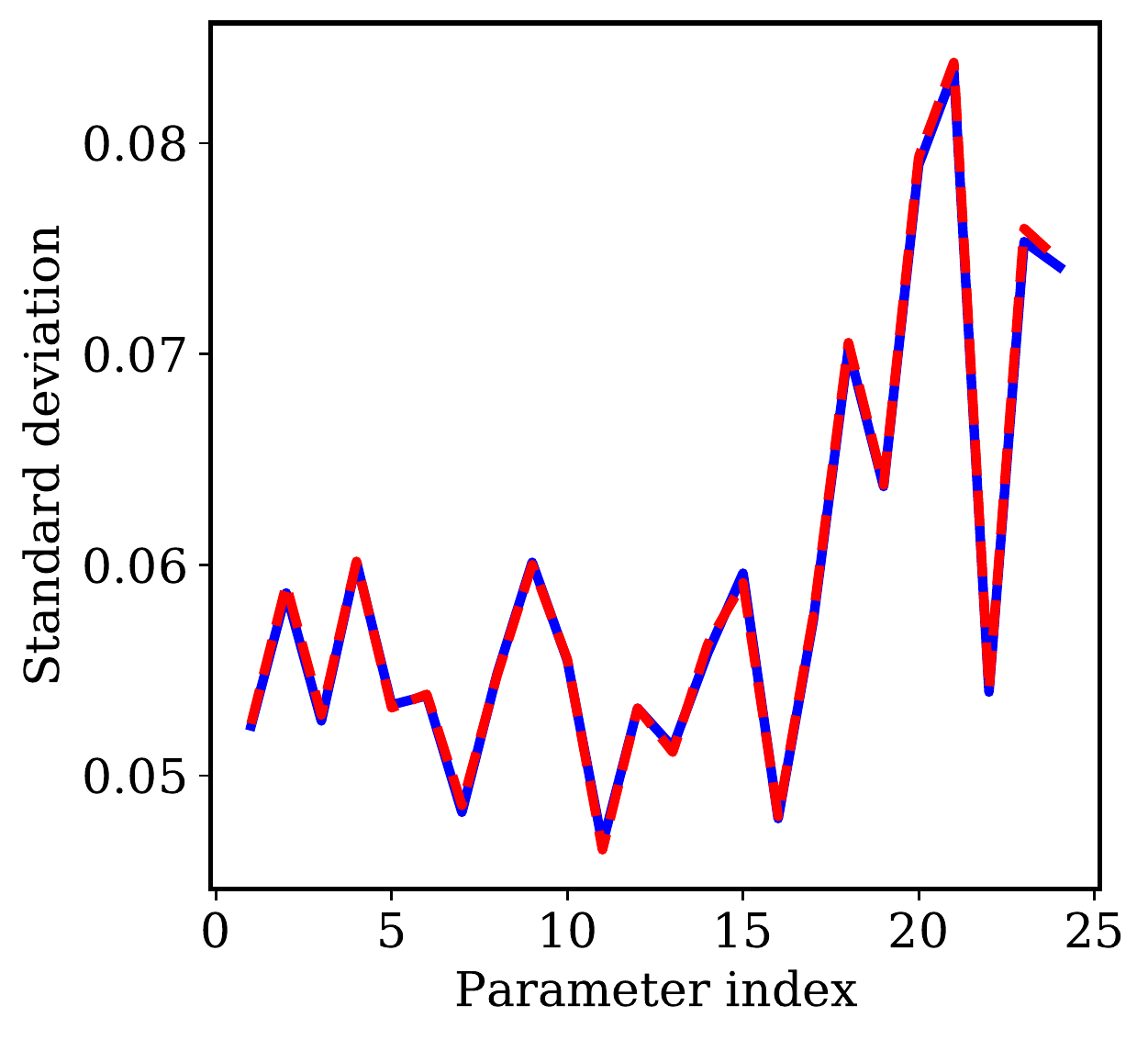} 
\caption{}
\label{24D_logistic_2}
\end{subfigure}
\caption{(a) Means and (b) standard deviations of the inferred parameters for the 24-D Bayesian logistic regression problem.}
\label{24D_logistic}
\end{figure}

\subsubsection{100-D rough well}\label{sec:100D_RW}

A 100-D rough well distribution from Levy et al. \cite{Levy2017a} and Gu and Sun \cite{Gu2020a} was considered:
\begin{equation}
    \label{eqn:100D_RW}
    f(\pmb{q}) \propto \exp \Big\{-\frac{1}{2}~\pmb{q}^\top \pmb{q}\ - \eta \sum_i \cos\big(\frac{q_i}{\eta}\big)\Big\}
\end{equation}
with $\eta = 0.01$. Levy et al. \cite{Levy2017a} point out that for a small $\eta$, although the cosine term alters the potential energy negligibly, it perturbs the gradient by a high frequency noise oscillation between -1 and 1. 

The architecture of the L-HNNs was composed of three hidden layers, each with 500 neurons. For training, a total of {100,000} steps at a learning rate of $1E-4$ were used. $\Delta_{max}^{hnn}$ and $N_{lf}$ in LHNN-NUTS with online error monitoring were set to 20 and 5, respectively. All other architecture details follow those above. The L-HNNs had 200 inputs (100 positions and 100 momenta) and predicted 100 latent variables, whose sum was treated as the Hamiltonian. A total of $10,000$ samples were simulated from the target density, with the first $2,000$ samples treated as burn-in.

For this very high-dimensional example, Table \ref{Table_Summary} shows that we still achieve a very large reduction in the number of gradient evaluations, although the ESS is smaller which reduces the improvement in ESS per gradient evaluation to only a factor of $\sim 2$. Nonetheless, the method outperforms the traditional NUTS and produces satisfactory samples as shown in Figure \ref{100D_RoughWell}, which presents scatter plots for 10 of the 100 dimensions. Whereas L-HNNs in NUTS with online error monitoring takes only several hours to simulate these $10,000$ samples, traditional NUTS takes several days.
\begin{figure}[htbp]
\centering
\includegraphics[width=0.9\textwidth]{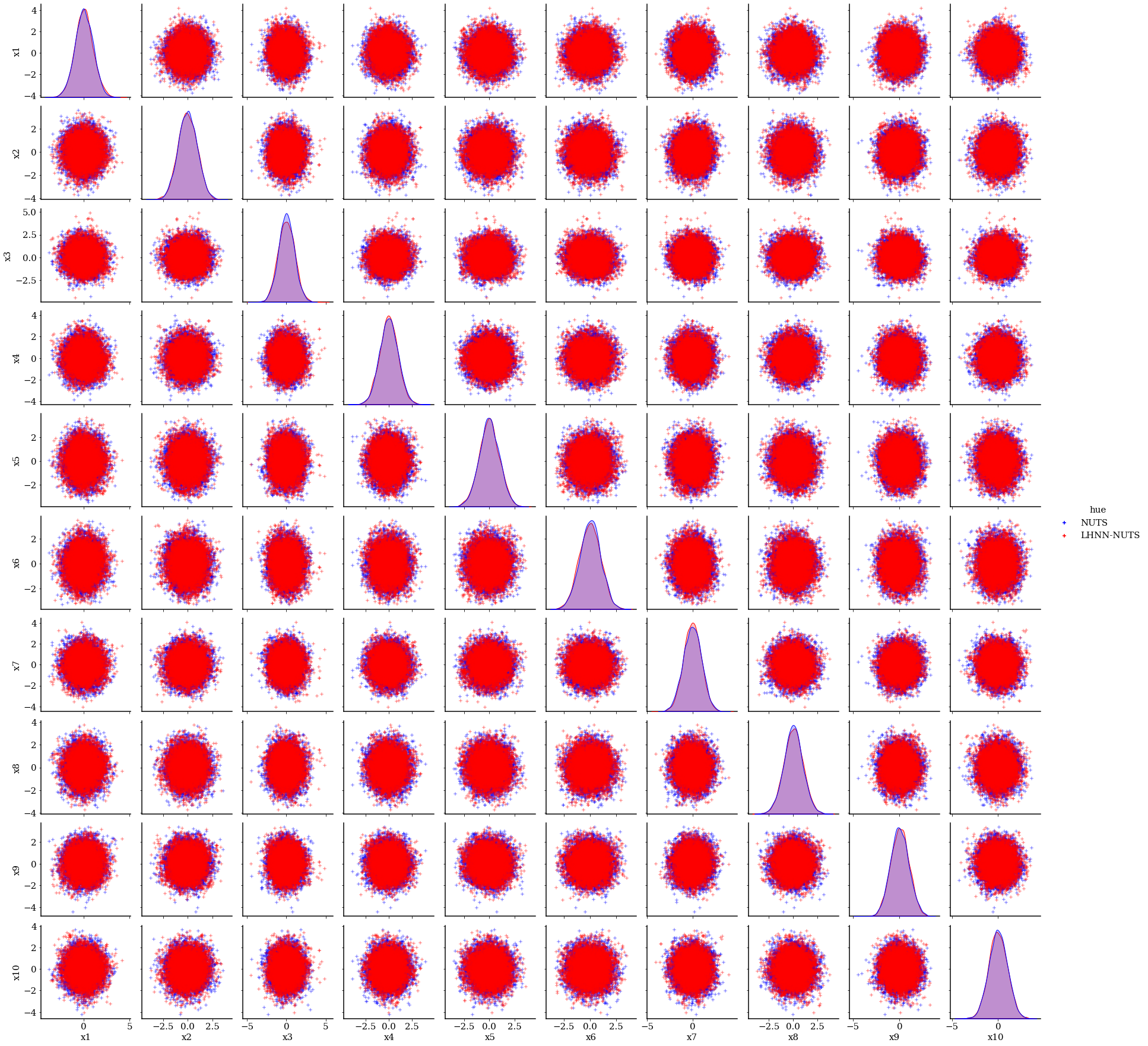} 
\caption{Scatter plot comparison between NUTS and LHNN-NUTS with online error monitoring for the 100-D rough well distribution showing 10 of the 100 dimensions.  (Blue dots: traditional NUTS; Red dots: LHNN-NUTS with online error monitoring)}
\label{100D_RoughWell}
\end{figure}

\subsubsection{2-D Neal's funnel density}

As a final example, the following 2-D Neal's funnel density was considered:
\begin{equation}
    \label{eqn:NF_1}
    f(\pmb{q}) \propto \begin{cases}\begin{aligned}
    &q_1 = \mathcal{N}(0,~3)\\
    &q_2 = \mathcal{N}(0,~\exp^{q_1})\\
    \end{aligned}\end{cases}
\end{equation}
This density can be complex to sample from, due its extreme curvature. The L-HNNs had four inputs (two positions and two momenta) and predicted two latent variables, whose sum was treated as the Hamiltonian. A total of $25,000$ samples were simulated from the target density, with the first $5,000$ samples treated as burn-in. 

In this case, the online error monitoring scheme was used extensively as seen in Table \ref{Table_Summary} due to the difficulty in training L-HNNs for this complex problem. Nonetheless, we see a reduction of more than 13 million gradient evaluations (traditional NUTS requires a huge number of gradient evaluations) and nearly an order of magitude improvement in the ESS per gradient evaluation. Figure \ref{2D_Neal_1} presents a scatter plot of the samples generated using LHNN-NUTS and traditional NUTS. Figures \ref{2D_Neal_2} and \ref{2D_Neal_3} compare the eCDF plots for the two dimensions. Note that the plots generated using LHNN-NUTS compare well with those obtained using traditional NUTS.

\begin{figure}[htbp]
\begin{subfigure}{0.32\textwidth}
\centering  
\includegraphics[width=\textwidth]{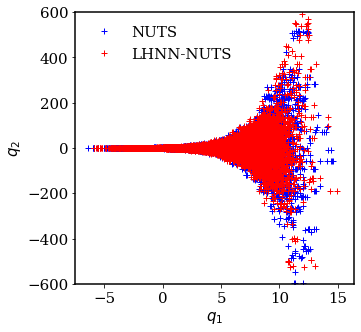} 
\caption{}
\label{2D_Neal_1}
\end{subfigure}
\begin{subfigure}{0.32\textwidth}
\centering  
\includegraphics[width=\textwidth]{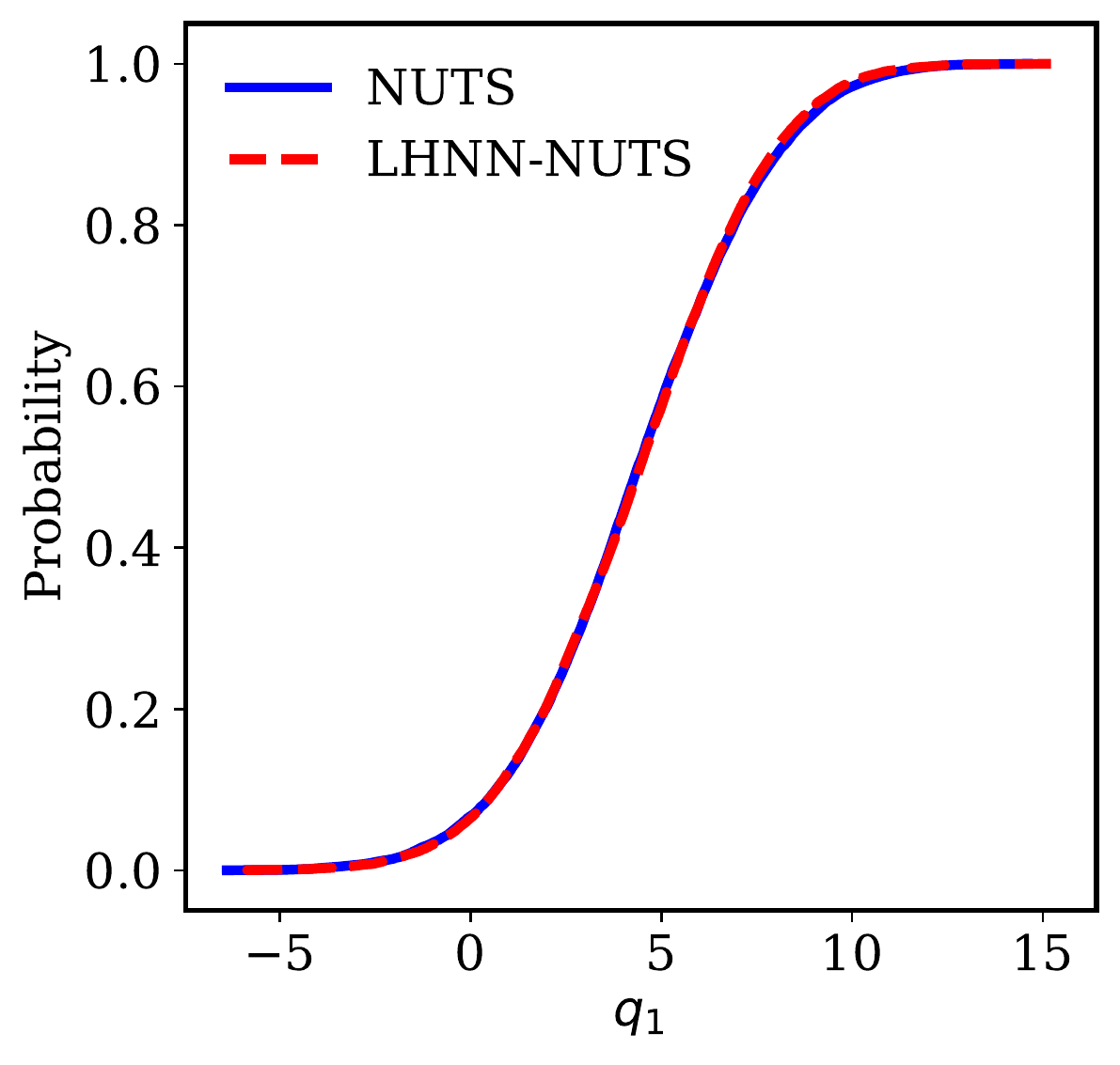} 
\caption{}
\label{2D_Neal_2}
\end{subfigure}
\begin{subfigure}{0.32\textwidth}
\centering  
\includegraphics[width=\textwidth]{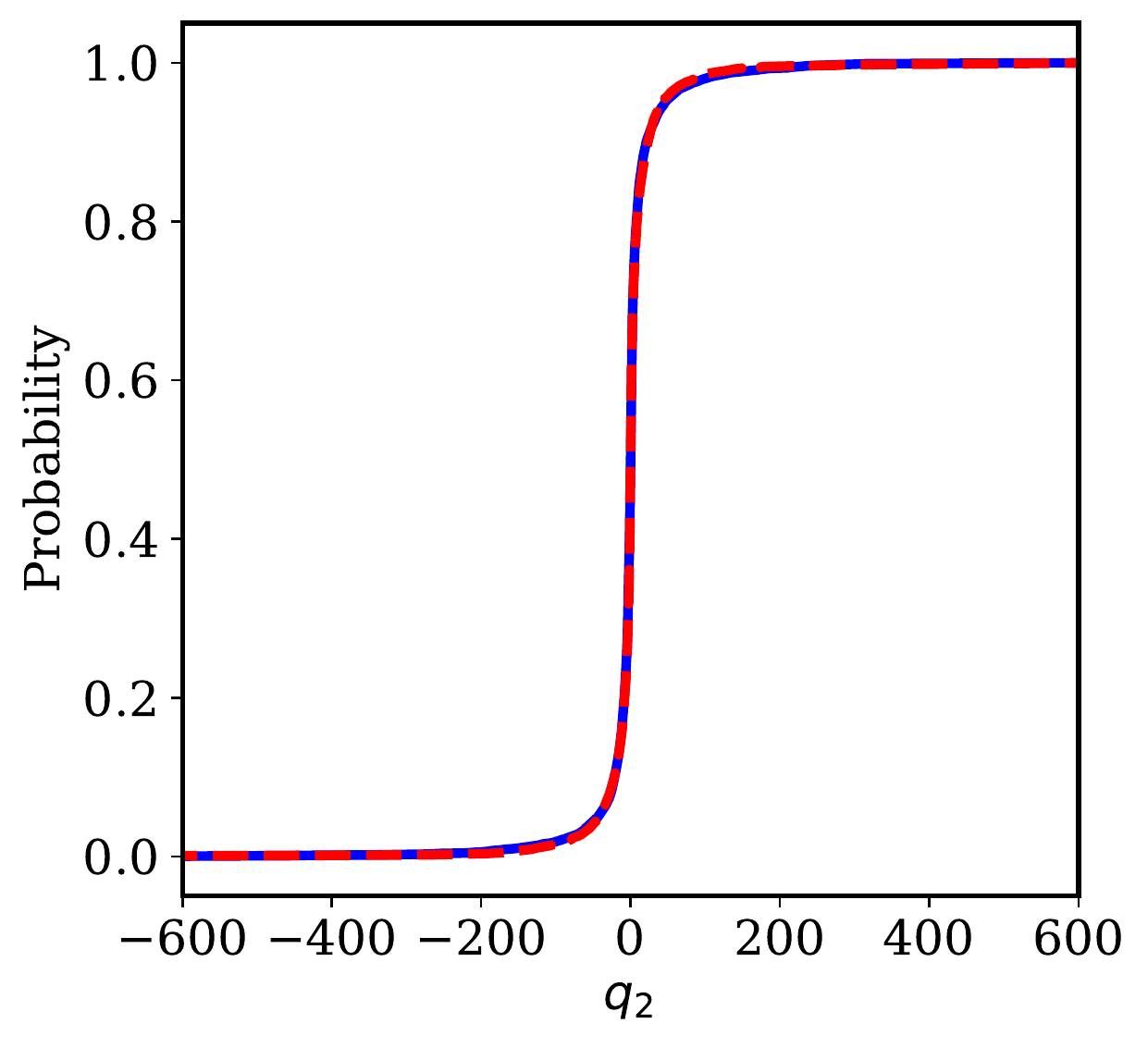} 
\caption{}
\label{2D_Neal_3}
\end{subfigure}
\caption{(a) Scatter plot comparison of samples generated from a 2-D Neal's funnel density, using NUTS and L-HNNs in NUTS with online error monitoring. Comparison of the eCDF plots for the dimensions (b) $q_1$ and (c) $q_2$.}
\label{2D_Neal}
\end{figure}

\subsection{Summary of the performance metrics}\label{sec:results_summary}

Overall, as shown in Table \ref{Table_Summary}, LHNN-NUTS requires significantly fewer gradients of the target density and produces ESSs per gradient that are about an order of magnitude better than those generated using traditional NUTS. This, of course, varies over the examples considered and is strongly dependent on a few factors: 1. the number of additional gradients needed from the online monitoring scheme; 2. the general efficiency of the NUTS algorithm for the specific problem; and 3. the size of the training set (which was held constant across all problems). For most of the cases, the additional gradients of the target density required during the online error monitoring scheme are small compared to the gradient evaluations needed to generate the training data where the number of additional gradients during online error monitoring seems to depend on the complexity of the target density. In particular, the 3-D Rosenbrock density, which has heavy tails, required a large number of additional gradients and the 2-D Neal's funnel density, which has a high local curvature, required the highest number of additional gradients. Note that, in both of these cases, the traditional NUTS algorithm required a huge number of gradient evaluations ($\sim 16$ Million). On the other hand, the 24D logistic regression problem required no additional gradient evaluations during the online error monitoring, while traditional NUTS is relatively efficient compared to the other problems (requiring only 1.25 Million evaluations). We therefore generalize and loosely state that, for problems where the NUTS algorithm requires a large number of gradient evaluations, we expect that the LHNN-NUTS will require a large number of additional evaluations due to error monitoring with some exceptions (e.g. 5D ill-conditioned Gaussian). Nonetheless, in all cases the LHNN-NUTS with online error monitoring offers considerable computational gains while producing samples that satisfactorily match the reference target density. 


\section{Discussion of future work}

In HMC and NUTS, there are opportunities to explore more recent neural network architectures that learn Hamiltonian systems, with the potential for improved performance over HNNs and L-HNNs. This line of investigation is motivated by a 100-D Gaussian with Wishart inverse covariance from Hoffman and Gelman \cite{Hoffman2014}. For this problem, the ESS per gradient values were similar between LHNN-NUTS and traditional NUTS. The poorer performance of L-HNNs in this case is a result of many strong correlations in the target distribution and large integration errors while simulating the Hamiltonian trajectories. Despite the large integration errors for L-HNNs, the online error monitoring scheme ensures that the correct distribution is inferred and the correlations between the variables are adequately captured (as seen from Figure \ref{100D_Gauss}) at the expense of calling the leapfrog scheme with gradients of the target density many times during the sampling. However, similar to the 100-D rough well example from Section \ref{sec:100D_RW}, while L-HNNs in NUTS with online error monitoring takes in the order of hours to simulate $10,000$ samples, traditional NUTS takes in the order of days due to the complexity of this problem (provided that the Hamiltonian gradients are computed numerically and not using a closed-form solution).
\begin{figure}[htbp]
\centering
\includegraphics[scale=0.15]{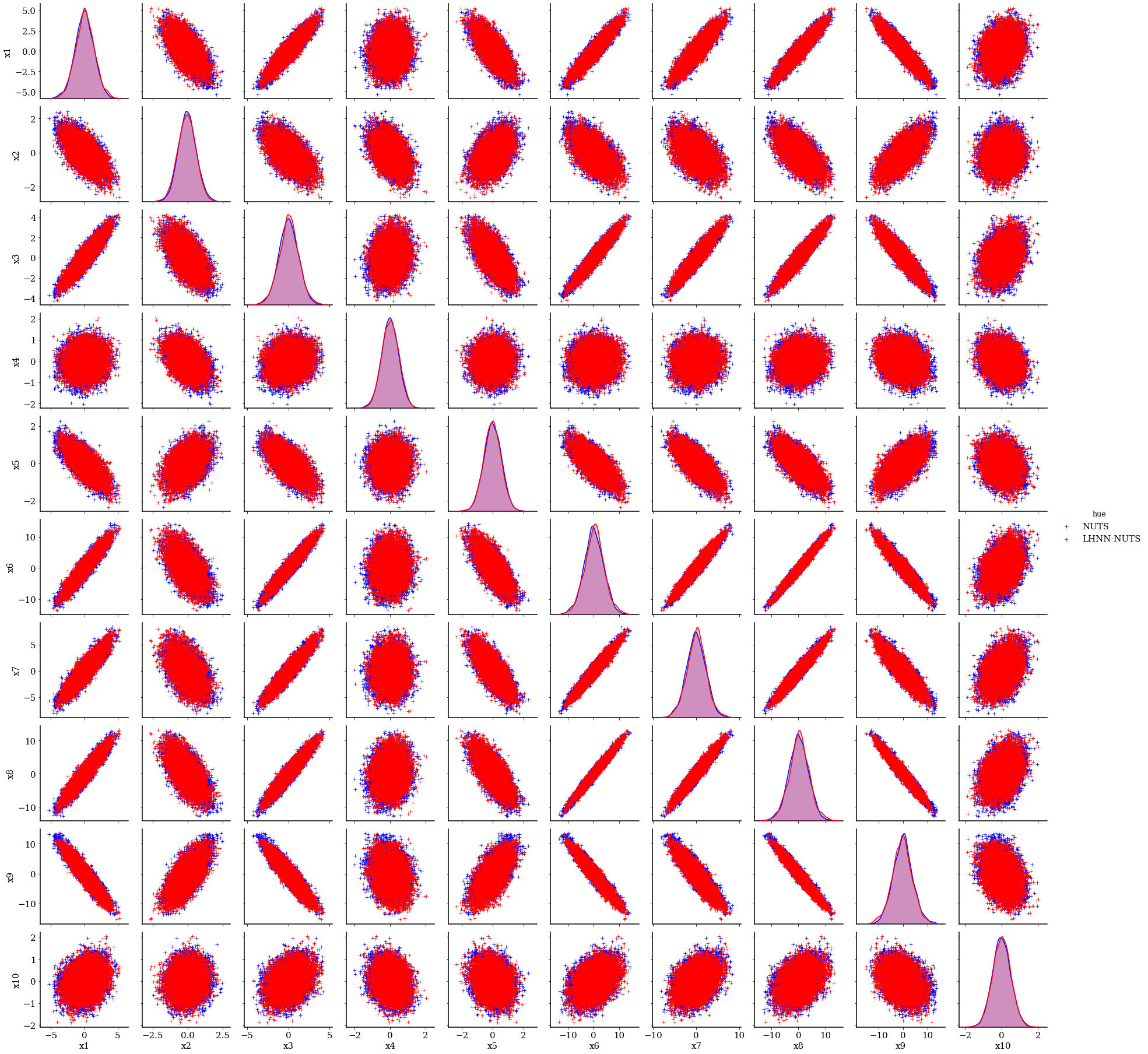} 
\caption{Scatter plot comparison between NUTS and LHNN-NUTS with online error monitoring, considering a 100-D Gaussian distribution with Wishart inverse covariance. As an illustration, 10 of the 100 dimensions are shown. Whereas L-HNNs in NUTS with online error monitoring takes in the order of hours to simulate $10,000$ samples, traditional NUTS takes in the order of days; provided that the Hamiltonian gradients are computed numerically and not using a closed-form solution. (Blue dots: traditional NUTS; Red dots: LHNN-NUTS with online error monitoring)}
\label{100D_Gauss}
\end{figure}

Mattheakis et al. \cite{Mattheakis2022a} propose neural networks that use time as inputs and predict the $\{\pmb{q},\pmb{p}\}$ coordinates. Their loss function minimizes the mean squared error between the time derivatives of the actual $\{\pmb{q},\pmb{p}\}$ coordinates and those predicted by the neural networks using the right-hand side of Equation \eqref{eqn:HD_2}. This form for the loss function seems to result in more rapid and substantial loss reduction. Jin et al. \cite{Jin2020a} propose symplectic networks, which lead to better conservation of the Hamiltonian than HNNs. Future work will investigate whether architectures such as these would lead to an improvement in ESS per gradient, while satisfactorily sampling from the target density.

While this paper uses an Eucledian formulation for generating momenta that are independent of the current position [Equation \eqref{eqn:HMC_4}], training L-HNNs under a Riemannian formulation for the momenta \cite{Girolami2011a} is also possible. For some densities with high local curvatures (e.g., the Neal's funnel or the multivariate Gaussian with Wishart inverse covariance), L-HNNs trained in a Riemannian space may generate more computational gains in terms of ESS per gradient. Finally, although the error monitoring scheme introduced in this paper calls the standard leapfrog with gradients of the target density when the L-HNN integration errors are large, it does not retrain the L-HNNs. Such retraining, when performed on-the-fly within the error-monitoring scheme, is similar in concept to active or online learning. Owing to the computational expense involved in training neural networks for high-dimensional problems, quantifying the optimal schedule for retraining that is agnostic to the complexity and properties of the target density will be an interesting avenue of exploration. 

\section{Summary and conclusions}

Neural networks that learn Hamiltonian systems (e.g., HNNs) have received a great deal of interest. We proposed using HNNs to solve Bayesian inference problems through the HMC framework. HNNs have important properties (e.g., time reversibility and Hamiltonian conservation) that make them amenable for usage in HMC sampling. Moreover, once trained, HNNs do not require numerical gradients of the target density\textemdash which can be computationally expensive\textemdash during sampling. To increase the expressivity of HNNs, we proposed HNNs with latent outputs (L-HNNs) and demonstrated that L-HNNs enable integration error reductions in predicting the Hamiltonian dynamics when compared to HNNs for a 3-D Rosenbrock density function. Although HMC is a state-of-the-art sampler for Bayesian inference, it requires the specification of an end time of the trajectory for each sample, and this can considerably influence the inference quality. NUTS automatically decides this end time, while also ensuring stationarity and that the next proposed sample is farther away from the previous one (i.e., lesser correlations between samples). We further proposed to use L-HNNs in NUTS with an online error monitoring scheme that reverts to the standard leapfrog integration for a few samples whenever the L-HNNs integration errors are high. L-HNNs integration errors can be high in regions of the uncertainty space that feature fewer training data (e.g., the tails). This proposed online error monitoring scheme ensures that L-HNNs rapidly move to regions of high probability density after visiting a region of low probability density, and it prevents degeneracy of the sampling scheme.

L-HNNs in NUTS with online error monitoring was demonstrated for several example cases involving complex, heavy-tailed, and high-local-curvature probability densities with dimensions ranging from 5 to 100. The performance of L-HNNs in NUTS was compared to that of traditional NUTS, which relies on numerical gradients of the target density. L-HNNs in NUTS satisfactorily inferred the distributions of these example cases, as evidenced by scatter plots and eCDFs, when compared to traditional NUTS. More interestingly, L-HNNs in NUTS required 1--2 orders of magnitude fewer evaluations of the target density gradients, and produced an ESS per gradient that was an order of magnitude better than that generated via traditional NUTS.

Several future opportunities exist for improving the performance of neural networks that identify Hamiltonian systems for Bayesian inference problems: (1) newer architectures (e.g., symplectic networks) better conserve the Hamiltonian than L-HNNs, and result in an efficient inference, as quantified by the ESS per gradient; (2) training and evaluating L-HNNs in a Riemannian space rather than an Eucledian space may lead to significant computational efficiencies for densities with high local curvatures; and (3) active learning for training L-HNNs; in particular, optimal schedules for retraining L-HNNs in light of the high neural network training costs for high-dimensional problems could improve efficiency.

\section*{Software}
The codes used for simulating the results in this paper are available at \href{https://github.com/IdahoLabResearch/BIhNNs}{https://github.com/IdahoLabResearch/BIhNNs}.

\section*{Acknowledgments}
We thank Denny Thaler from RWTH Aachen University for his comments on our code. This research is supported through INL's Laboratory Directed Research and Development (LDRD) Program under U.S. Department of Energy (DOE) Idaho Operations Office contract no. DE-AC07-05ID14517. This research made use of the resources of the High Performance Computing Center at INL, which is supported by the DOE Office of Nuclear Energy and the Nuclear Science User Facilities under contract no. DE-AC07-05ID14517.

\bibliographystyle{siamplain}
\bibliography{references}


\section*{\Large SUPPLEMENTARY MATERIAL}

\normalsize

\section*{Synchronized leap frog integration}

Algorithm \ref{alg:Leap_Frog} presents the synchronized leap frog integration scheme.

\begin{algorithm}
  \caption{Synchronized Leap Frog integration for simulating Hamiltonian dynamics}
  \label{alg:Leap_Frog}
    \begin{algorithmic}[1]
    \STATE{Hamiltonian: $H$, Initial conditions: $\pmb{z}(0)=\{\pmb{q}(0),~\pmb{p}(0)\}$, Dimensions: $d$, Steps: $N$, End time: $T$}
    \STATE{Numerical gradients of $H$ with respect to $\pmb{q}(t)$ using packages like \texttt{autograd} \cite{autograd2015}}
    \STATE{$\Delta t = \frac{T}{N}$}
    \FOR{$j = 0:N-1$}
        \STATE{$t = j~\Delta t$}
        \STATE{Compute $\frac{\partial H}{\partial \pmb{q}(t)}$}
        \FOR{$i = 1:d$}
            \STATE{$q_i(t+\Delta t) = q_i(t) + \frac{\Delta t}{m_i}~p_i(t) - \frac{\Delta t^2}{2m_i}~\frac{\partial H}{\partial q_i(t)}$}
        \ENDFOR
        \STATE{Compute $\frac{\partial H}{\partial \pmb{q}(t+\Delta t)}$}
        \FOR{$i = 1:d$}
            \STATE{$p_i(t+\Delta t) = p_i(t) - \frac{\Delta t}{2}~\bigg(\frac{\partial H}{\partial q_i(t)} + \frac{\partial H}{\partial q_i(t+\Delta t)}\bigg)$}
        \ENDFOR
    \ENDFOR
     \end{algorithmic}
\end{algorithm}
\normalsize

\section*{Hamiltonian Monte Carlo}

Algorithm \ref{alg:HMC} presents the HMC algorithm.

\begin{algorithm}
\caption{Hamiltonian Monte Carlo}
\label{alg:HMC}
\begin{algorithmic}[1]
\STATE{Hamiltonian: $H = U(\pmb{q}) + K(\pmb{p})$, Samples: M, Starting sample: $\{\pmb{q}^0,~\pmb{p}^0\}$, End time for trajectory: $T$}
\FOR{$i = 1:M$}
\STATE{$\pmb{p}(0) \sim \mathcal{MVN}(\pmb{0},~\pmb{M})$}
\STATE{$\pmb{q}(0) = \pmb{q}^{i-1}$}
\STATE{Compute $\{\pmb{q}^*,~\pmb{p}^*\} = \{\pmb{q}(T),~\pmb{p}(T)\}$ with Algorithm \ref{alg:Leap_Frog}}
\STATE{$\alpha = \exp \big(H(\{\pmb{q}^{i-1},\pmb{p}^{i-1}\})-H(\{\pmb{q}^*,\pmb{p}^*\})\big)$}
\STATE{With probability $\textrm{min}\{1,~\alpha\}$, set $\{\pmb{q}^{i},\pmb{p}^{i}\} \gets \{\pmb{q}^*,\pmb{p}^*\}$}
\ENDFOR
\end{algorithmic}
\end{algorithm}
\normalsize

\section*{Effective sample size}

The effective sample size (ESS) is a metric for measuring the quality of MCMC samples \cite{Vehtari2021a}. Since MCMC samples are correlated, the effective number of samples will typically be less than the total samples drawn. ESS is defined in terms of the autocorrelations within the chain as \cite{Vehtari2021a}:

\begin{equation}
    \label{eqn:ESS_1}
    M_{eff} = \frac{M}{1+2 \sum_{t=1}^\infty \rho^t}
\end{equation}

\noindent where, $M$ is the total samples drawn and $\rho^t$ is the autocorrelation at lag $t \geq 0$. $\rho^t$ for dimension $i$ is defined as:

\begin{equation}
    \label{eqn:ESS_2}
    \rho^t = \frac{1}{\sigma^2} \int_{\mathcal{Q}_i} q_i^{c}~q_i^{c+t} ~f(q_i)~dq_i
\end{equation}

\noindent where, $\sigma$ is the standard deviation, $q_i^{c}$ and $q_i^{c+t}$ are the samples with lag $t$, and $f(.)$ is the required distribution. In practice, quantities in the above equation are estimated from the MCMC samples itself and Equation \eqref{eqn:ESS_1} is numerically computed using fast Fourier transforms in Tensorflow \cite{tensorflow2015-whitepaper}. In this study, we use ESS as a metric to evaluate the performance of HNNs for Bayesian inference relative to HMC-based sampling.

\section*{Details of L-HNNs}

Table \ref{Table_LHNNs} presents the architecture and training data details of L-HNNs.

\begin{table}[h]
\centering
\caption{Details of L-HNNs architecture and training data.}
\label{Table_LHNNs}
\small
\begin{tabular}{ |c|c|c|c| }
\hline
 \# Neurons & \# HMC samples  & Error threshold ($\Delta_{max}^{lhnn}$) & Leap frog samples ($N_{lf}$)\\
\hline
\multicolumn{4}{|c|}{\textbf{3D degenerate Rosenbrock}}\\
\hline
100 & 40 & 10 & 20\\
\hline
\multicolumn{4}{|c|}{\textbf{5-D ill-conditioned Gaussian}}\\
\hline
100 & 40 & 10 & 20\\
\hline
\multicolumn{4}{|c|}{\textbf{10-D degenerate Rosenbrock}}\\
\hline
100 & 40 & 10 & 20\\
\hline
\multicolumn{4}{|c|}{\textbf{24-D Bayesian logistic regression}}\\
\hline
100 & 40 & 10 & 20\\
\multicolumn{4}{|c|}{\textbf{100D rough well}}\\
\hline
500 & 40 & 20 & 5\\
\hline
\multicolumn{4}{|c|}{\textbf{2-D Neal's funnel}}\\
\hline
100 & 40 & 10 & 20\\
\hline
\end{tabular}
\begin{tablenotes}
\small
\item[-] - All example cases use 3 hidden layers.
\item[-] - All example cases use a sine nonlinear activation function.
\item[-] - All example cases use 250 units as the end time for HMC samples with $\Delta t = 0.025$.
\end{tablenotes}
\end{table}
\normalsize

\section*{Additional results}

\begin{figure}[htbp]
\begin{subfigure}{0.45\textwidth}
\centering  
\includegraphics[scale=0.45]{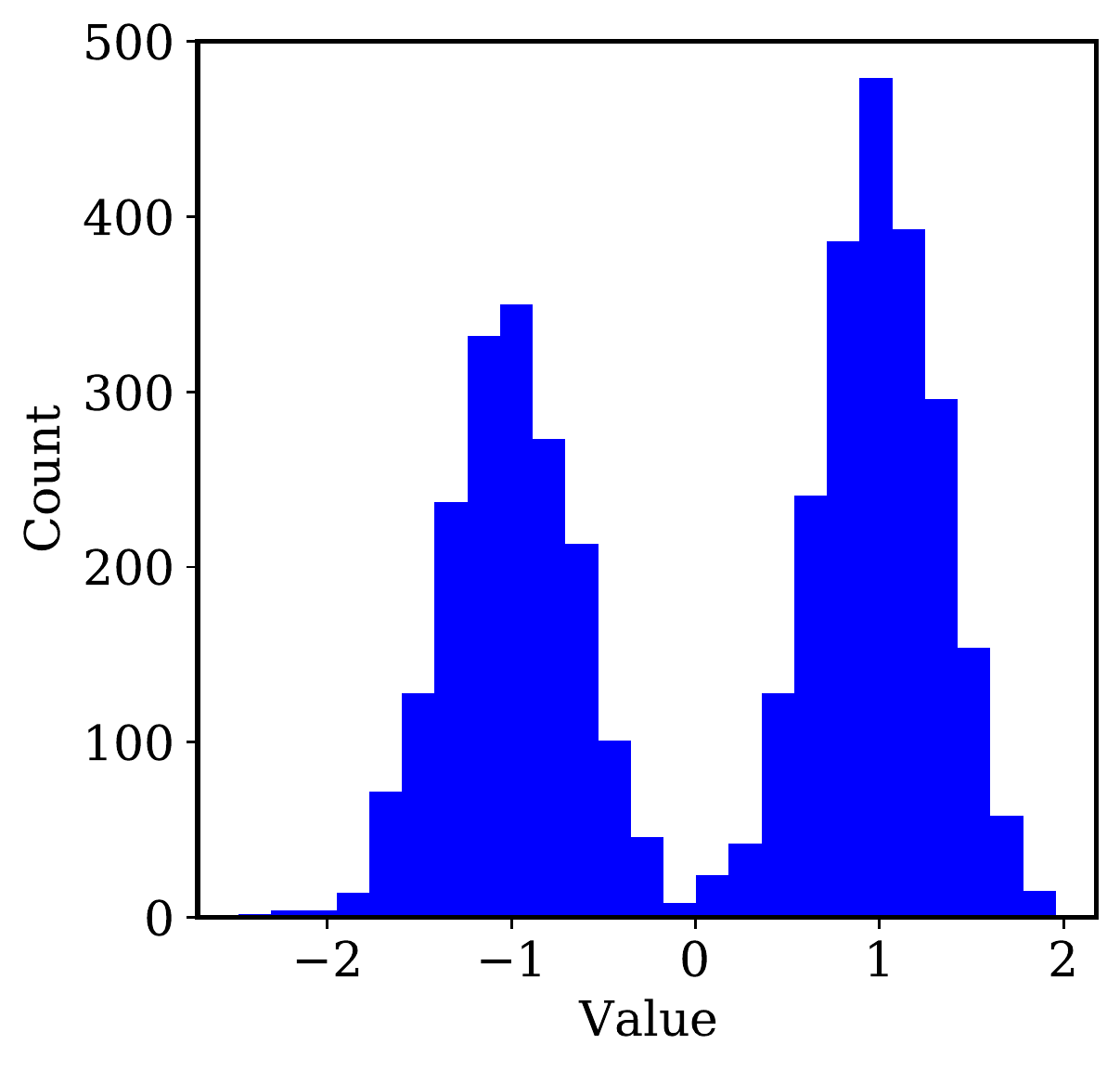} 
\caption{}
\label{Sampling_Demo_NUTS_1}
\end{subfigure}
\begin{subfigure}{0.45\textwidth}
\centering  
\includegraphics[scale=0.45]{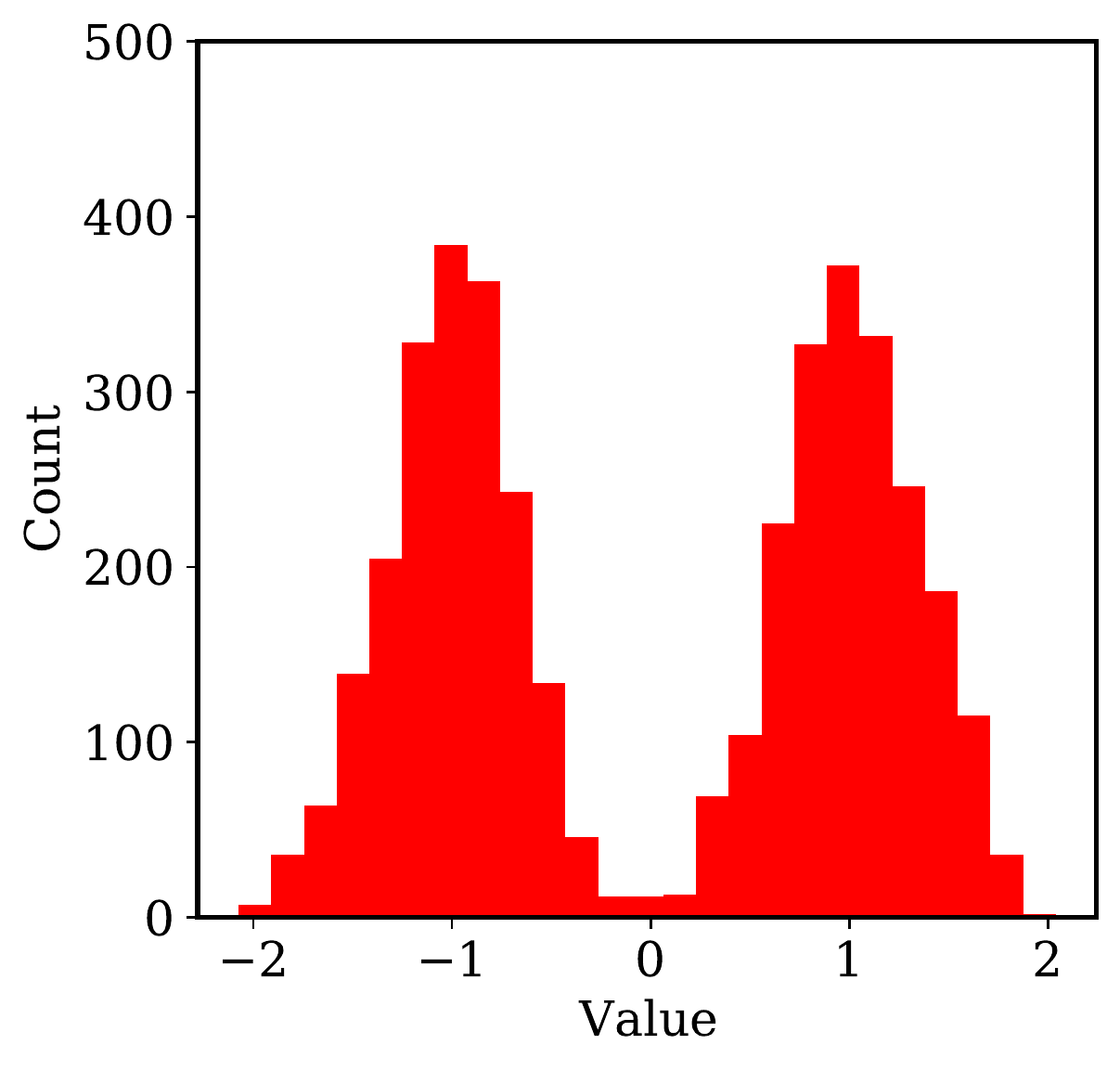} 
\caption{}
\label{Sampling_Demo_NUTS_2}
\end{subfigure}
\caption{Histogram of samples from a 1D Gaussian mixture density using: (a) traditional NUTS; and (b) HNN-NUTS.}
\label{Sampling_Demo_NUTS_Hist}
\end{figure}

\begin{figure}[htbp]
\centering
\includegraphics[width=\textwidth]{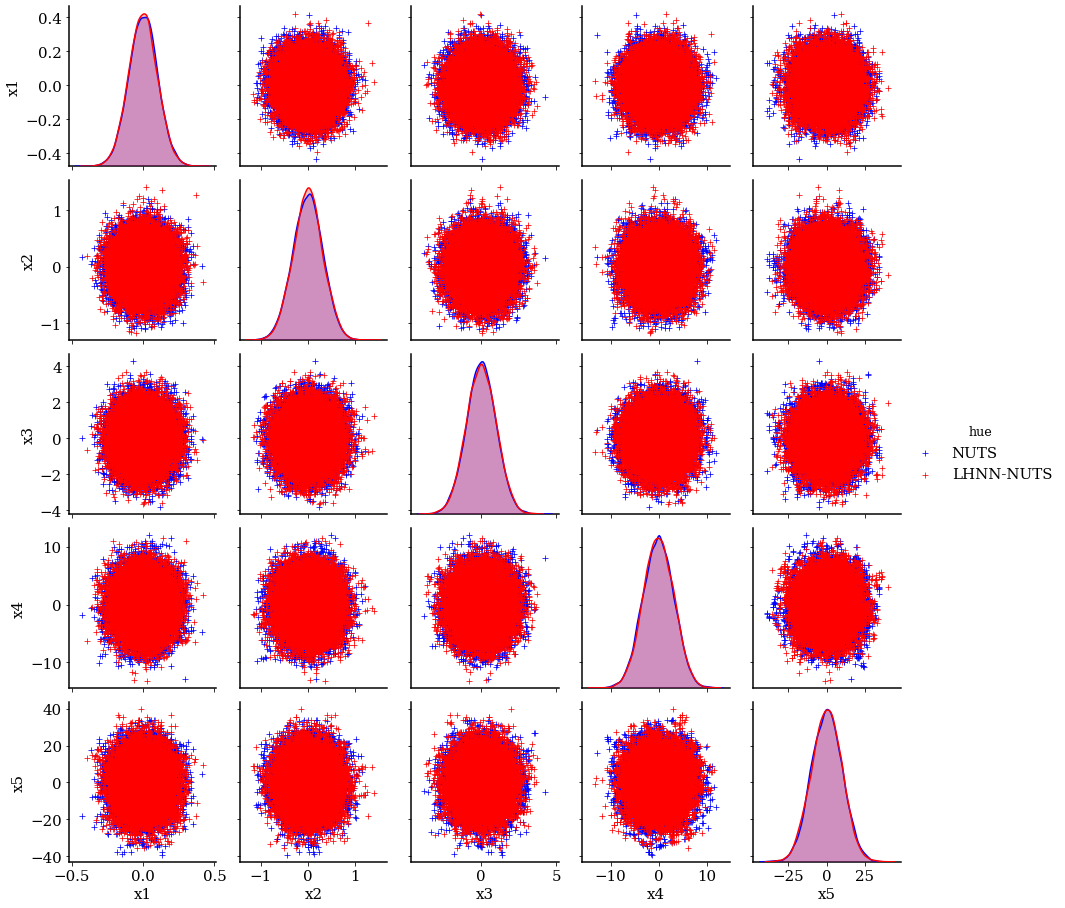} 
\caption{Scatter plot comparison between NUTS and LHNN-NUTS with online error monitoring, considering a 5-D ill-conditioned Gaussian distribution. (Blue dots: traditional NUTS; Red dots: LHNN-NUTS with online error monitoring)}
\label{5D_IllGauss_highres}
\end{figure}

\begin{figure}[htbp]
\begin{subfigure}{0.32\textwidth}
\centering  
\includegraphics[scale=0.32]{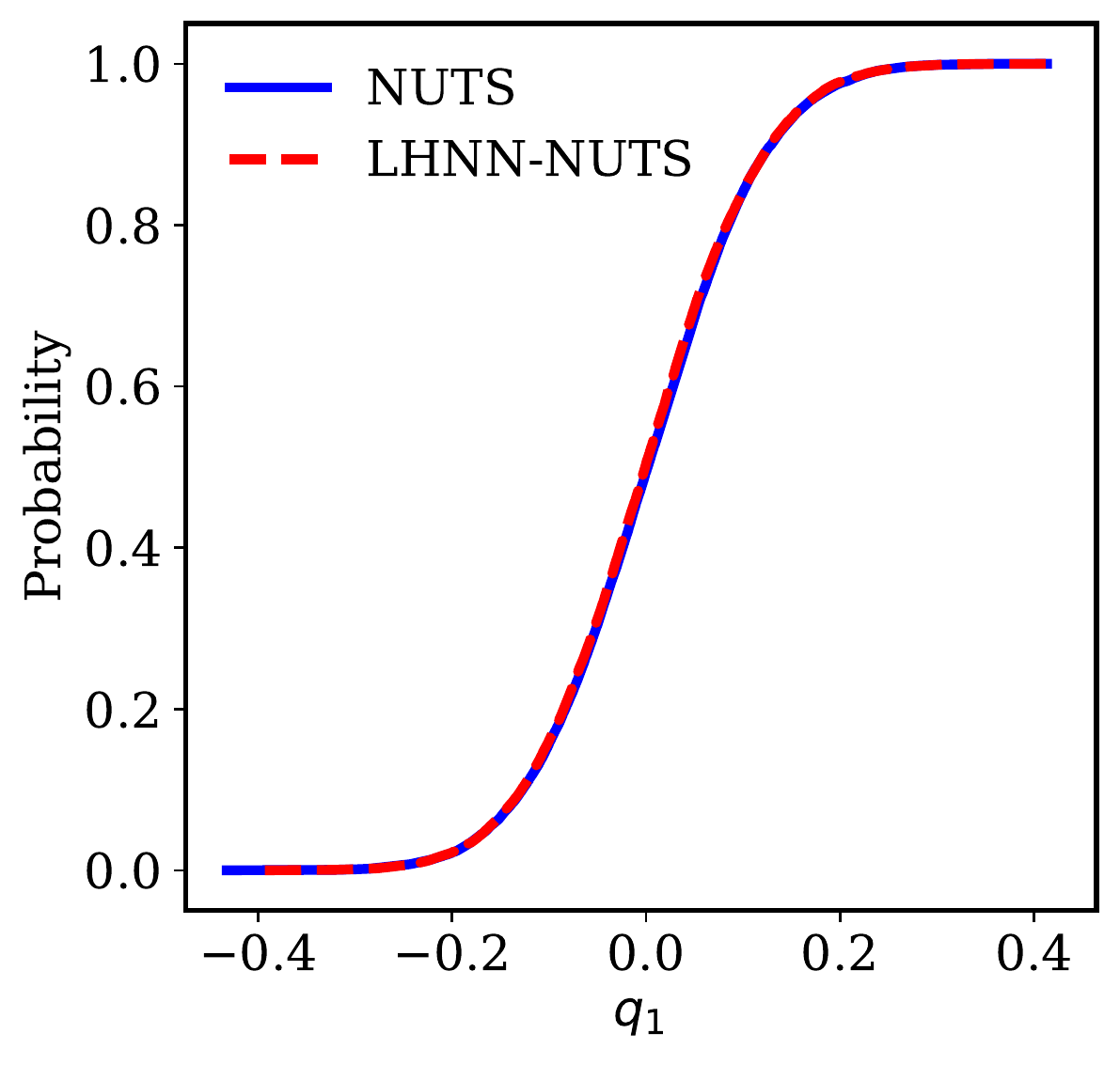} 
\caption{}
\label{5D_IllGauss_ecdf_1}
\end{subfigure}
\begin{subfigure}{0.32\textwidth}
\centering  
\includegraphics[scale=0.32]{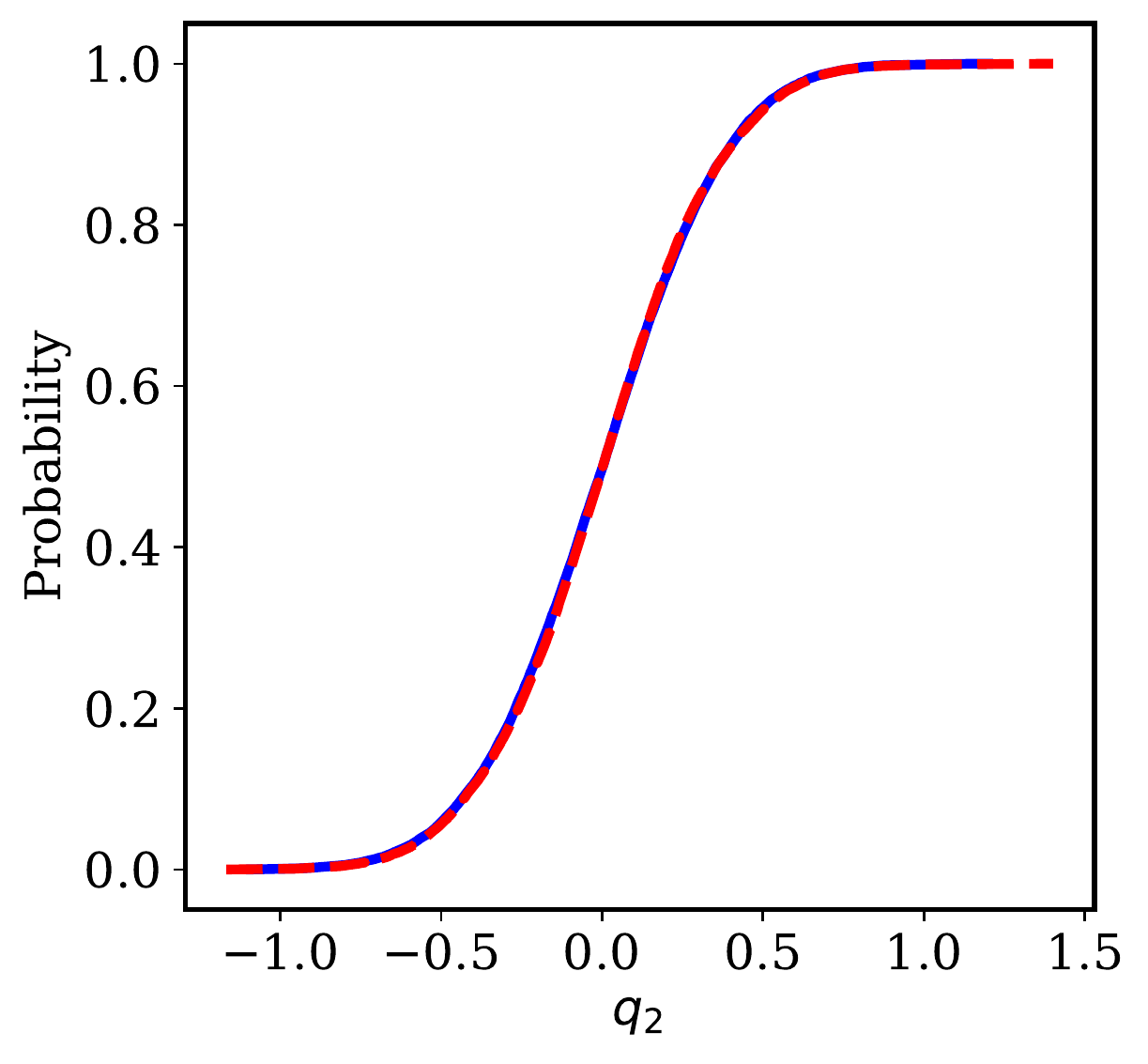} 
\caption{}
\label{5D_IllGauss_ecdf_2}
\end{subfigure}
\begin{subfigure}{0.32\textwidth}
\centering  
\includegraphics[scale=0.32]{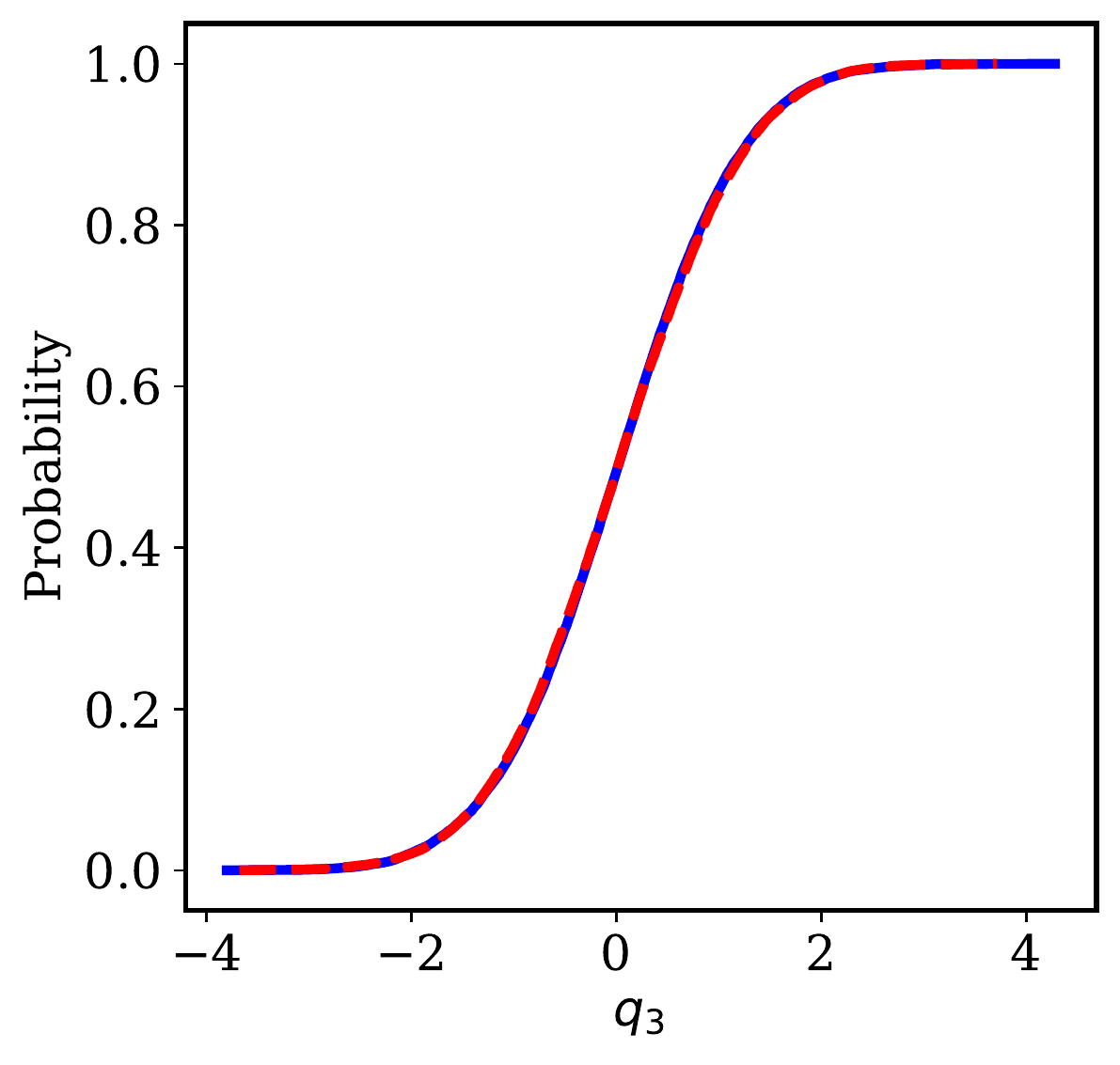} 
\caption{}
\label{5D_IllGauss_ecdf_3}
\end{subfigure}
\begin{subfigure}{0.32\textwidth}
\centering  
\includegraphics[scale=0.32]{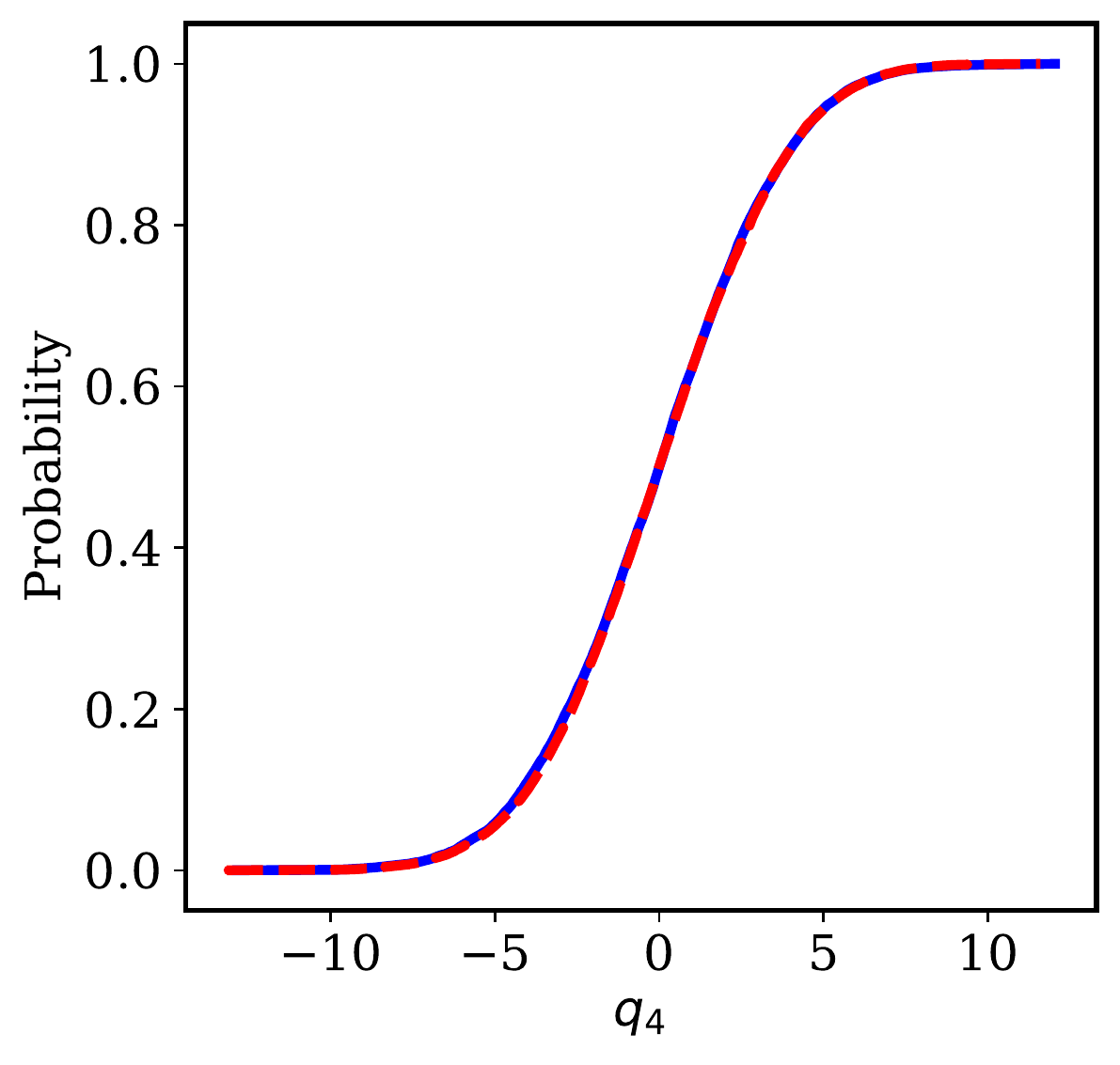} 
\caption{}
\label{5D_IllGauss_ecdf_4}
\end{subfigure}
\begin{subfigure}{0.32\textwidth}
\centering  
\includegraphics[scale=0.32]{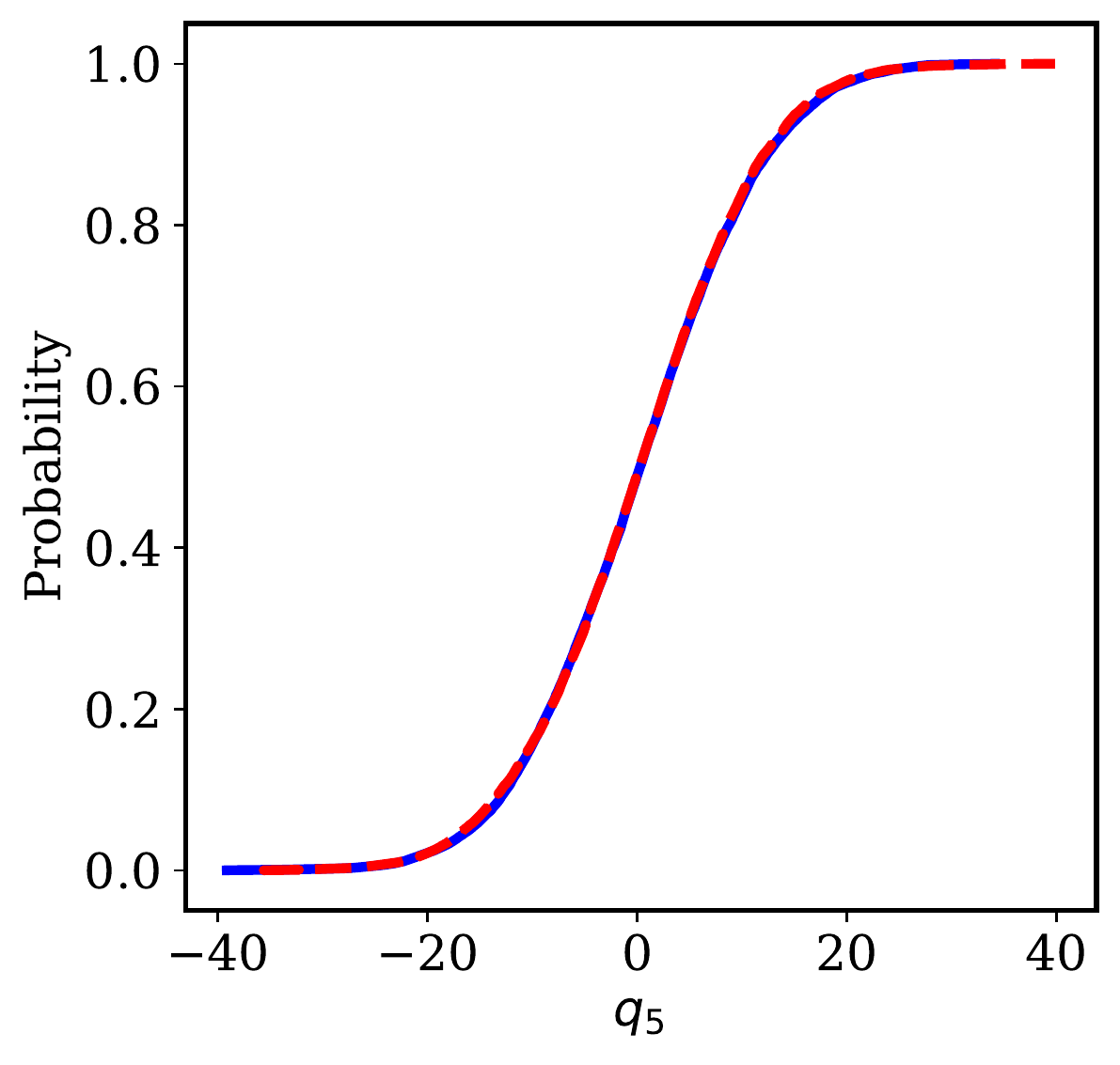} 
\caption{}
\label{5D_IllGauss_ecdf_5}
\end{subfigure}
\caption{5D ill-conditioned Gaussian density empirical cumulative distribution functions comparison.}
\label{5D_IllGauss_ecdf}
\end{figure}

\begin{figure}[htbp]
\centering
\includegraphics[width=\textwidth]{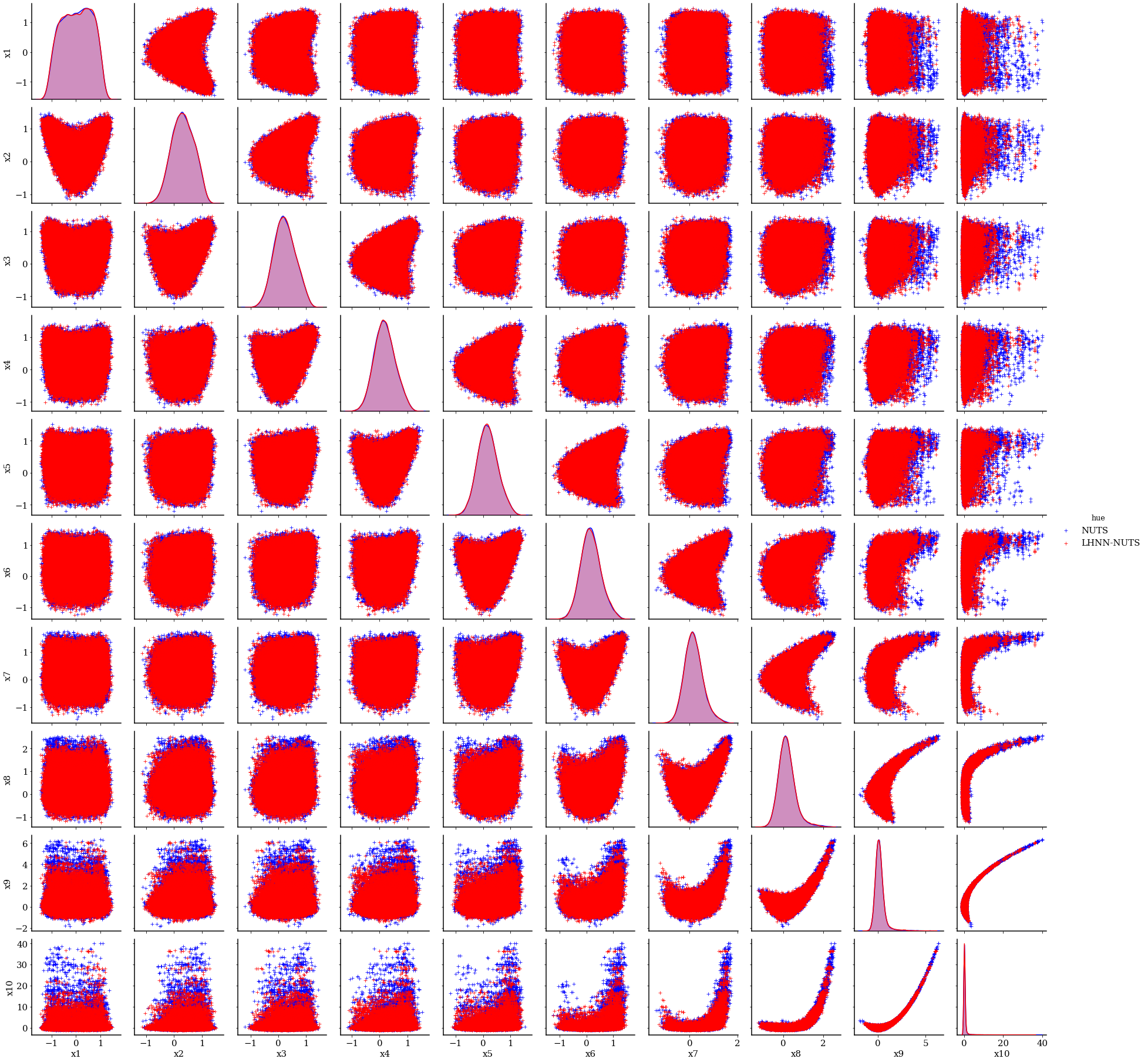} 
\caption{Scatter plot comparison between NUTS and LHNN-NUTS with online error monitoring, considering a 10-D degenerate Rosenbrock distribution. (Blue dots: traditional NUTS; Red dots: LHNN-NUTS with online error monitoring)}
\label{10D_Rosen_highres}
\end{figure}

\begin{figure}[htbp]
\begin{subfigure}{0.32\textwidth}
\centering  
\includegraphics[scale=0.32]{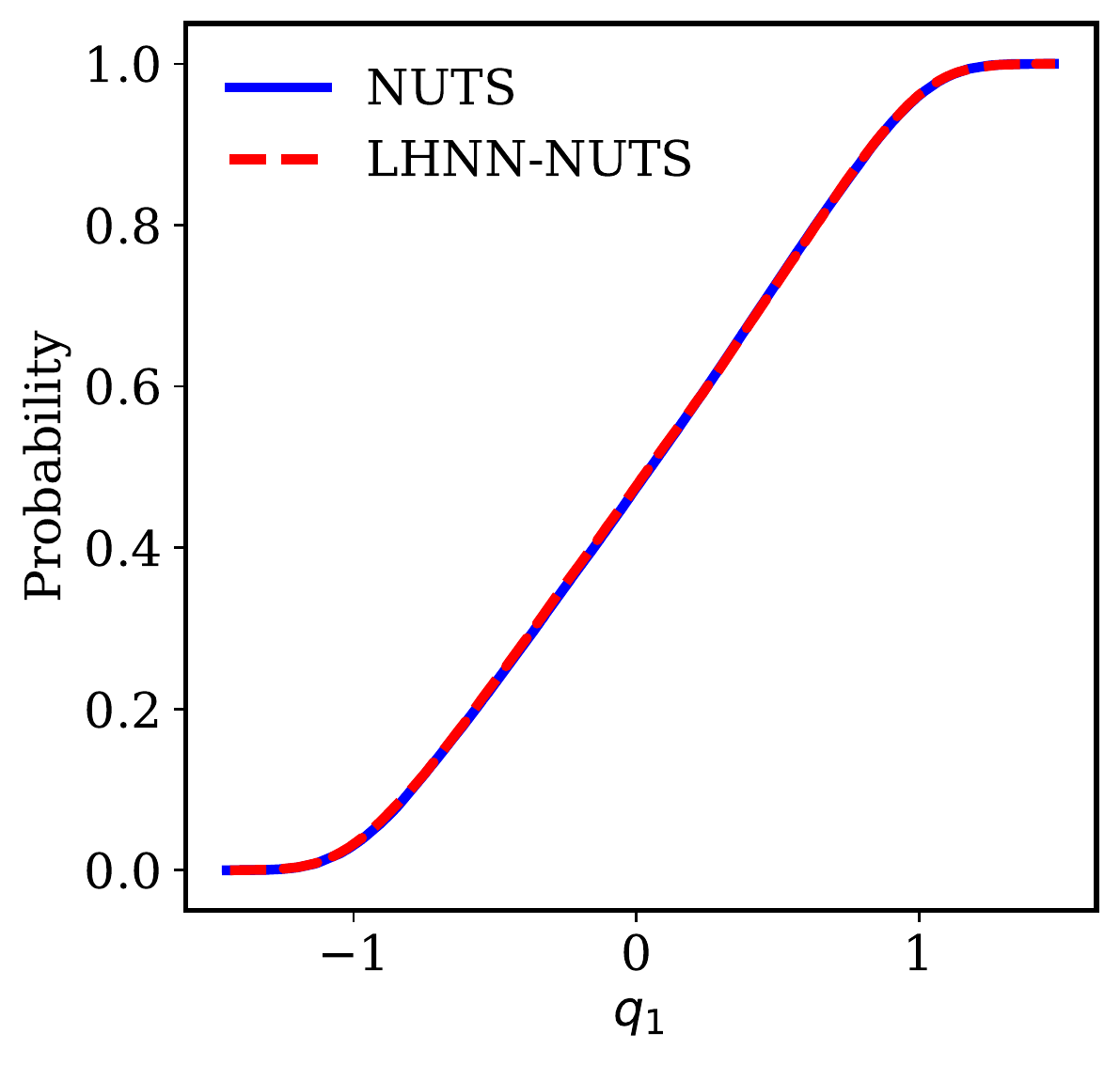} 
\caption{}
\label{10D_Rosen_ecdf_1}
\end{subfigure}
\begin{subfigure}{0.32\textwidth}
\centering  
\includegraphics[scale=0.32]{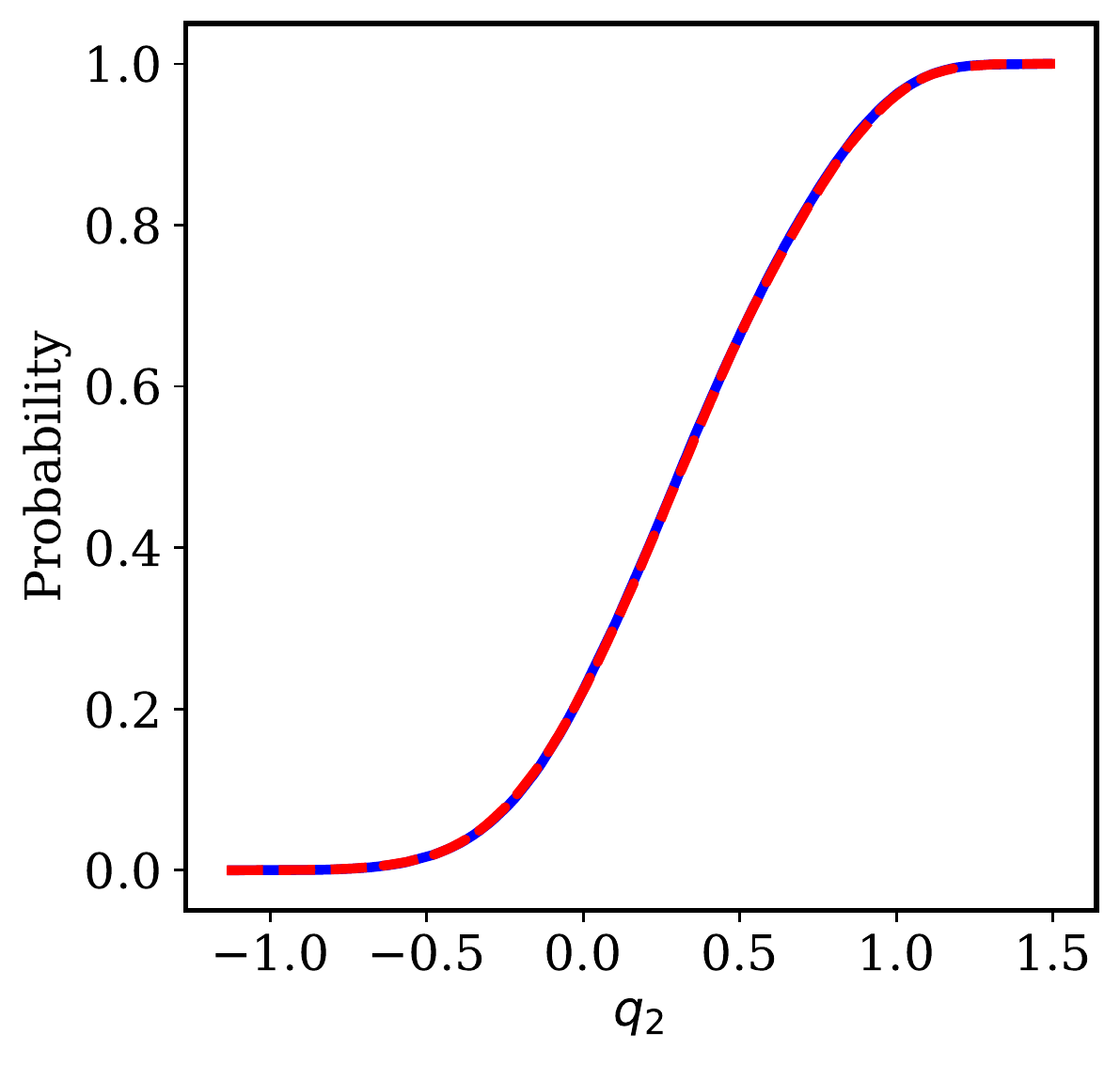} 
\caption{}
\label{10D_Rosen_ecdf_2}
\end{subfigure}
\begin{subfigure}{0.32\textwidth}
\centering  
\includegraphics[scale=0.32]{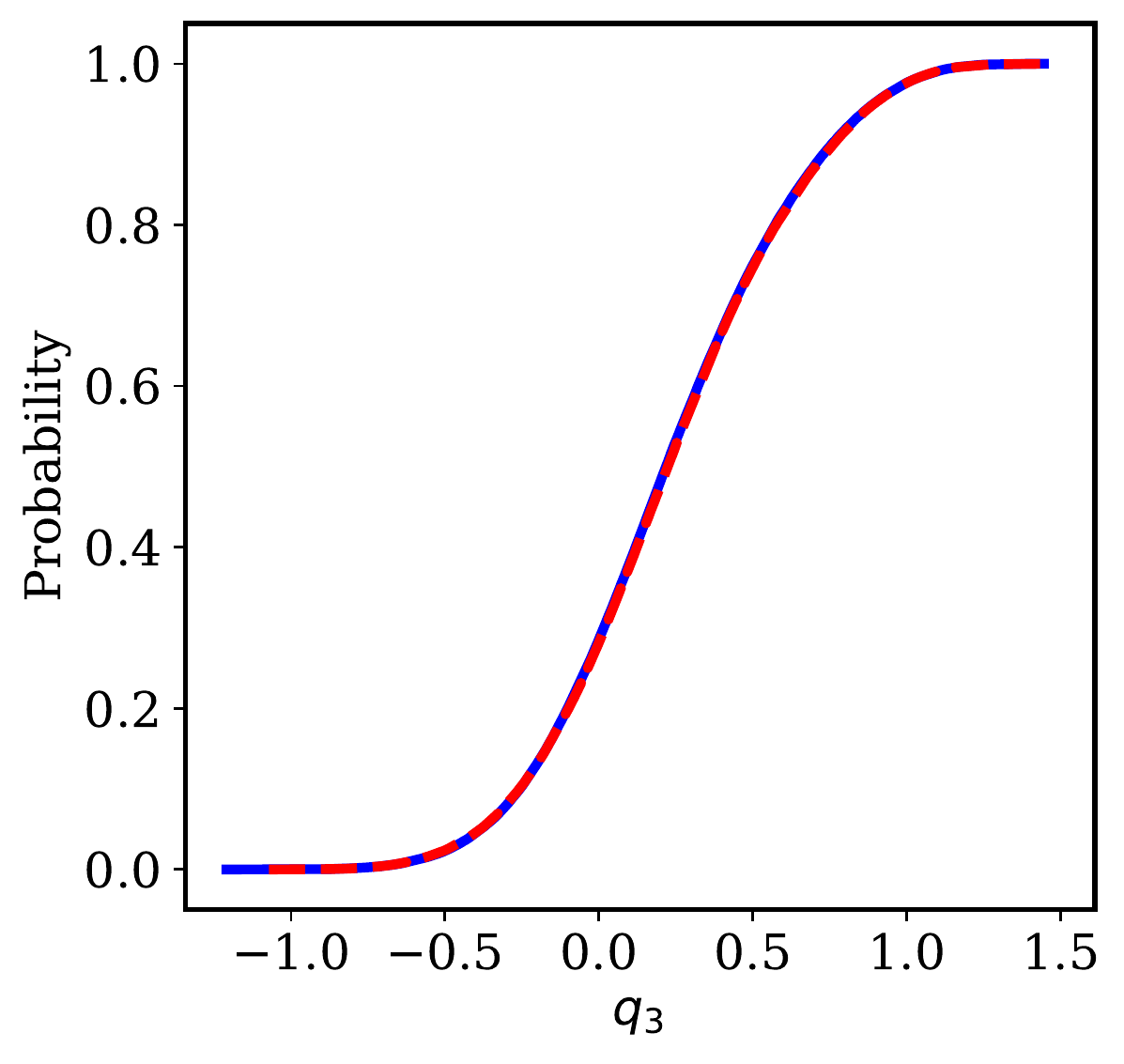} 
\caption{}
\label{10D_Rosen_ecdf_3}
\end{subfigure}
\begin{subfigure}{0.32\textwidth}
\centering  
\includegraphics[scale=0.32]{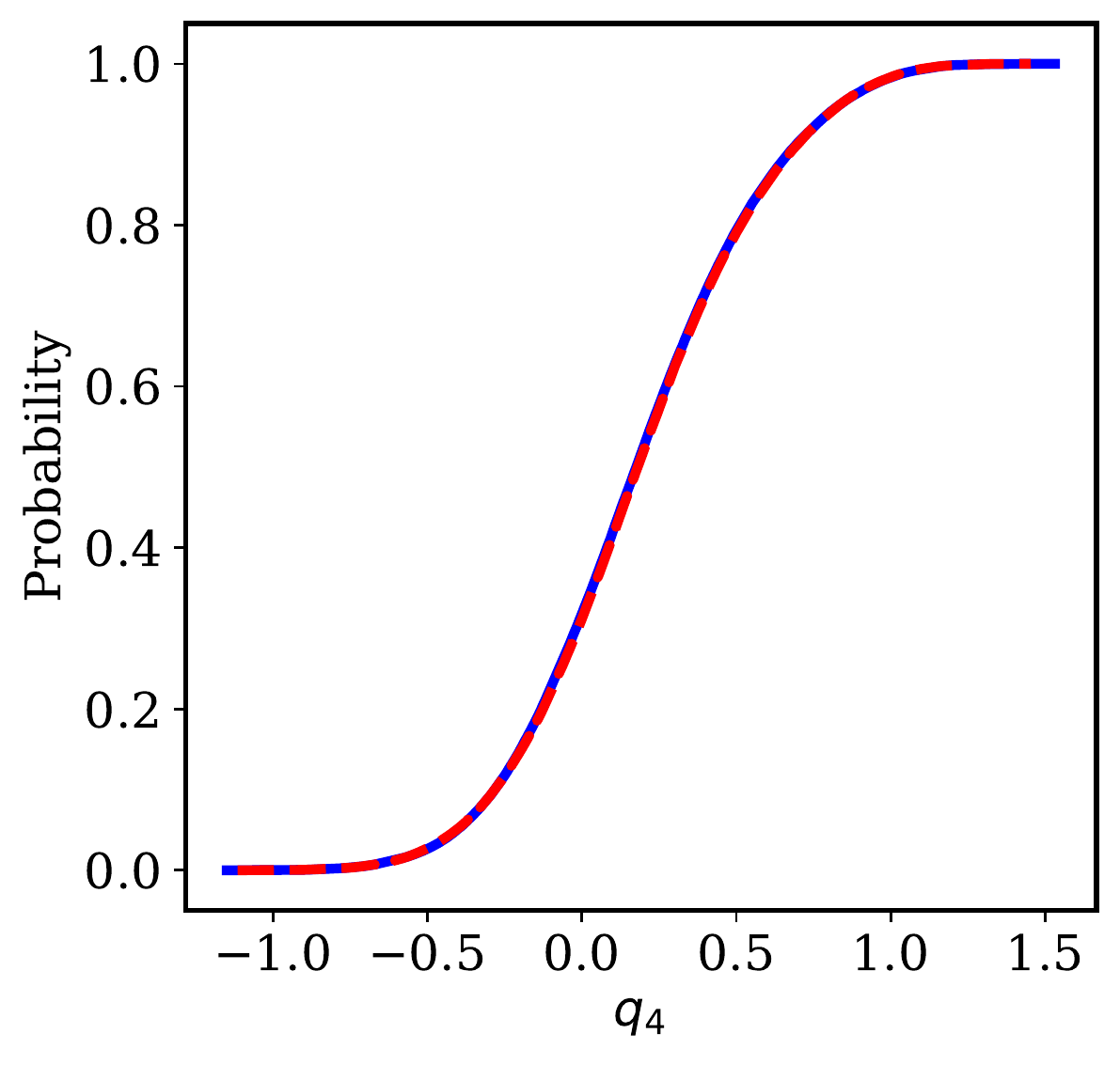} 
\caption{}
\label{10D_Rosen_ecdf_4}
\end{subfigure}
\begin{subfigure}{0.32\textwidth}
\centering  
\includegraphics[scale=0.32]{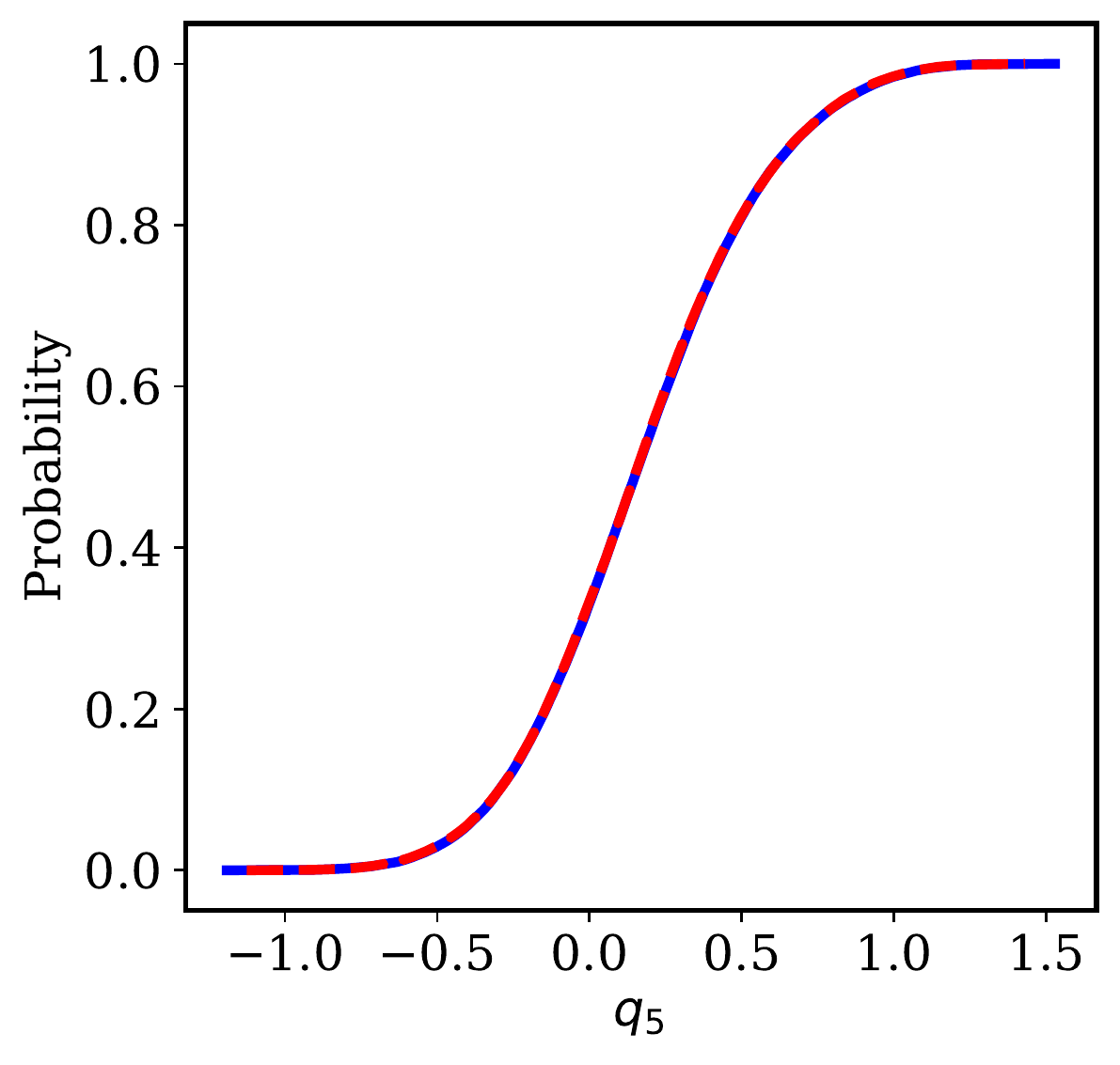} 
\caption{}
\label{10D_Rosen_ecdf_5}
\end{subfigure}
\begin{subfigure}{0.32\textwidth}
\centering  
\includegraphics[scale=0.32]{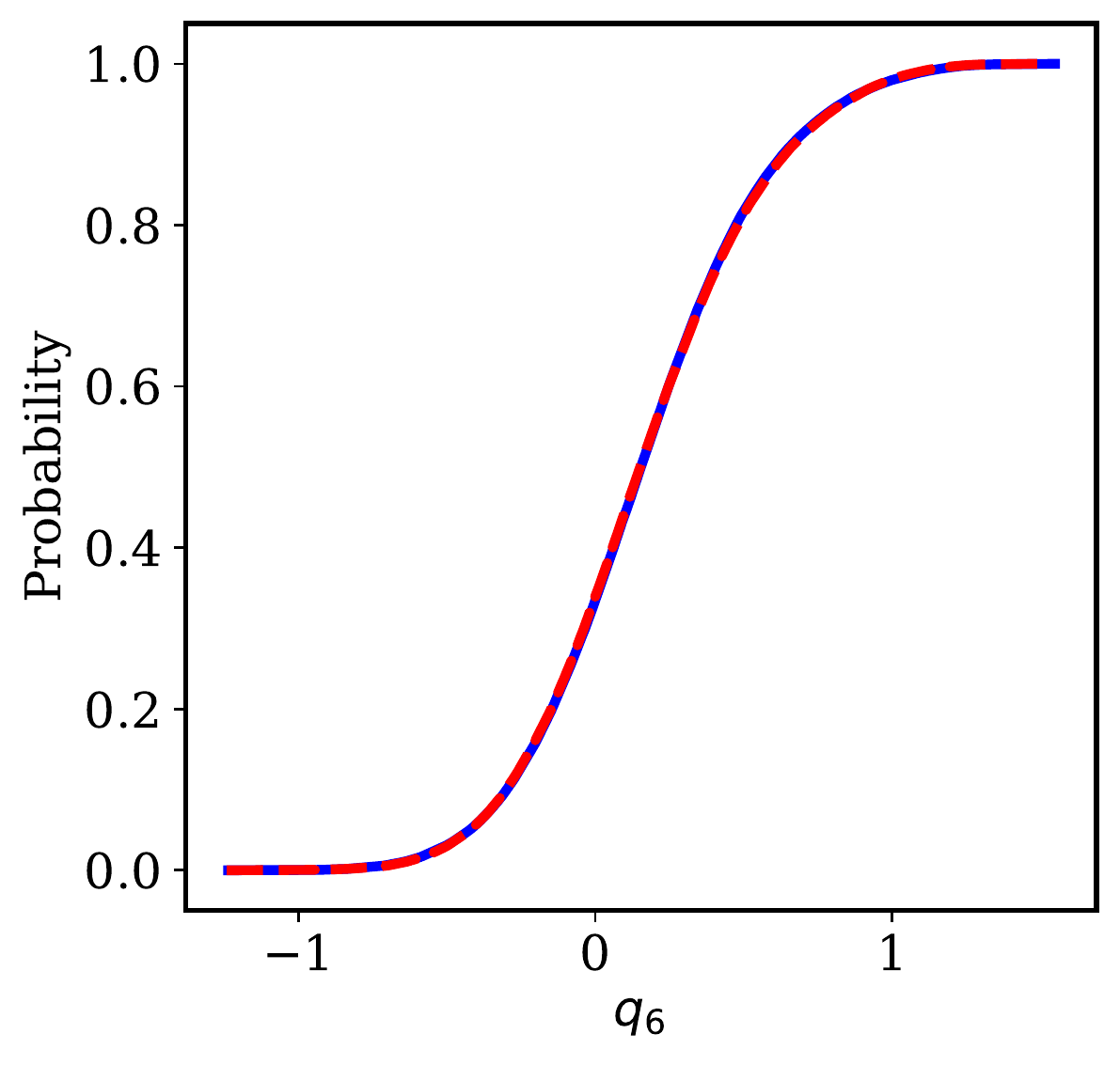} 
\caption{}
\label{10D_Rosen_ecdf_6}
\end{subfigure}
\begin{subfigure}{0.32\textwidth}
\centering  
\includegraphics[scale=0.32]{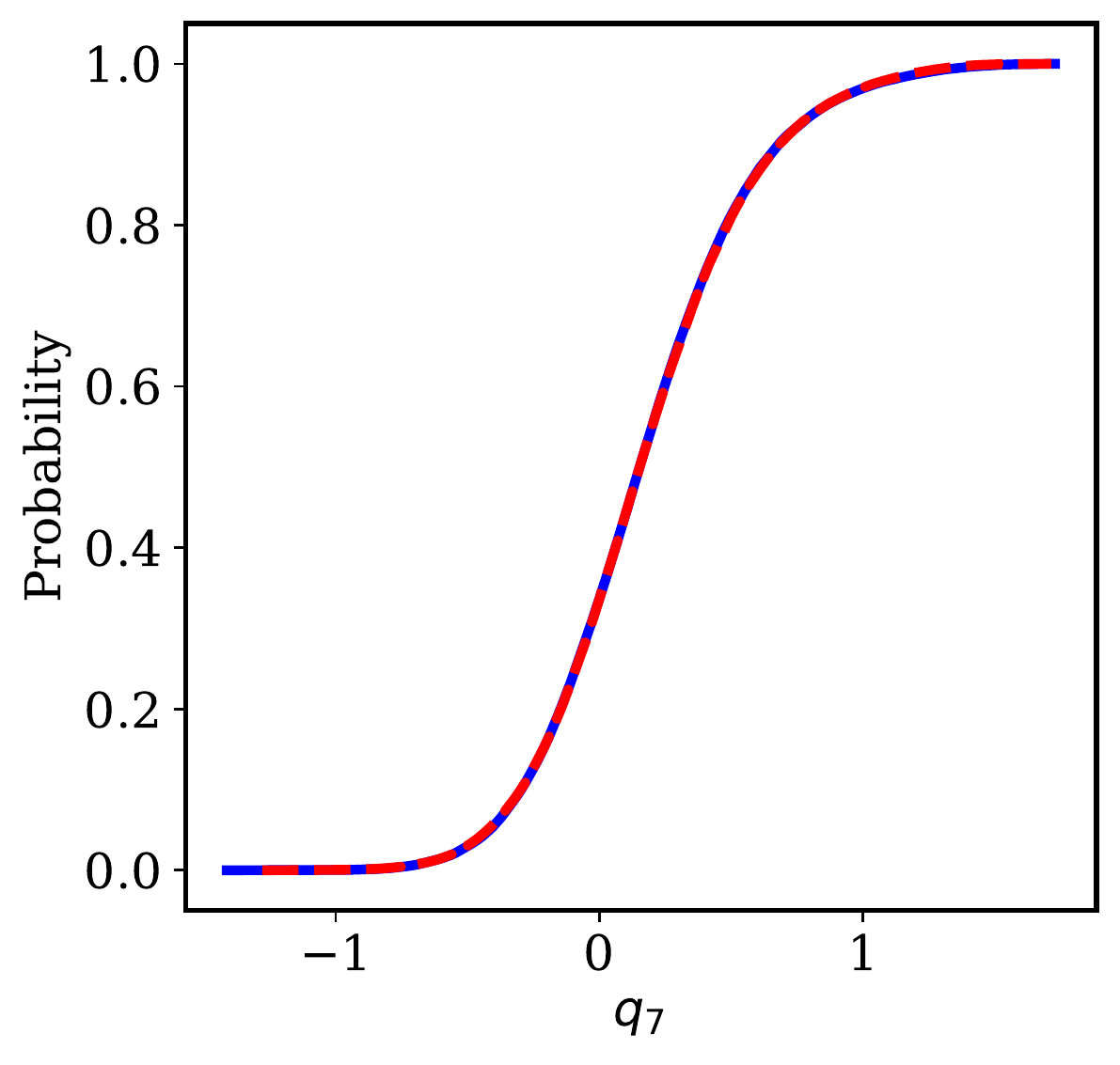} 
\caption{}
\label{10D_Rosen_ecdf_7}
\end{subfigure}
\begin{subfigure}{0.32\textwidth}
\centering  
\includegraphics[scale=0.32]{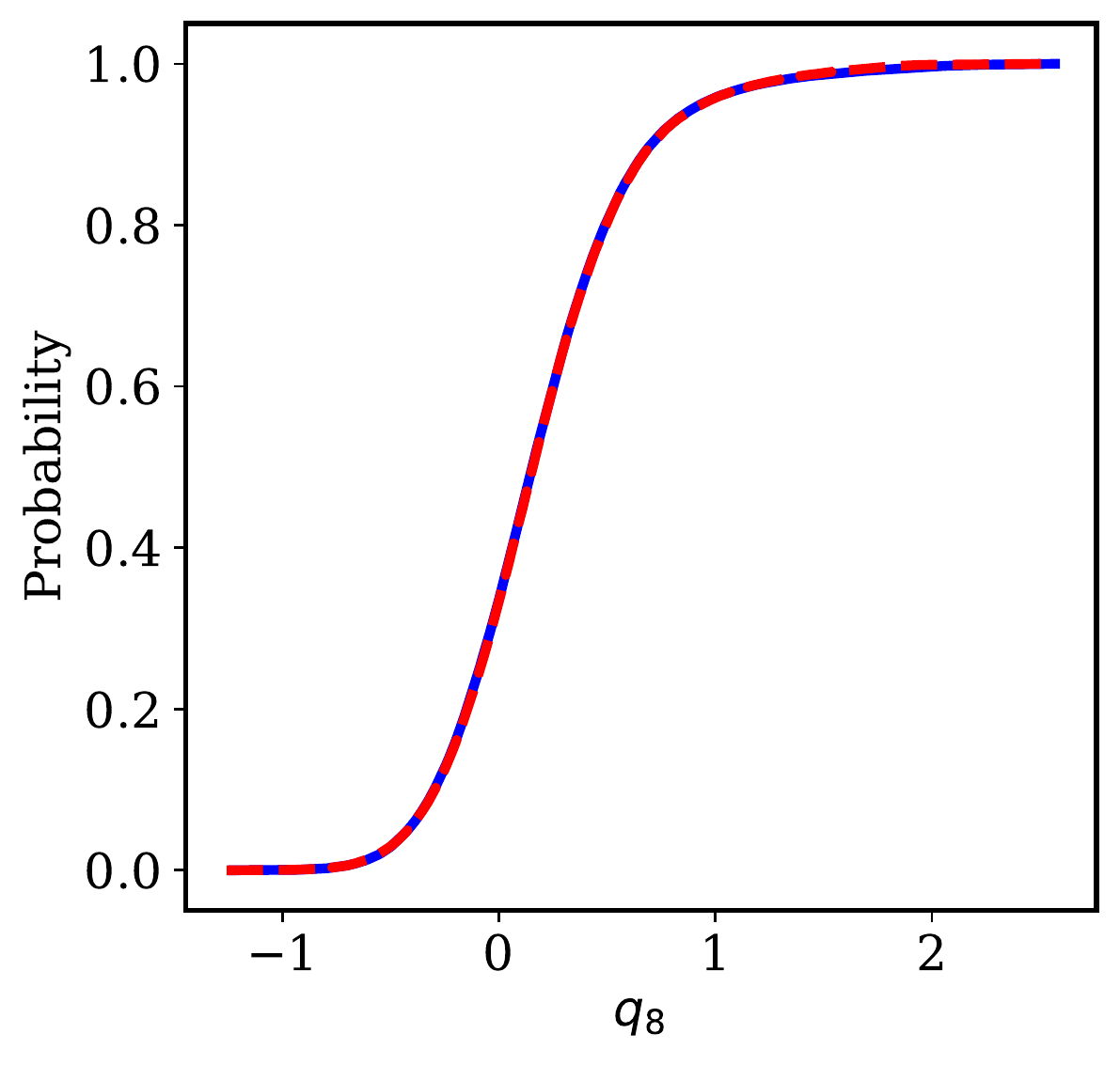} 
\caption{}
\label{10D_Rosen_ecdf_8}
\end{subfigure}
\begin{subfigure}{0.32\textwidth}
\centering  
\includegraphics[scale=0.32]{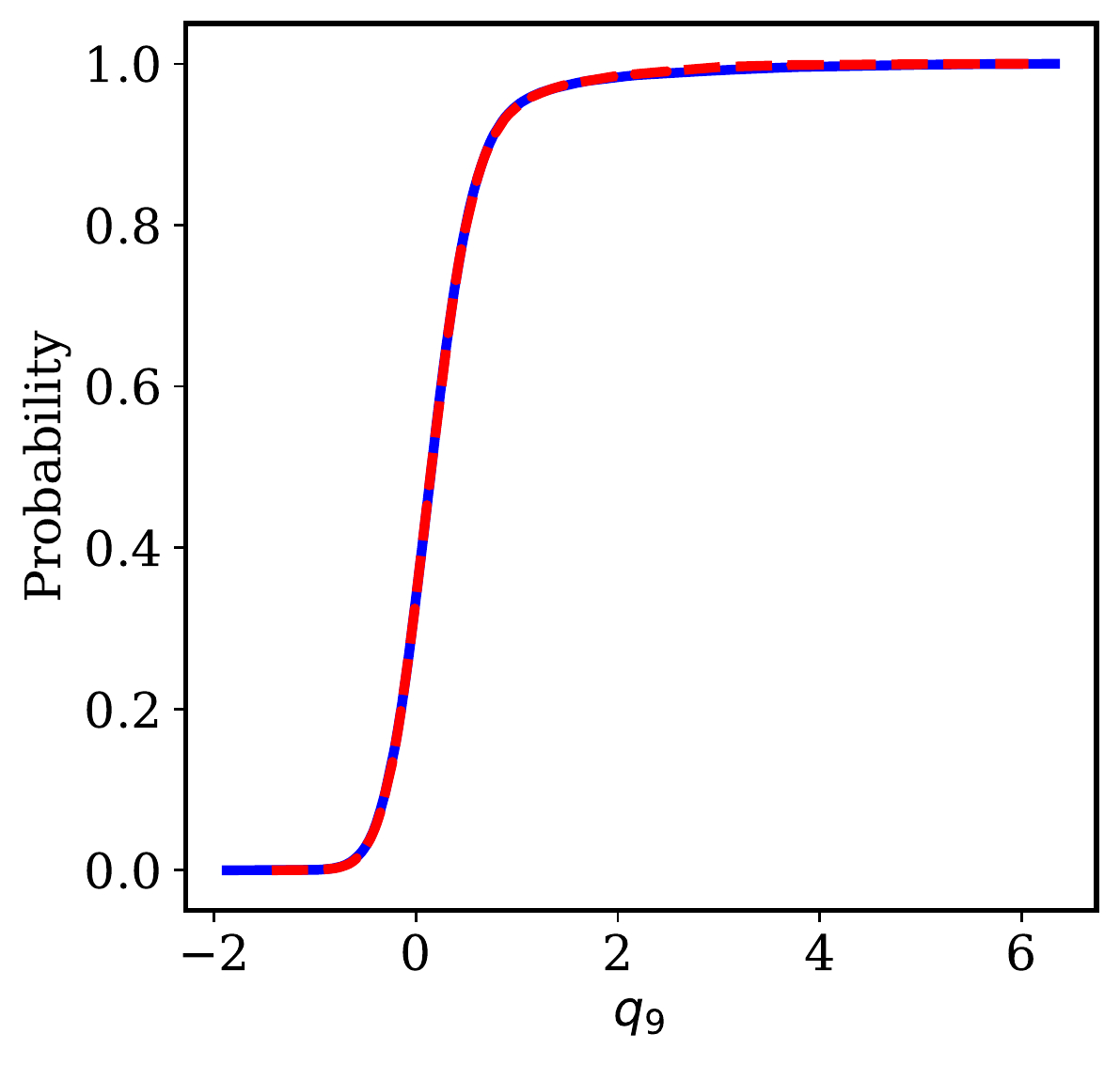} 
\caption{}
\label{10D_Rosen_ecdf_9}
\end{subfigure}
\begin{subfigure}{1.0\textwidth}
\centering  
\includegraphics[scale=0.32]{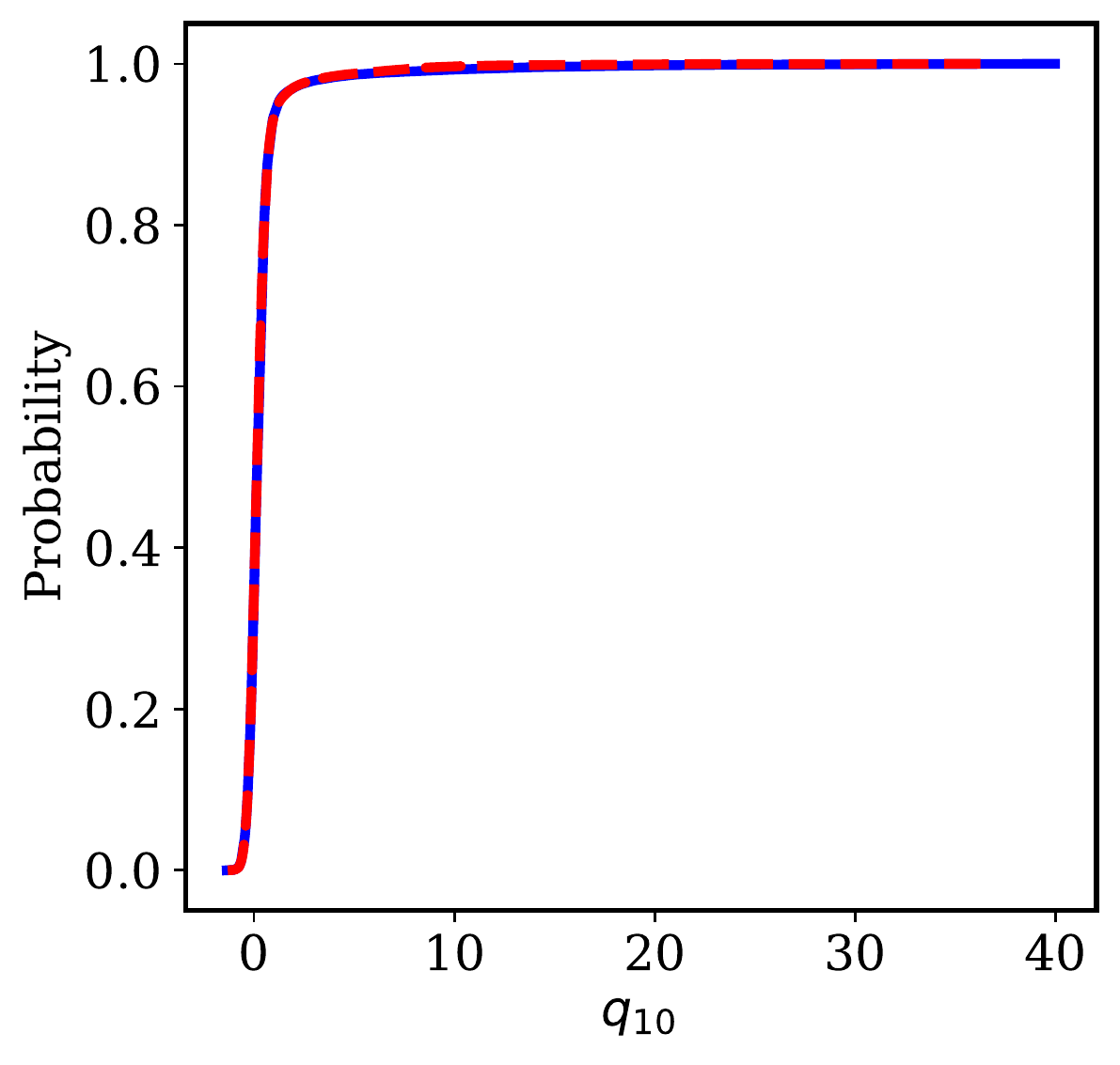} 
\caption{}
\label{10D_Rosen_ecdf_10}
\end{subfigure}
\caption{10D degenerate Rosenbrock density empirical cumulative distribution functions comparison.}
\label{10D_Rosen_ecdf}
\end{figure}

\begin{figure}[htbp]
\centering
\includegraphics[width=\textwidth]{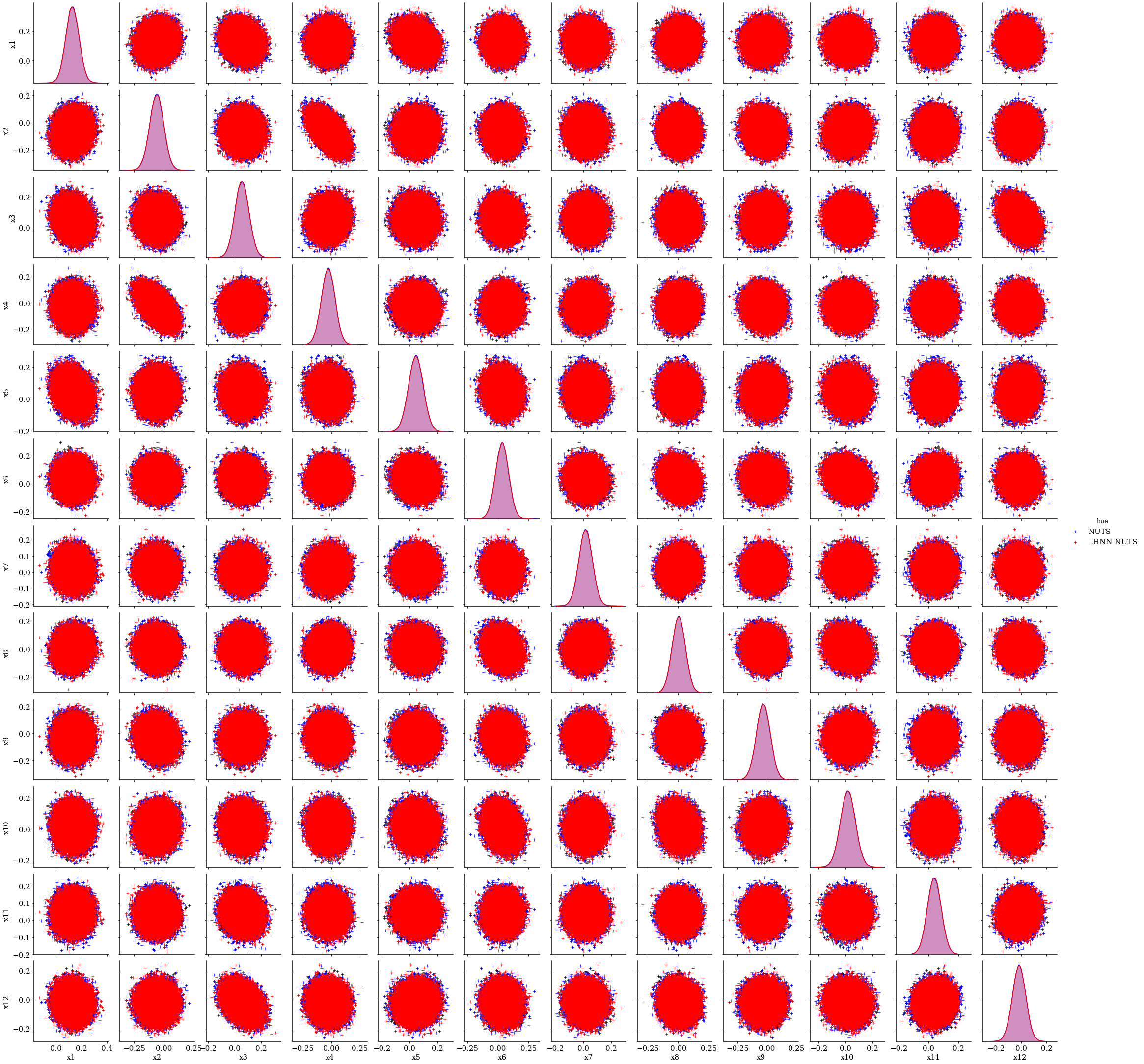} 
\caption{Comparison of scatter plots between NUTS and HNN-NUTS considering a 24D Bayesian logistic regression problem. Dimensions 1 to 12 are presented here. (blue dots: traditional NUTS; red dots: HNN-NUTS)}
\label{12D_Logistic_1}
\end{figure}

\begin{figure}[htbp]
\centering
\includegraphics[width=\textwidth]{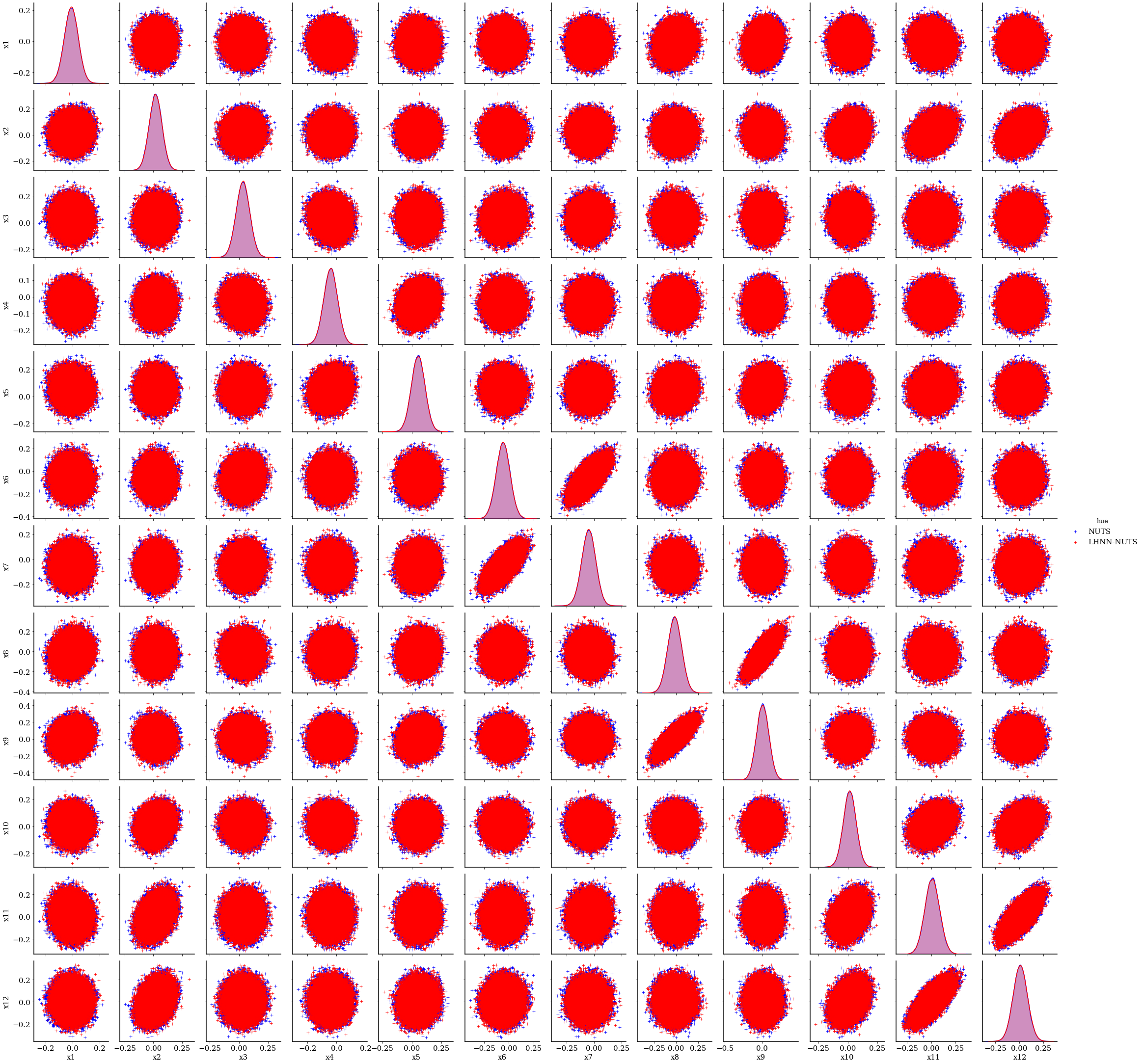} 
\caption{Comparison of scatter plots between NUTS and HNN-NUTS considering a 24D Bayesian logistic regression problem. Dimensions 13 to 24 are presented here. (blue dots: traditional NUTS; red dots: HNN-NUTS)}
\label{12D_Logistic_2}
\end{figure}

\begin{figure}[htbp]
\begin{subfigure}{0.32\textwidth}
\centering  
\includegraphics[scale=0.32]{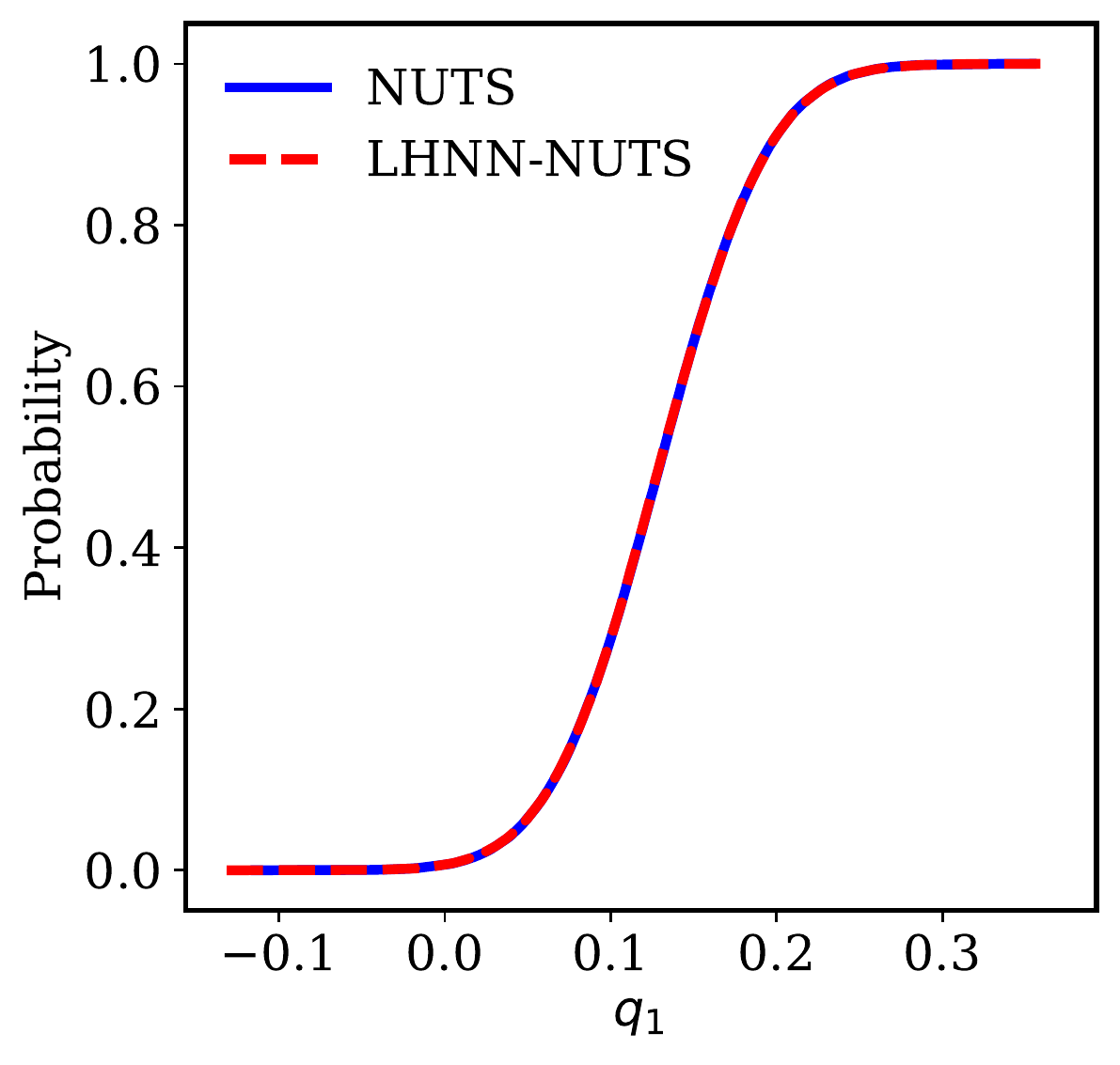} 
\caption{}
\label{24D_Logistic_ecdf_1}
\end{subfigure}
\begin{subfigure}{0.32\textwidth}
\centering  
\includegraphics[scale=0.32]{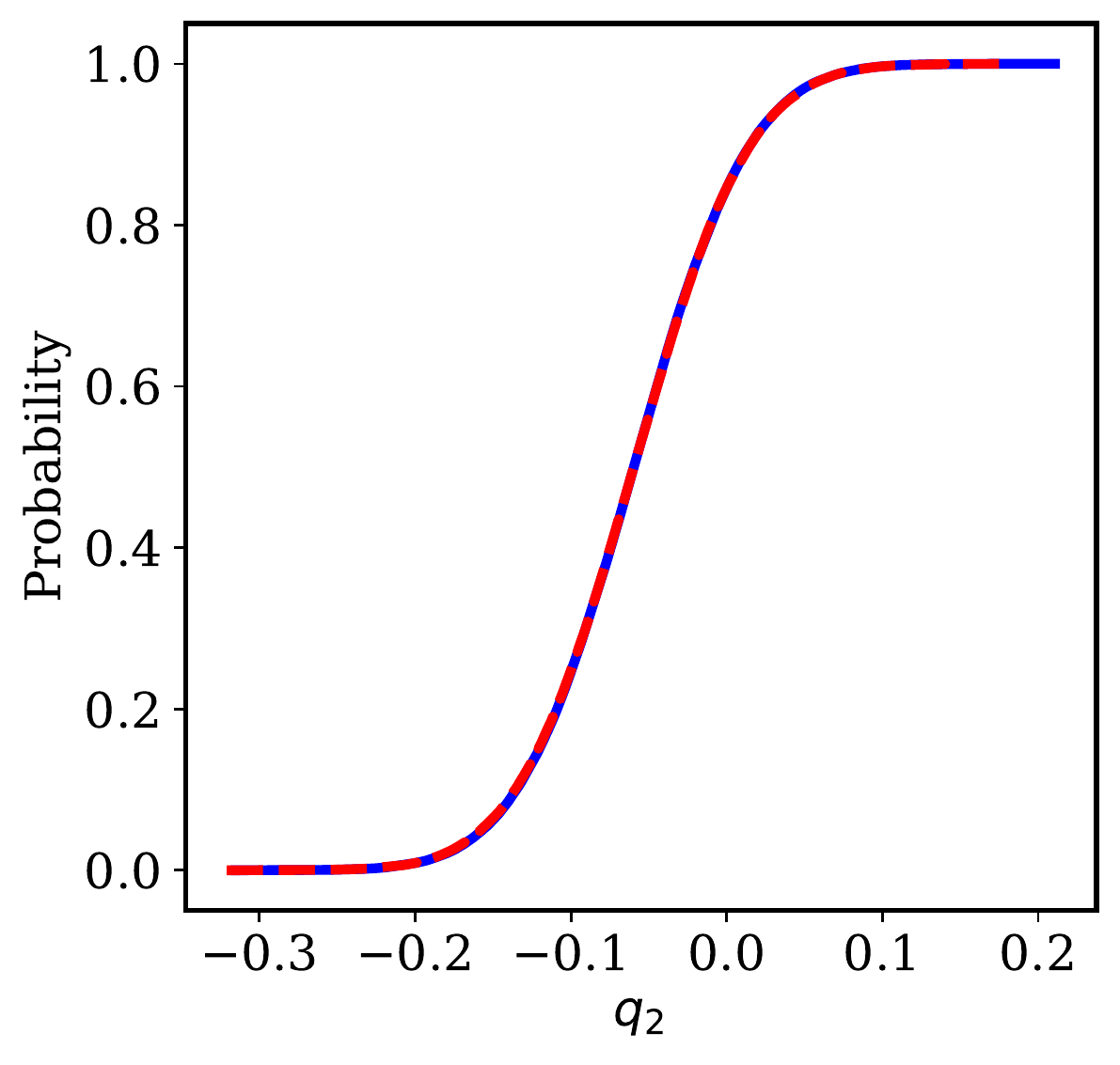} 
\caption{}
\label{24D_Logistic_ecdf_2}
\end{subfigure}
\begin{subfigure}{0.32\textwidth}
\centering  
\includegraphics[scale=0.32]{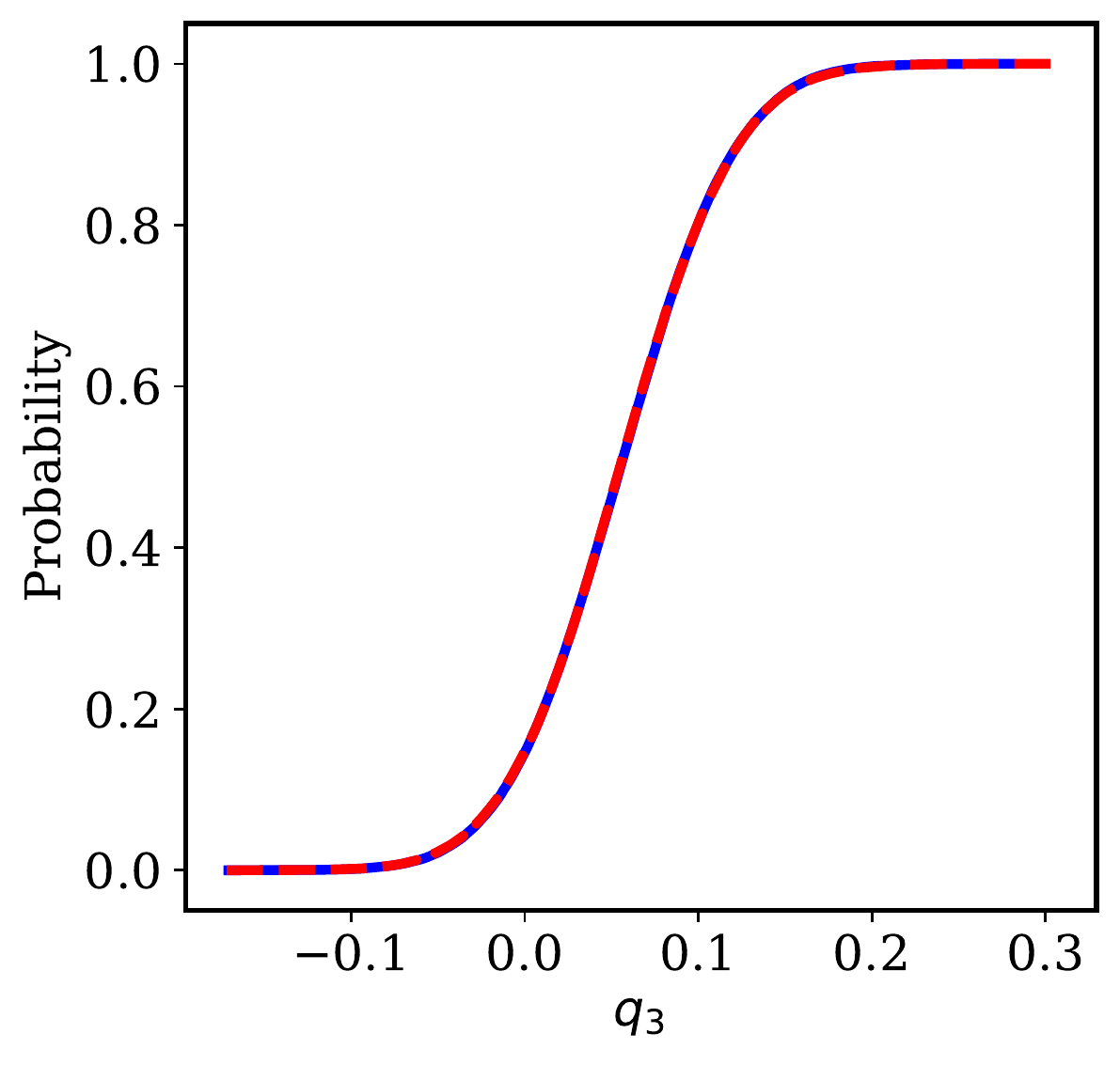} 
\caption{}
\label{24D_Logistic_ecdf_3}
\end{subfigure}
\begin{subfigure}{0.32\textwidth}
\centering  
\includegraphics[scale=0.32]{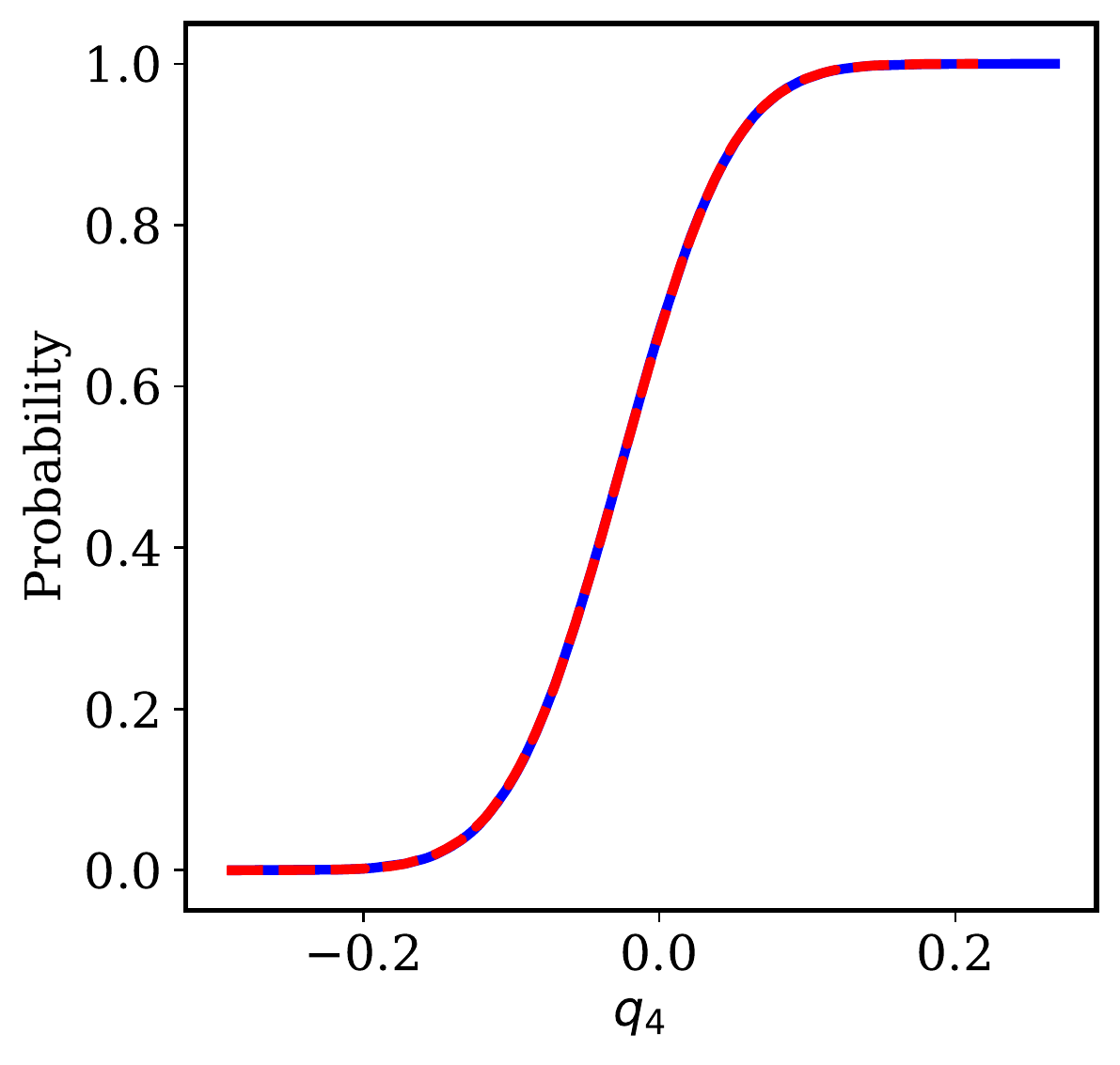} 
\caption{}
\label{24D_Logistic_ecdf_4}
\end{subfigure}
\begin{subfigure}{0.32\textwidth}
\centering  
\includegraphics[scale=0.32]{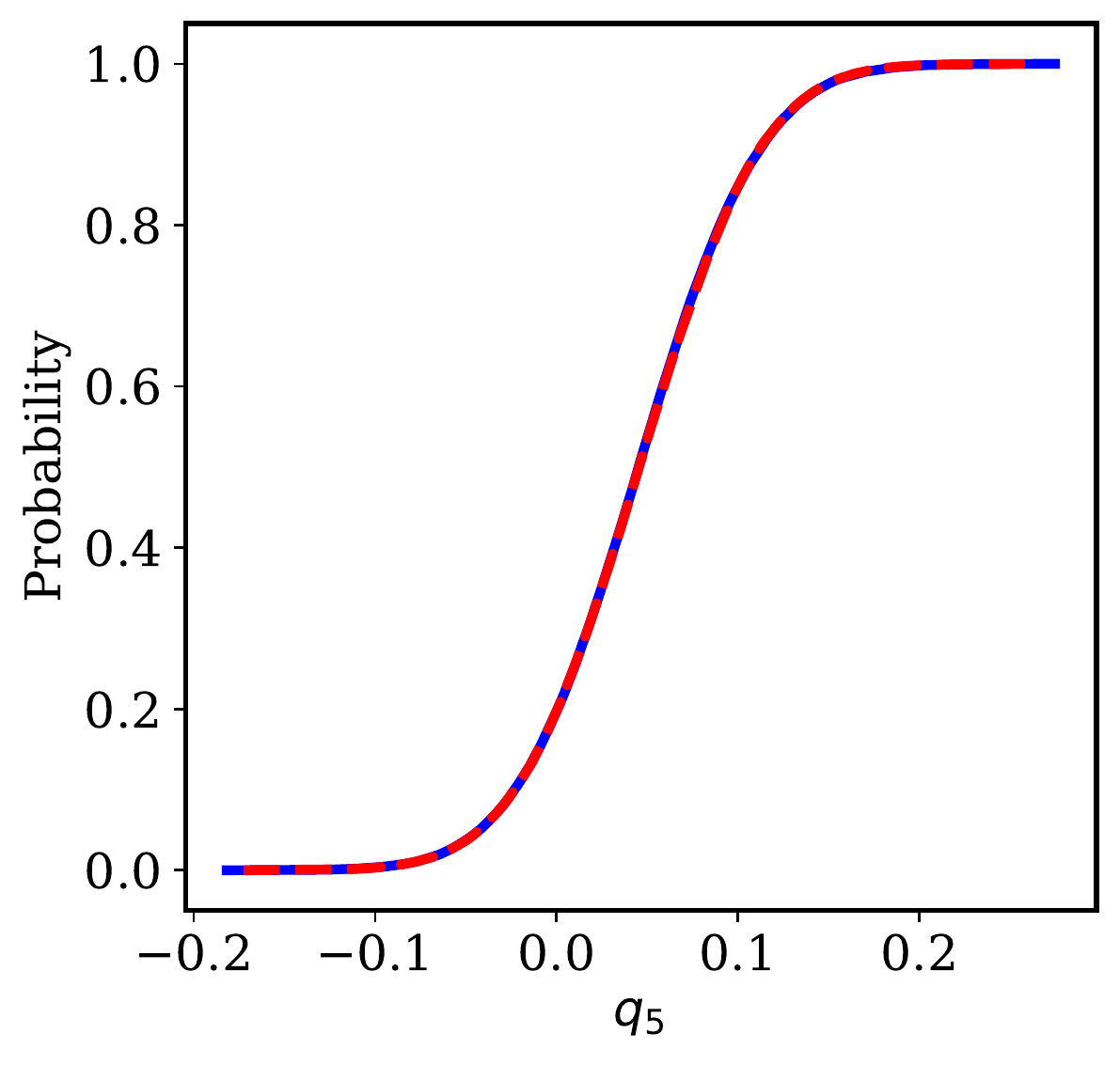} 
\caption{}
\label{24D_Logistic_ecdf_5}
\end{subfigure}
\begin{subfigure}{0.32\textwidth}
\centering  
\includegraphics[scale=0.32]{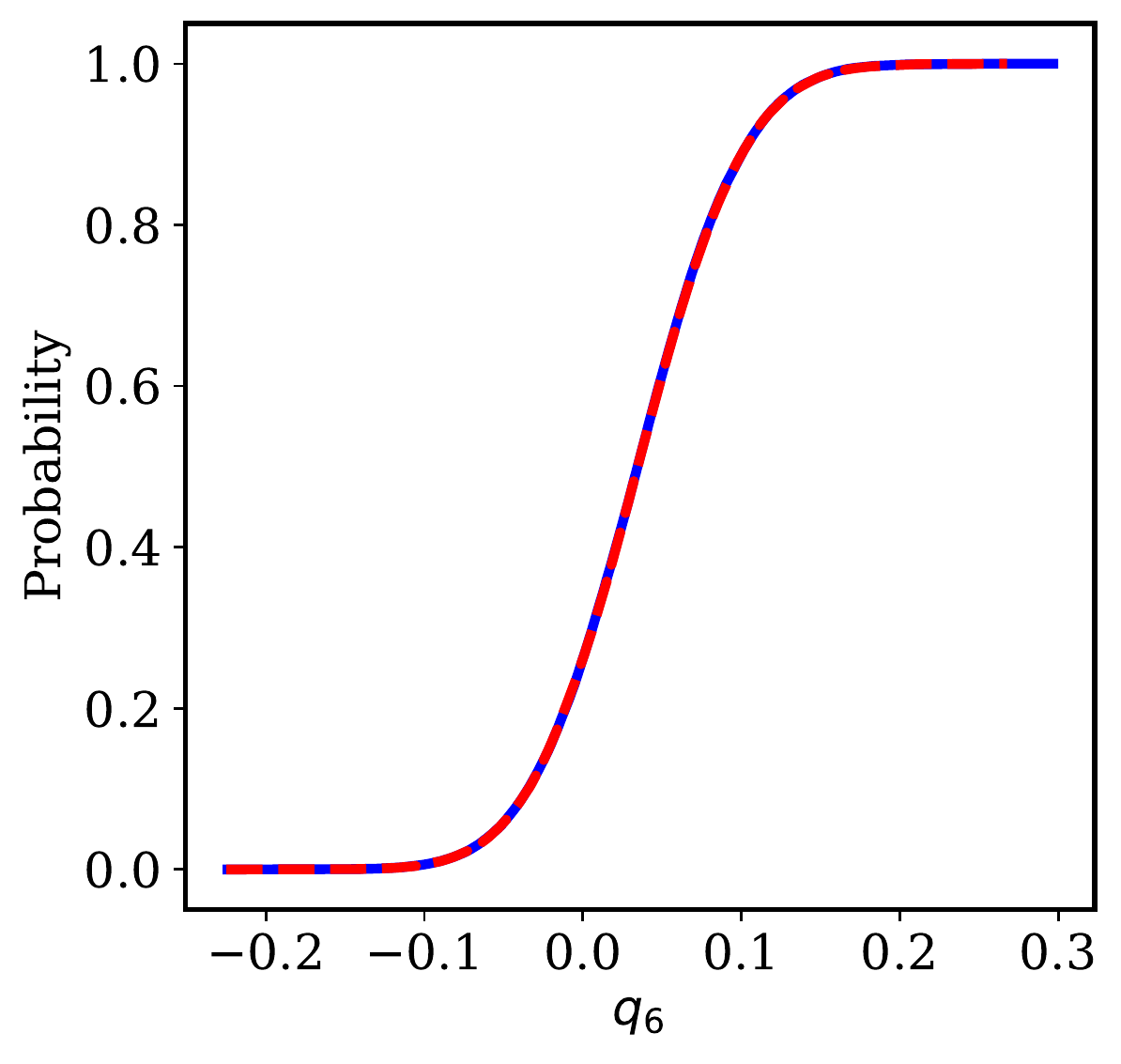} 
\caption{}
\label{24D_Logistic_ecdf_6}
\end{subfigure}
\begin{subfigure}{0.32\textwidth}
\centering  
\includegraphics[scale=0.32]{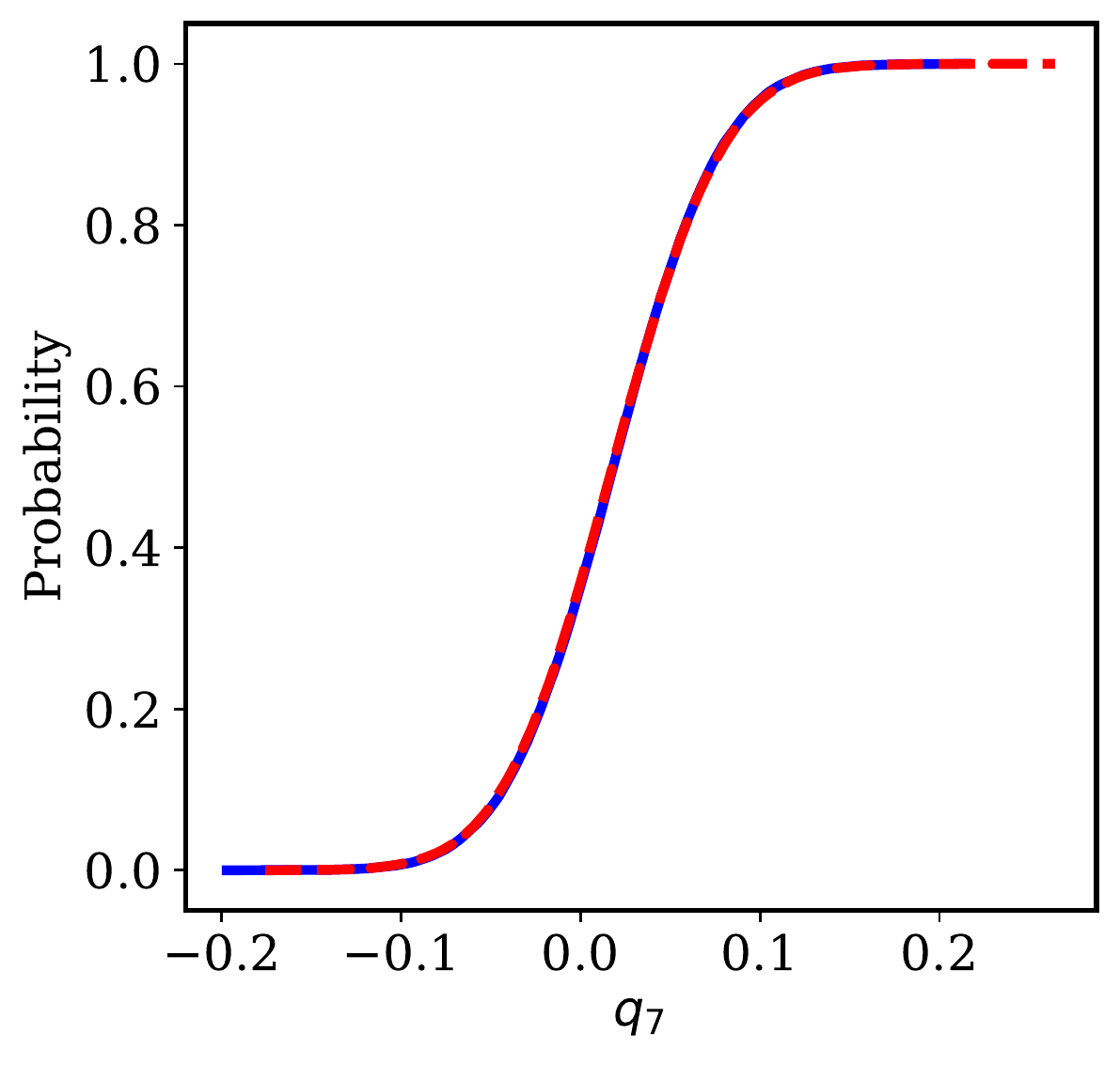} 
\caption{}
\label{24D_Logistic_ecdf_7}
\end{subfigure}
\begin{subfigure}{0.32\textwidth}
\centering  
\includegraphics[scale=0.32]{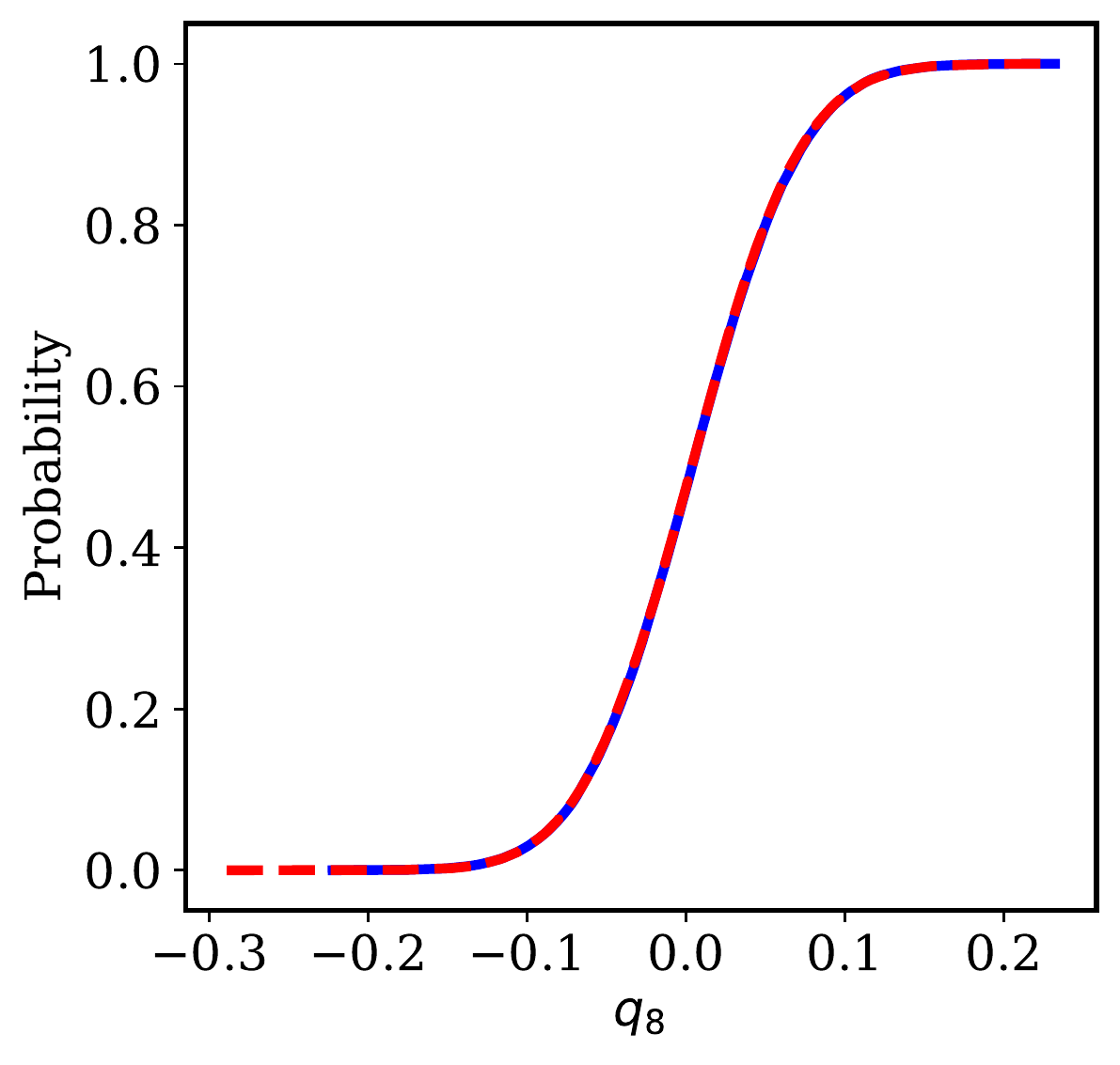} 
\caption{}
\label{24D_Logistic_ecdf_8}
\end{subfigure}
\begin{subfigure}{0.32\textwidth}
\centering  
\includegraphics[scale=0.32]{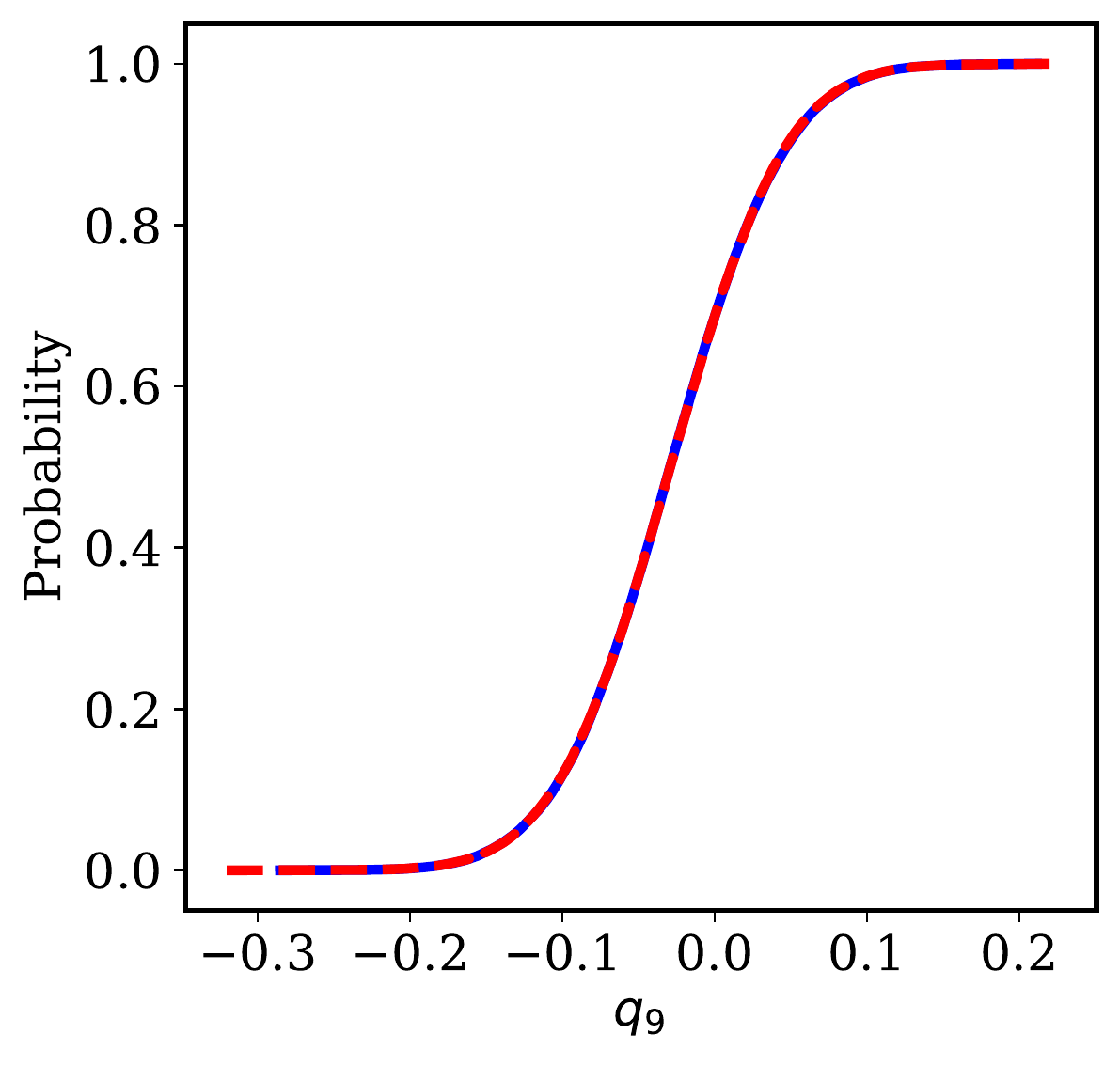} 
\caption{}
\label{24D_Logistic_ecdf_9}
\end{subfigure}
\begin{subfigure}{0.32\textwidth}
\centering  
\includegraphics[scale=0.32]{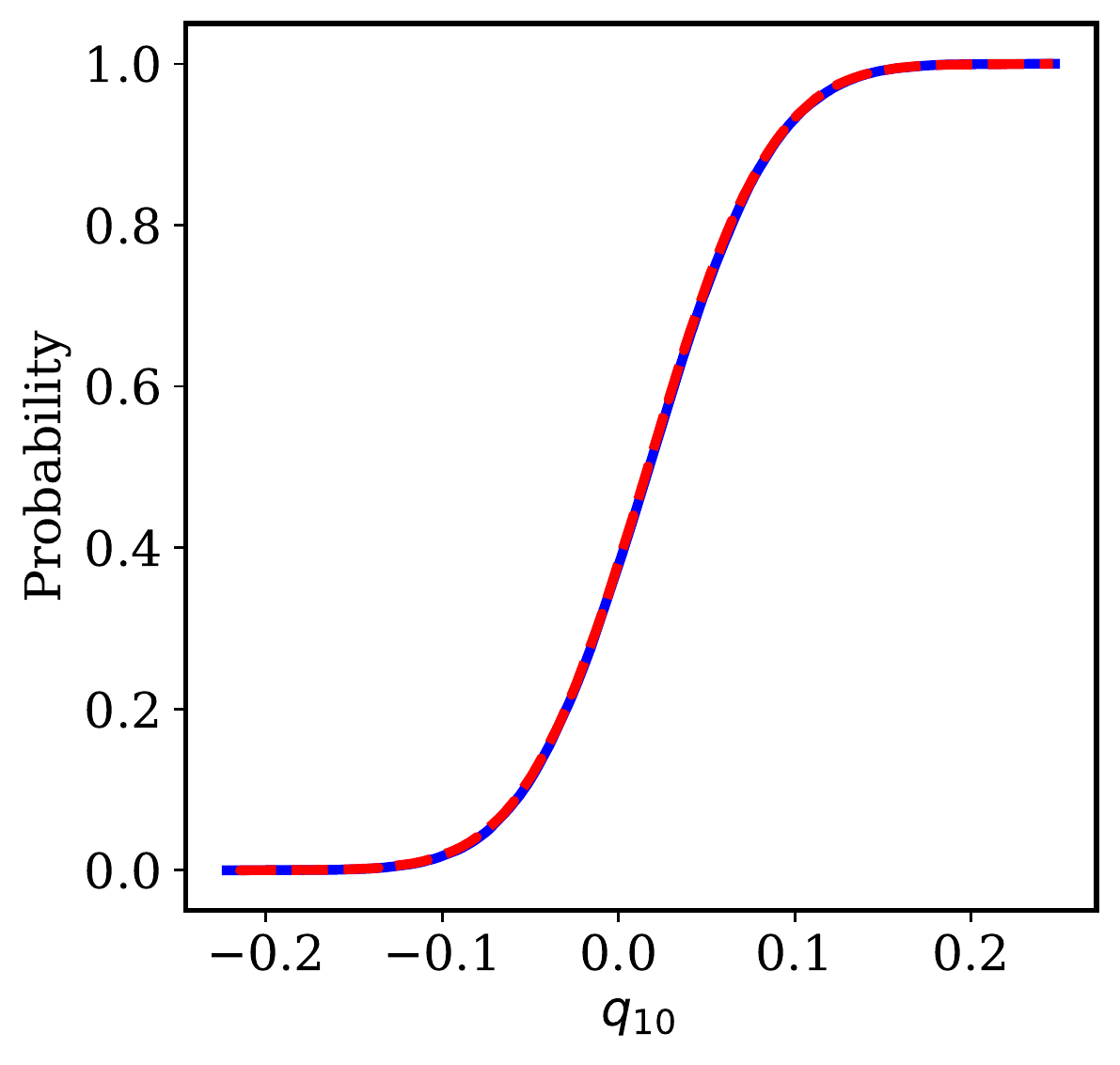} 
\caption{}
\label{24D_Logistic_ecdf_10}
\end{subfigure}
\begin{subfigure}{0.32\textwidth}
\centering  
\includegraphics[scale=0.32]{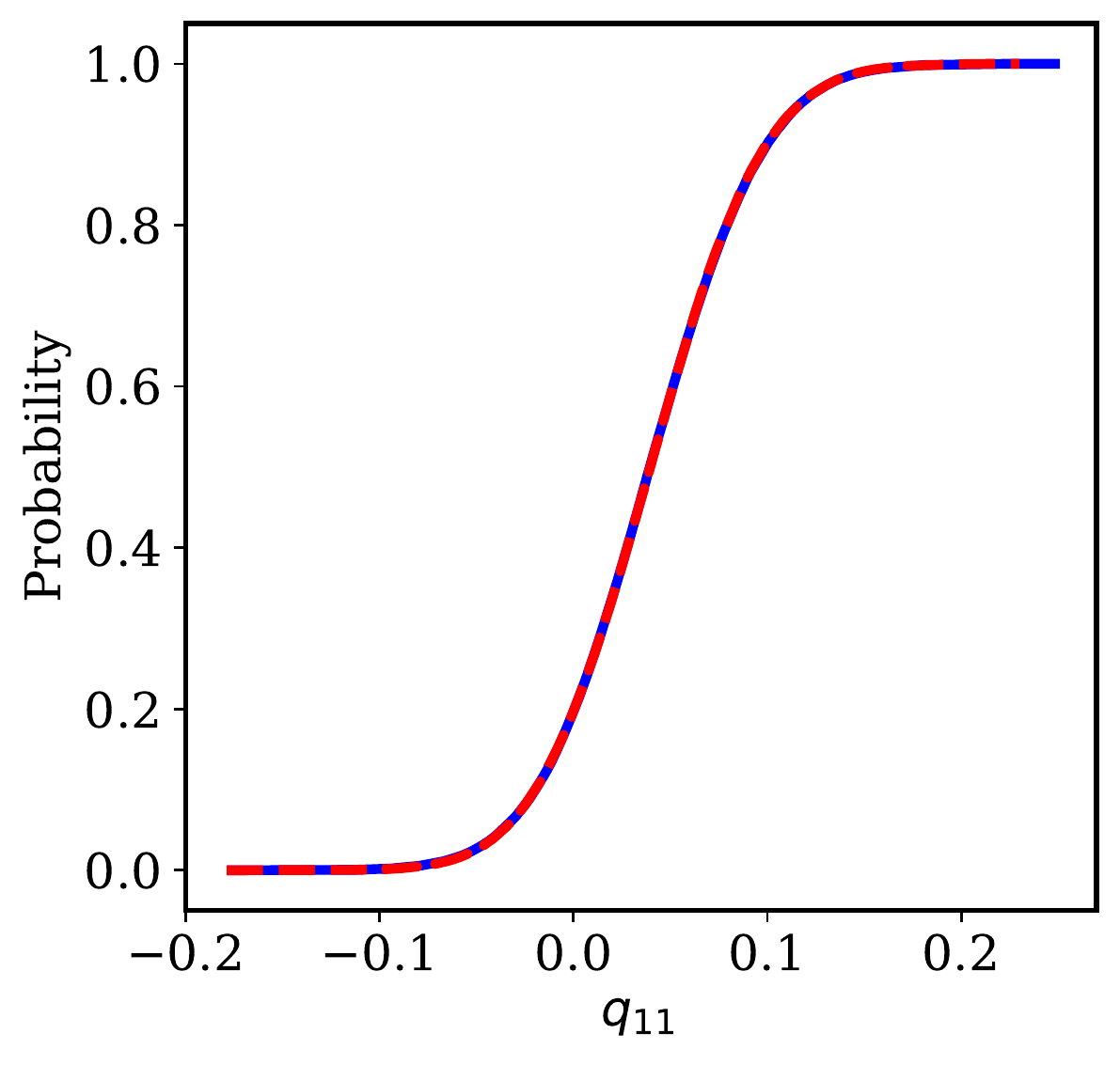} 
\caption{}
\label{24D_Logistic_ecdf_11}
\end{subfigure}
\begin{subfigure}{0.32\textwidth}
\centering  
\includegraphics[scale=0.32]{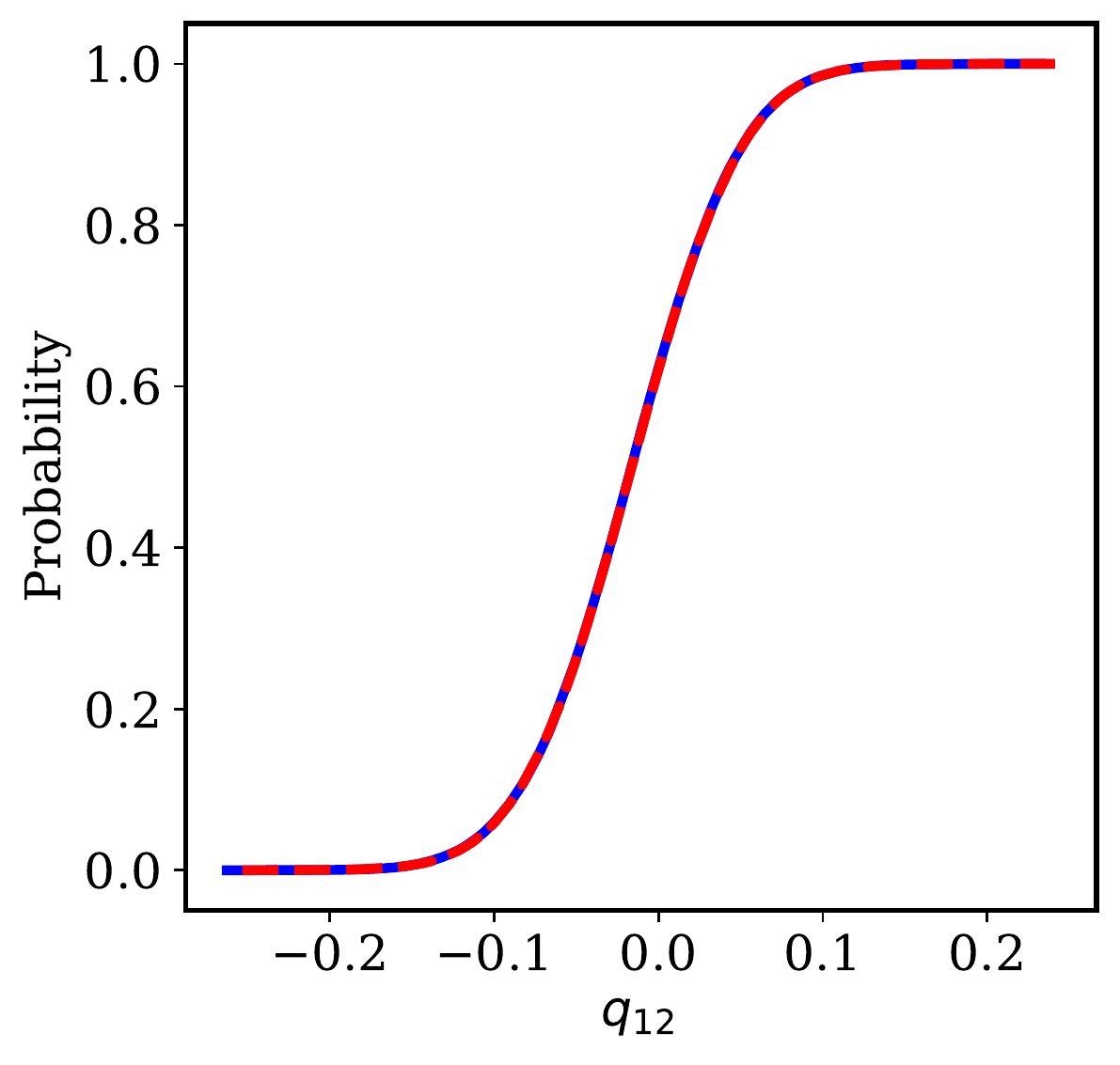} 
\caption{}
\label{24D_Logistic_ecdf_12}
\end{subfigure}
\caption{24D Bayesian logistic regression empirical cumulative distribution functions comparison for dimensions 1 to 12.}
\label{24D_Logistic_ecdf_dim1}
\end{figure}

\begin{figure}[htbp]
\begin{subfigure}{0.32\textwidth}
\centering  
\includegraphics[scale=0.32]{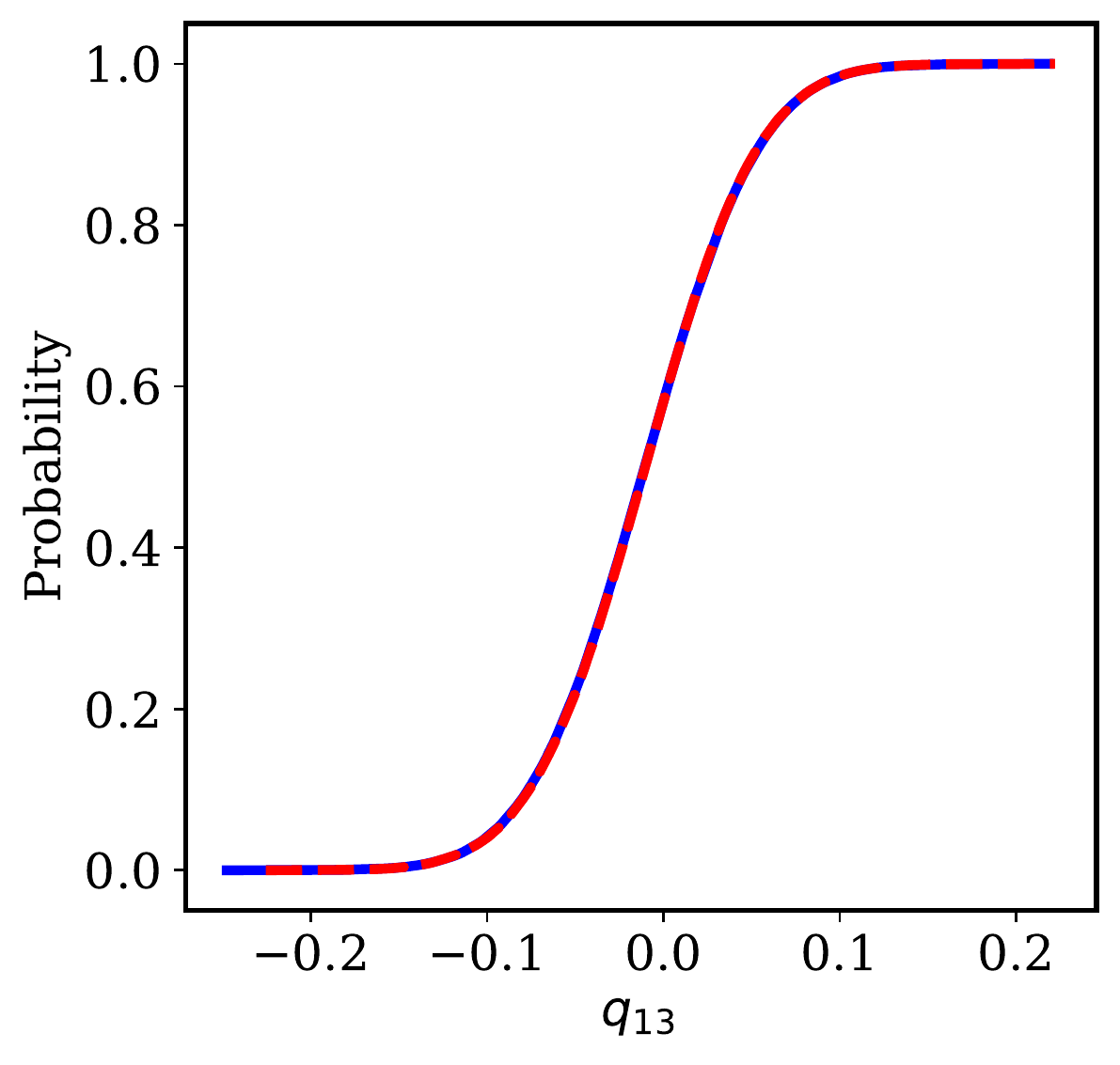} 
\caption{}
\label{24D_Logistic_ecdf_13}
\end{subfigure}
\begin{subfigure}{0.32\textwidth}
\centering  
\includegraphics[scale=0.32]{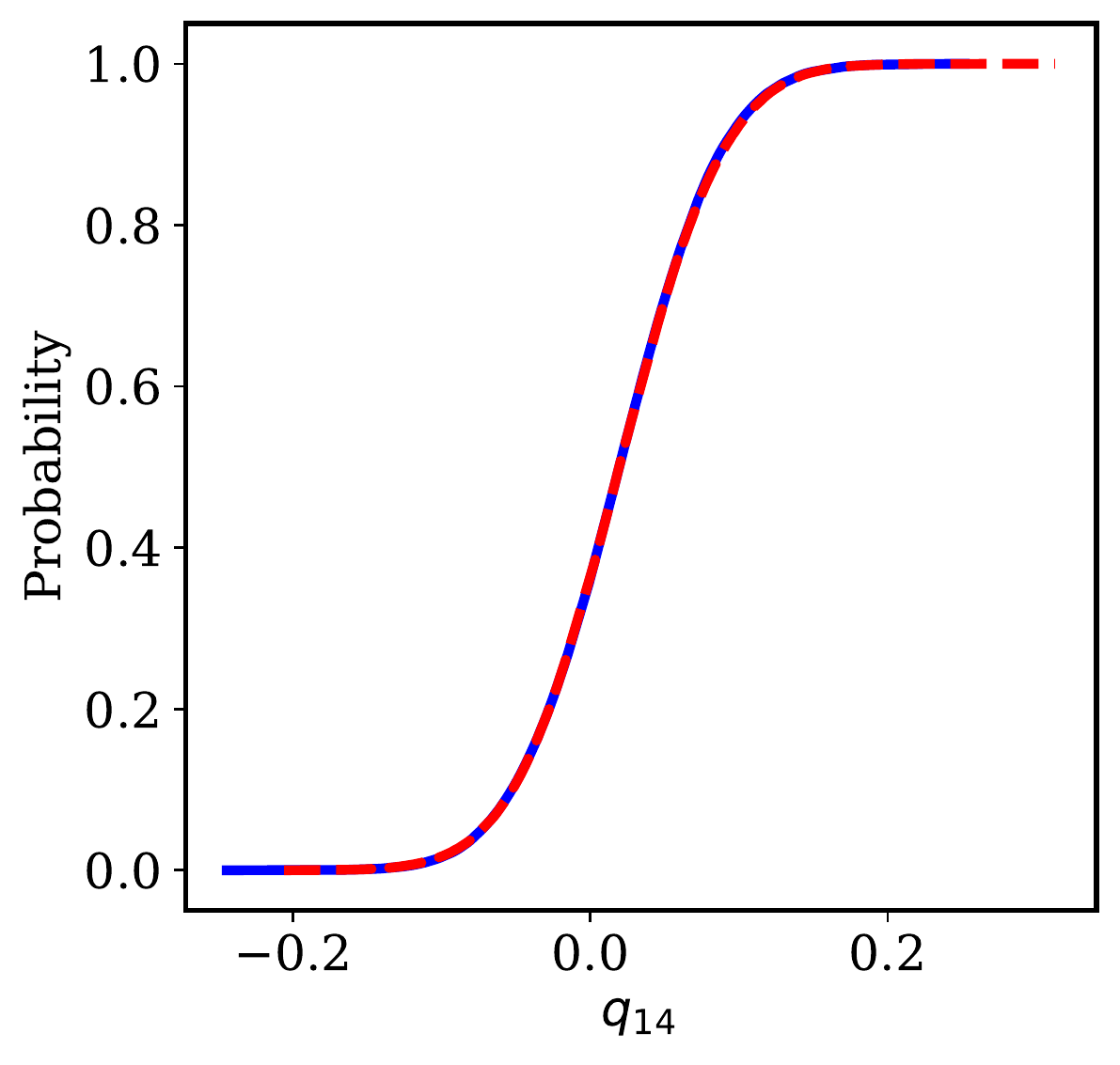} 
\caption{}
\label{24D_Logistic_ecdf_14}
\end{subfigure}
\begin{subfigure}{0.32\textwidth}
\centering  
\includegraphics[scale=0.32]{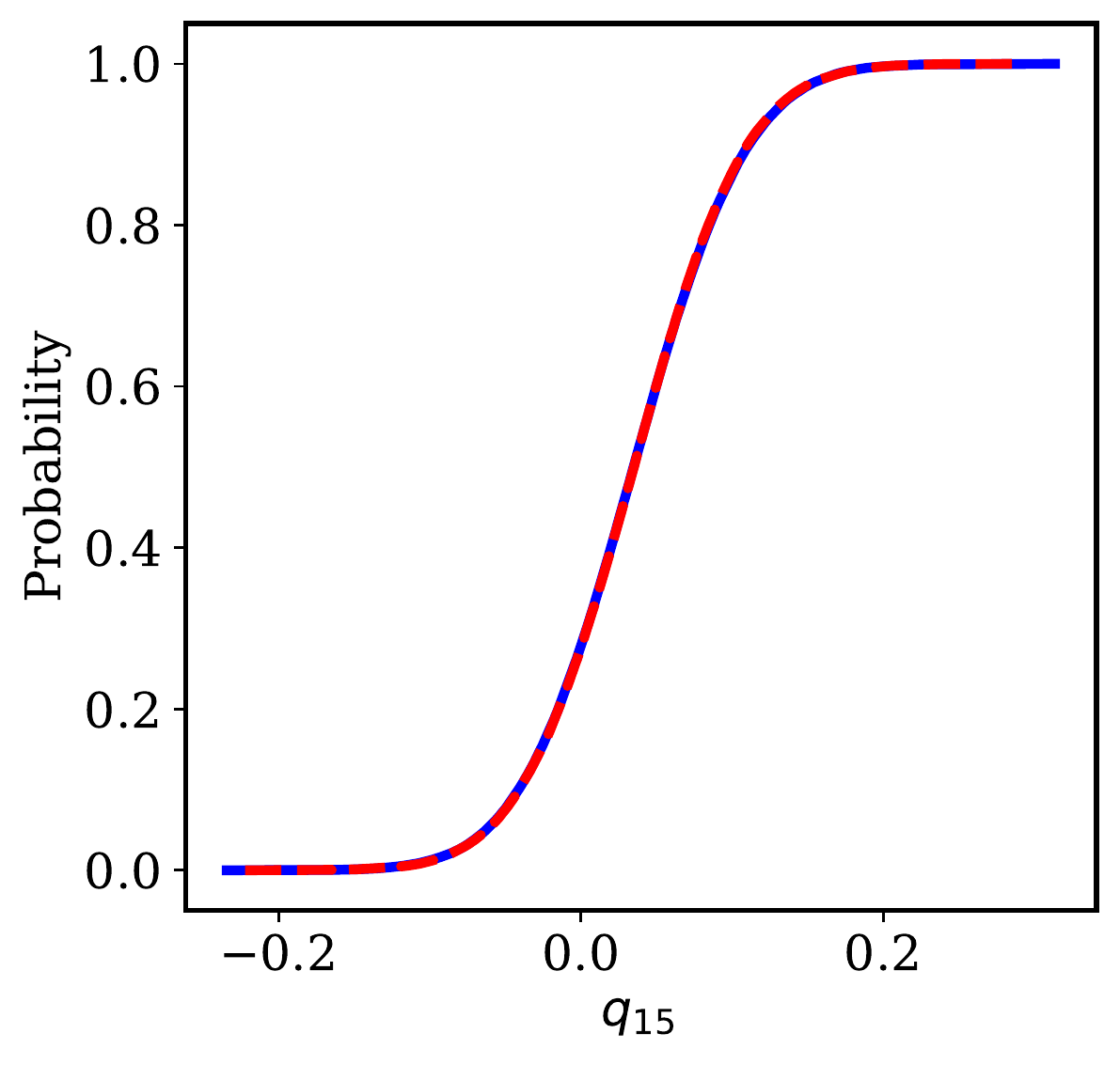} 
\caption{}
\label{24D_Logistic_ecdf_15}
\end{subfigure}
\begin{subfigure}{0.32\textwidth}
\centering  
\includegraphics[scale=0.32]{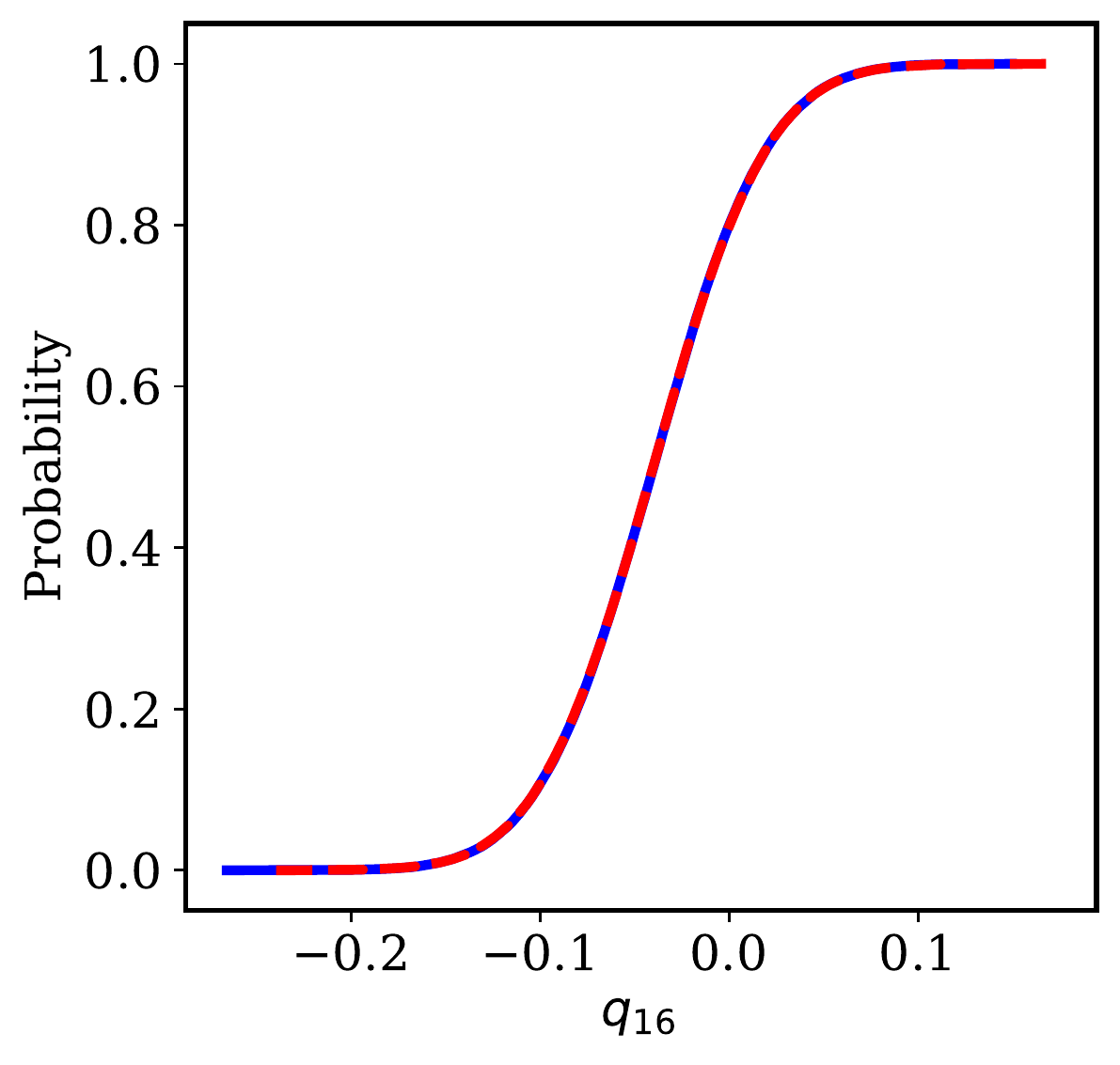} 
\caption{}
\label{24D_Logistic_ecdf_16}
\end{subfigure}
\begin{subfigure}{0.32\textwidth}
\centering  
\includegraphics[scale=0.32]{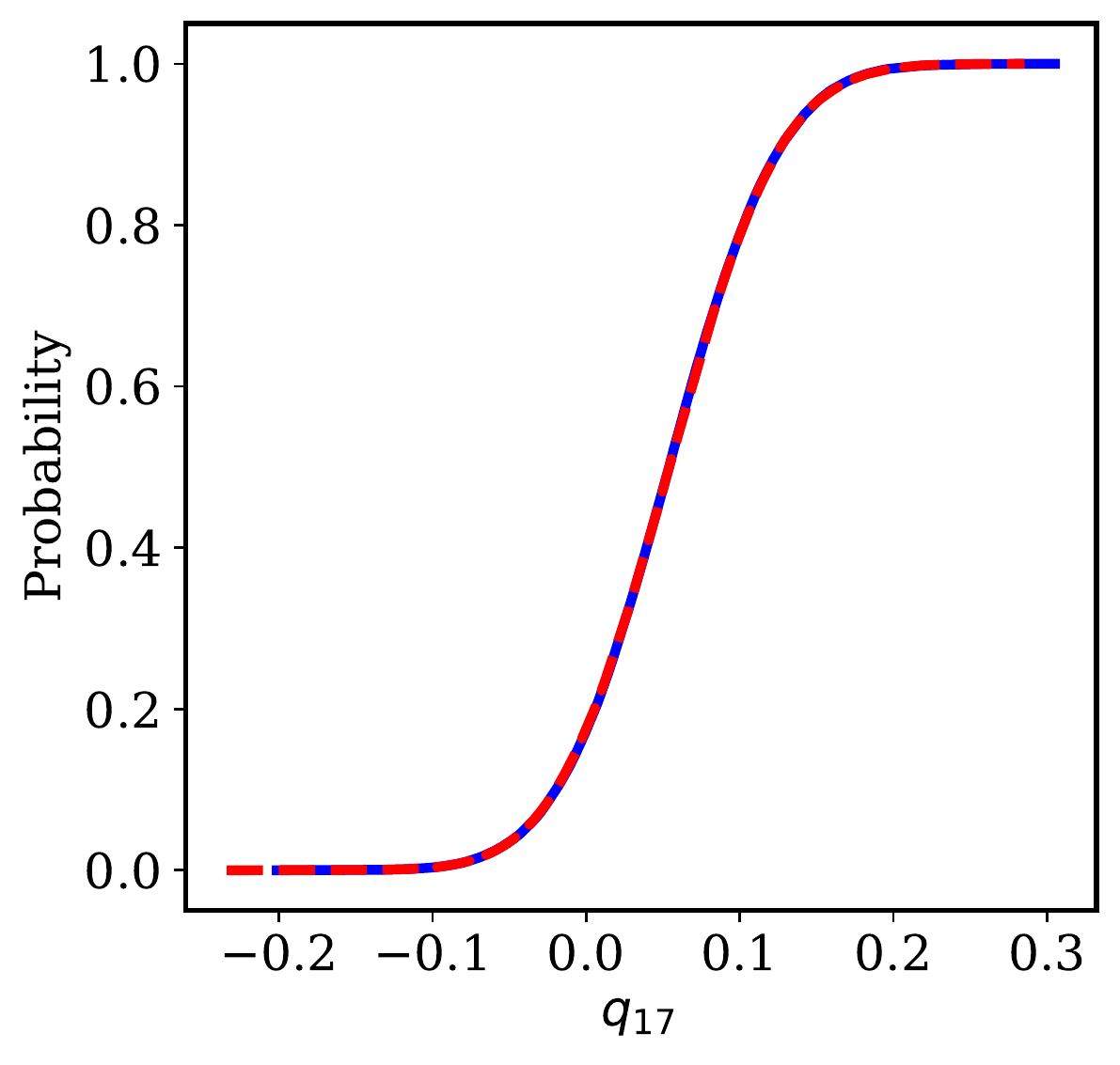} 
\caption{}
\label{24D_Logistic_ecdf_17}
\end{subfigure}
\begin{subfigure}{0.32\textwidth}
\centering  
\includegraphics[scale=0.32]{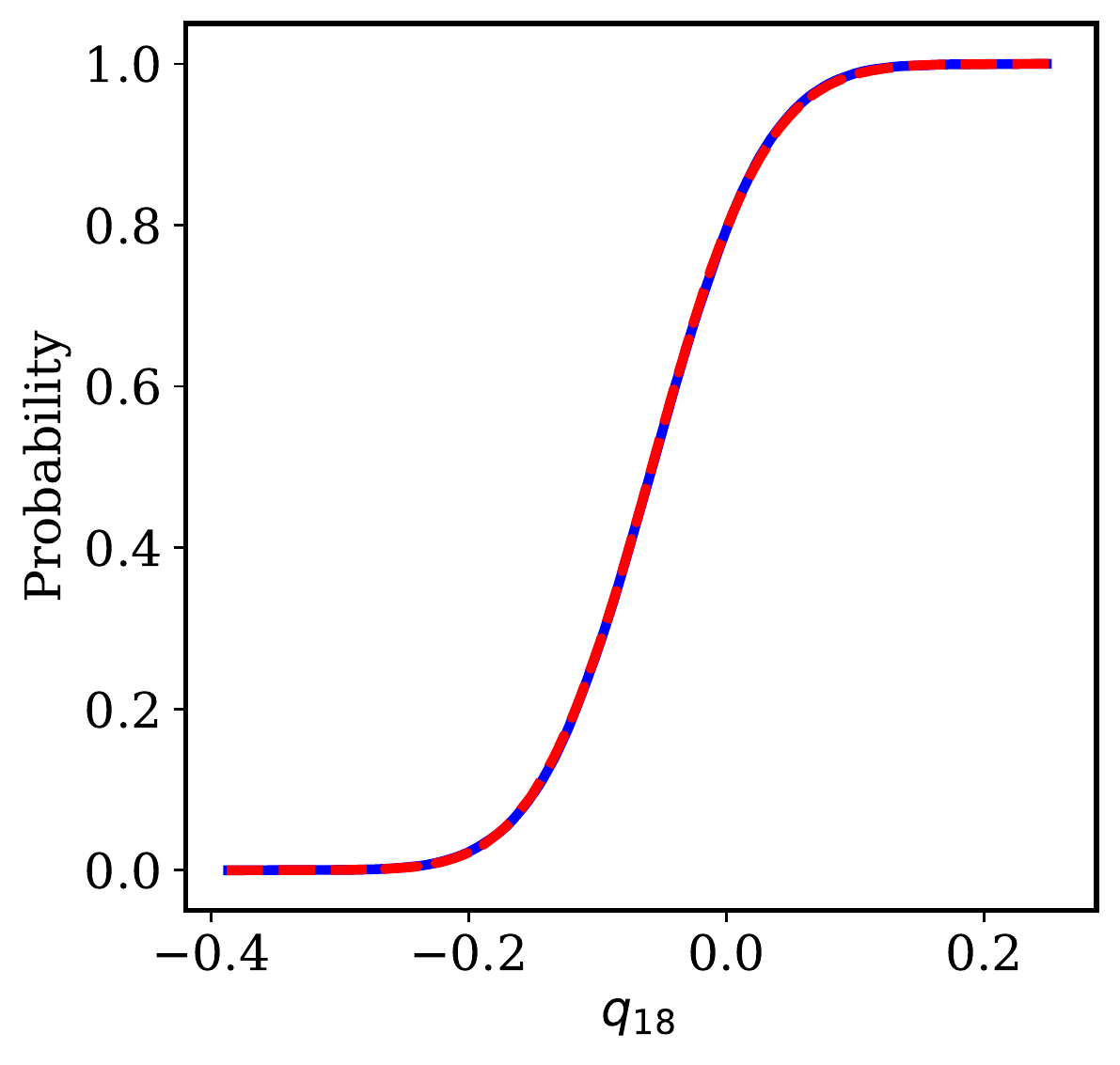} 
\caption{}
\label{24D_Logistic_ecdf_18}
\end{subfigure}
\begin{subfigure}{0.32\textwidth}
\centering  
\includegraphics[scale=0.32]{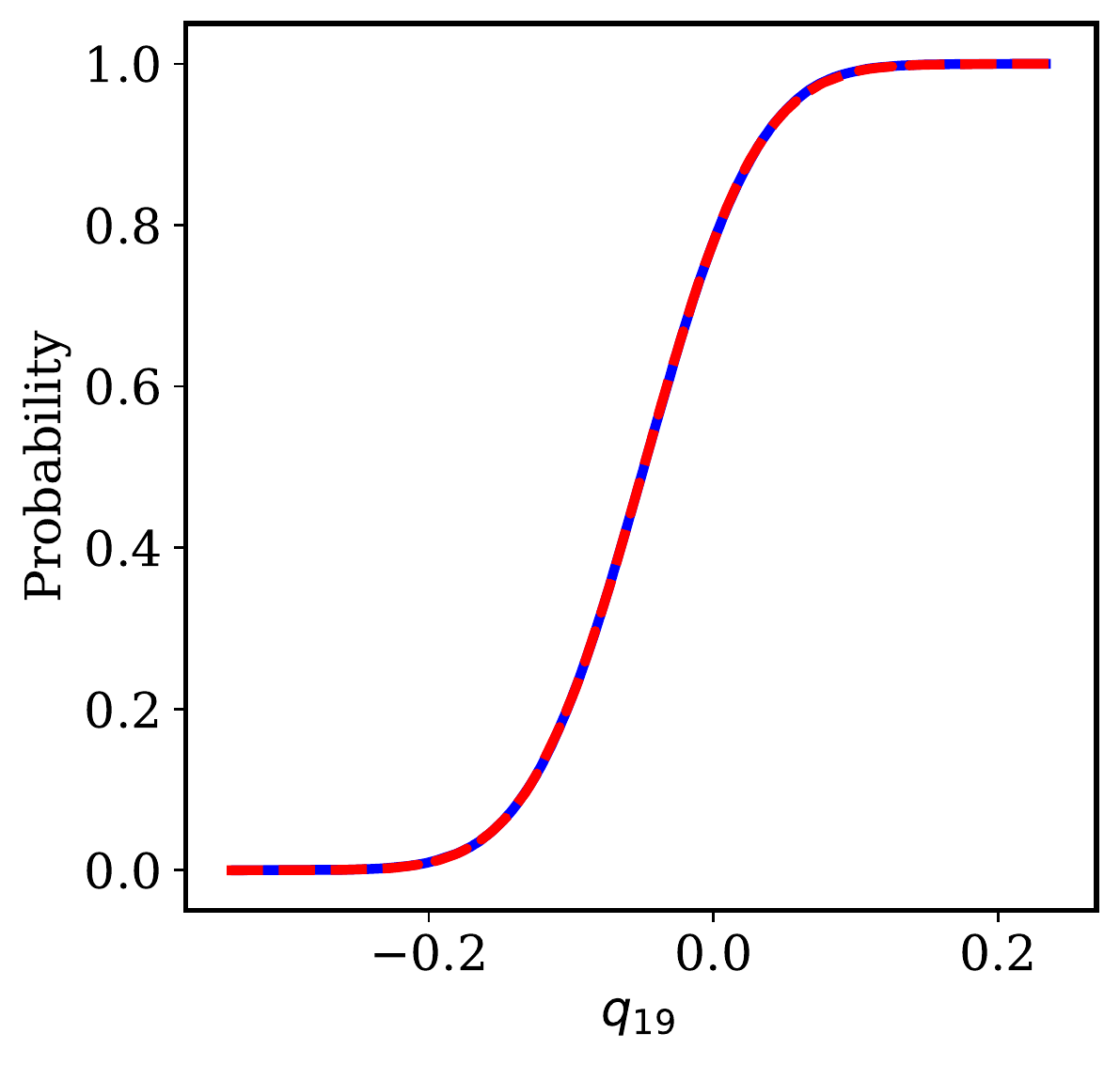} 
\caption{}
\label{24D_Logistic_ecdf_19}
\end{subfigure}
\begin{subfigure}{0.32\textwidth}
\centering  
\includegraphics[scale=0.32]{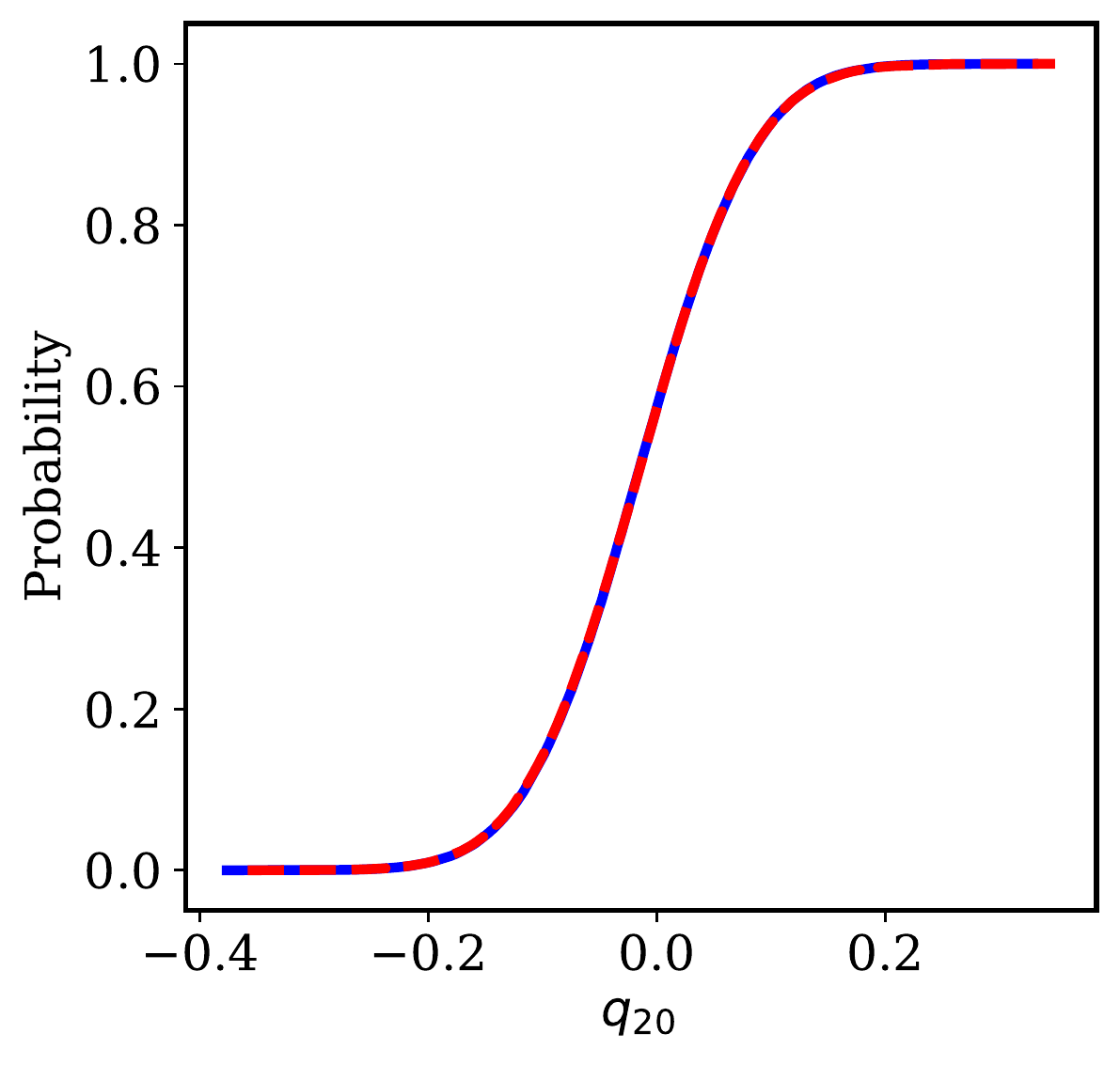} 
\caption{}
\label{24D_Logistic_ecdf_20}
\end{subfigure}
\begin{subfigure}{0.32\textwidth}
\centering  
\includegraphics[scale=0.32]{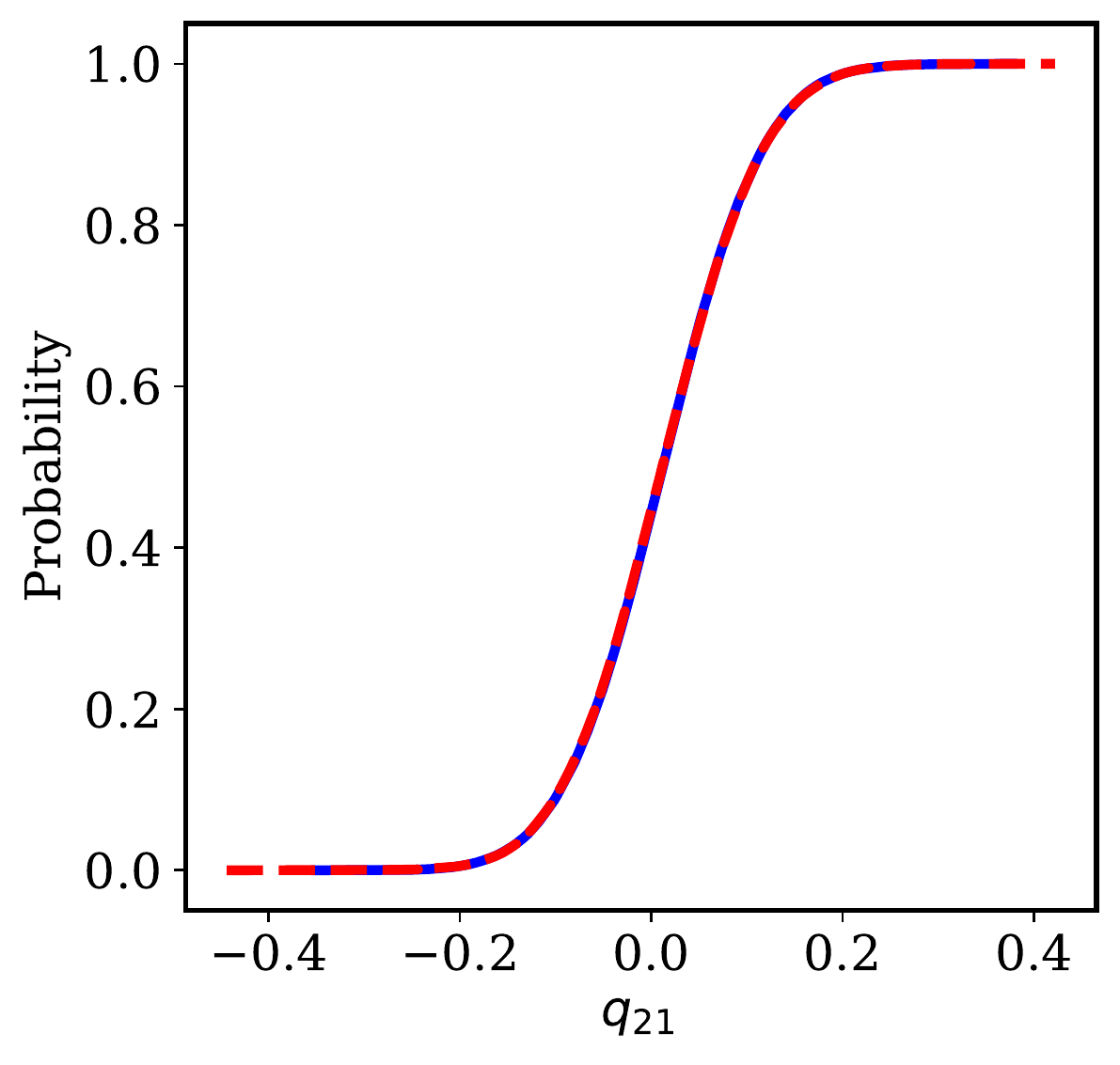} 
\caption{}
\label{24D_Logistic_ecdf_21}
\end{subfigure}
\begin{subfigure}{0.32\textwidth}
\centering  
\includegraphics[scale=0.32]{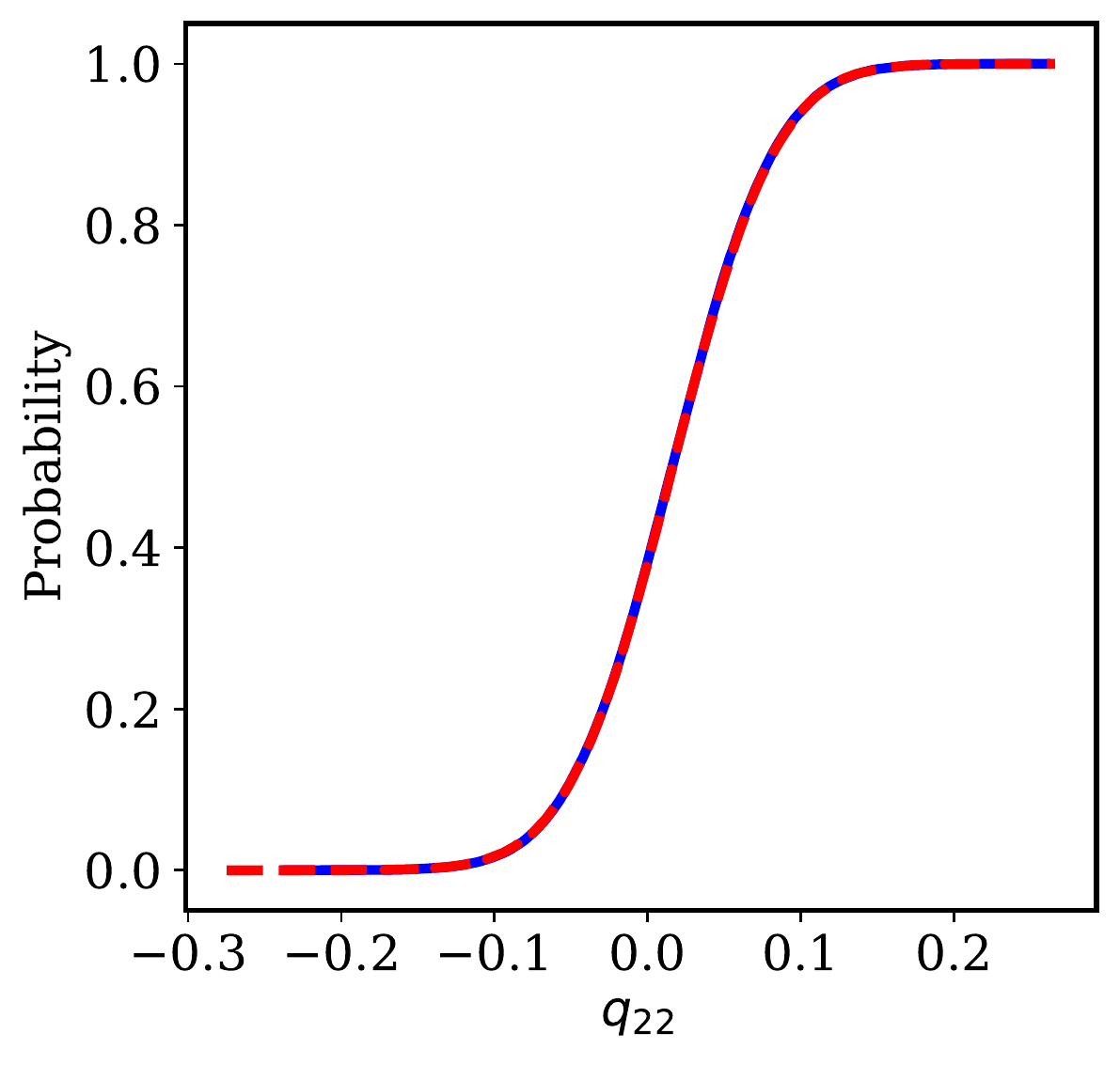} 
\caption{}
\label{24D_Logistic_ecdf_22}
\end{subfigure}
\begin{subfigure}{0.32\textwidth}
\centering  
\includegraphics[scale=0.32]{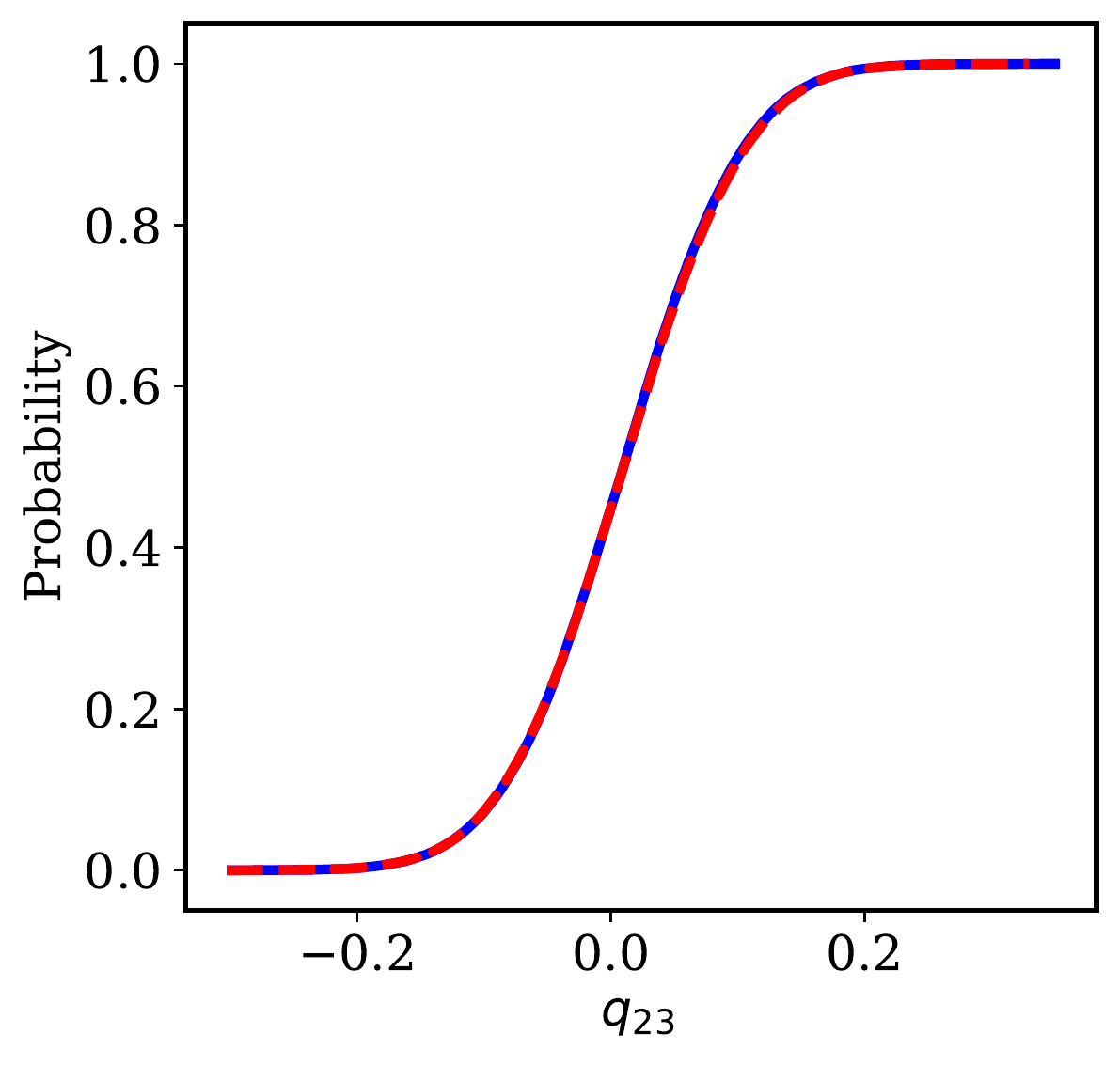} 
\caption{}
\label{24D_Logistic_ecdf_23}
\end{subfigure}
\begin{subfigure}{0.32\textwidth}
\centering  
\includegraphics[scale=0.32]{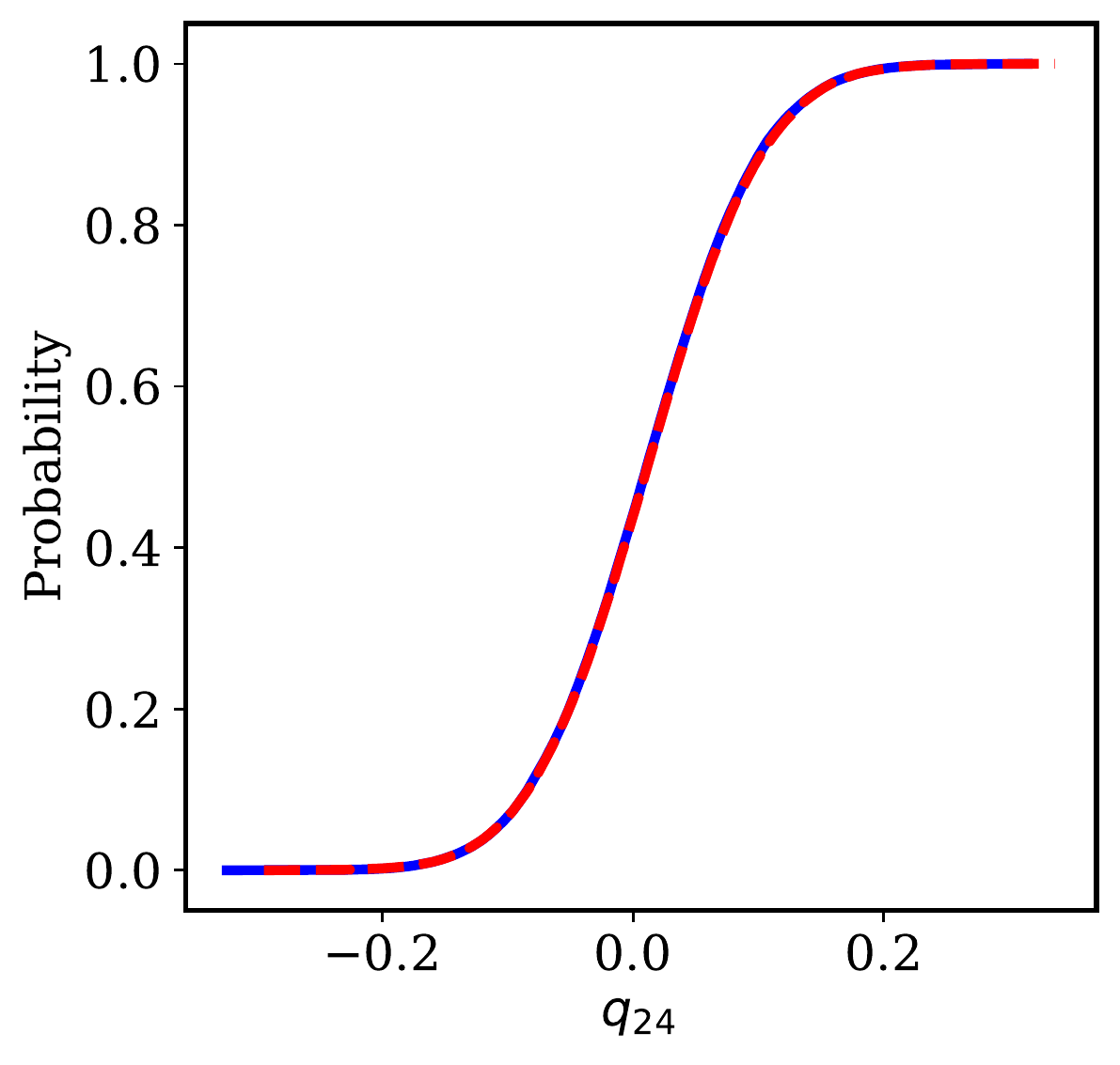} 
\caption{}
\label{24D_Logistic_ecdf_24}
\end{subfigure}
\caption{24D Bayesian logistic regression empirical cumulative distribution functions comparison for dimensions 13 to 24.}
\label{24D_Logistic_ecdf_dim2}
\end{figure}

\begin{figure}[htbp]
\centering
\includegraphics[width=\textwidth]{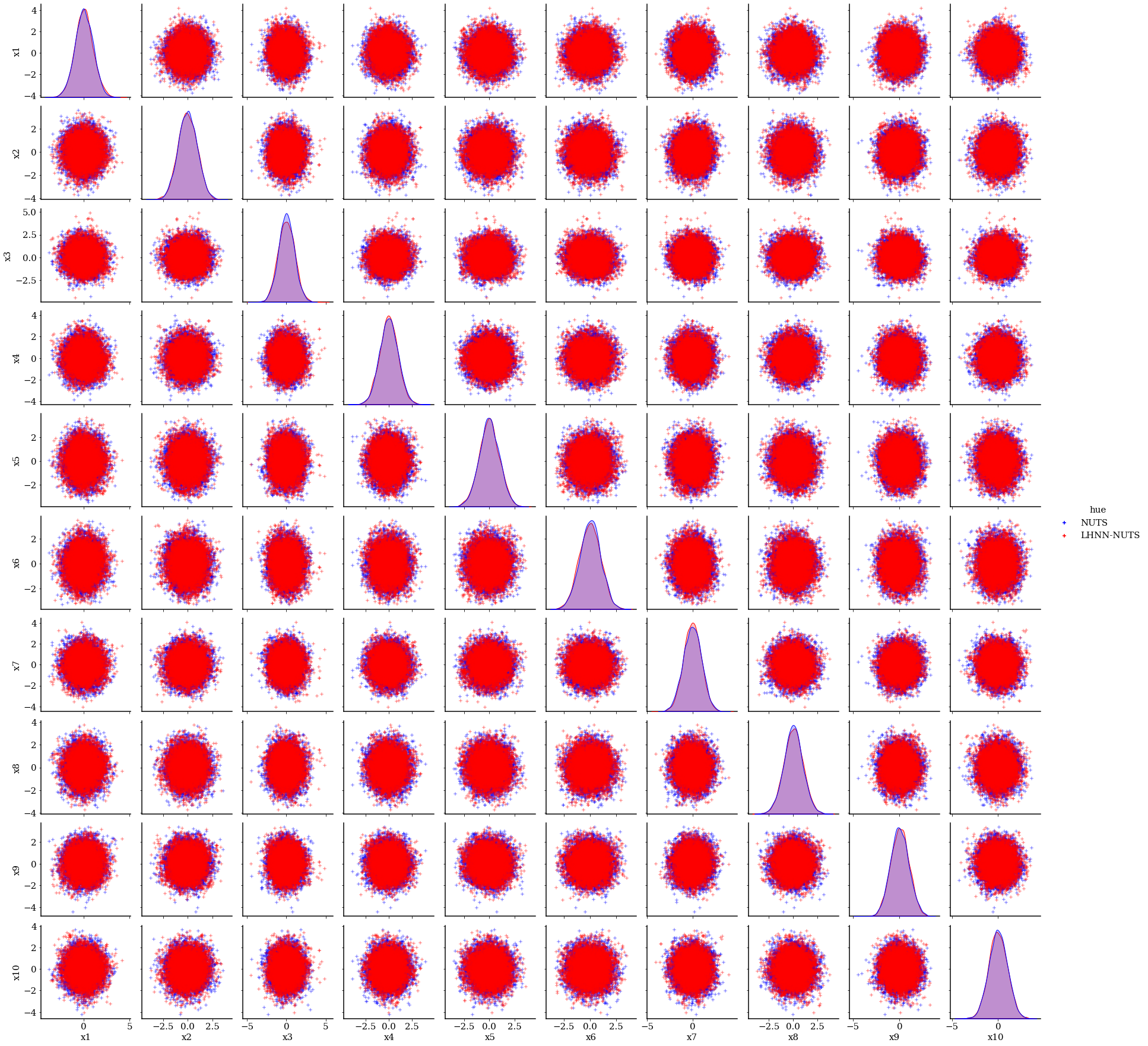} 
\caption{Scatter plot comparison between NUTS and LHNN-NUTS with online error monitoring, considering a 100-D rough well distribution. As an illustration, 10 of the 100 dimensions are shown. (Blue dots: traditional NUTS; Red dots: LHNN-NUTS with online error monitoring)}
\label{100D_RoughWell_highres}
\end{figure}

\begin{figure}[htbp]
\centering
\includegraphics[width=\textwidth]{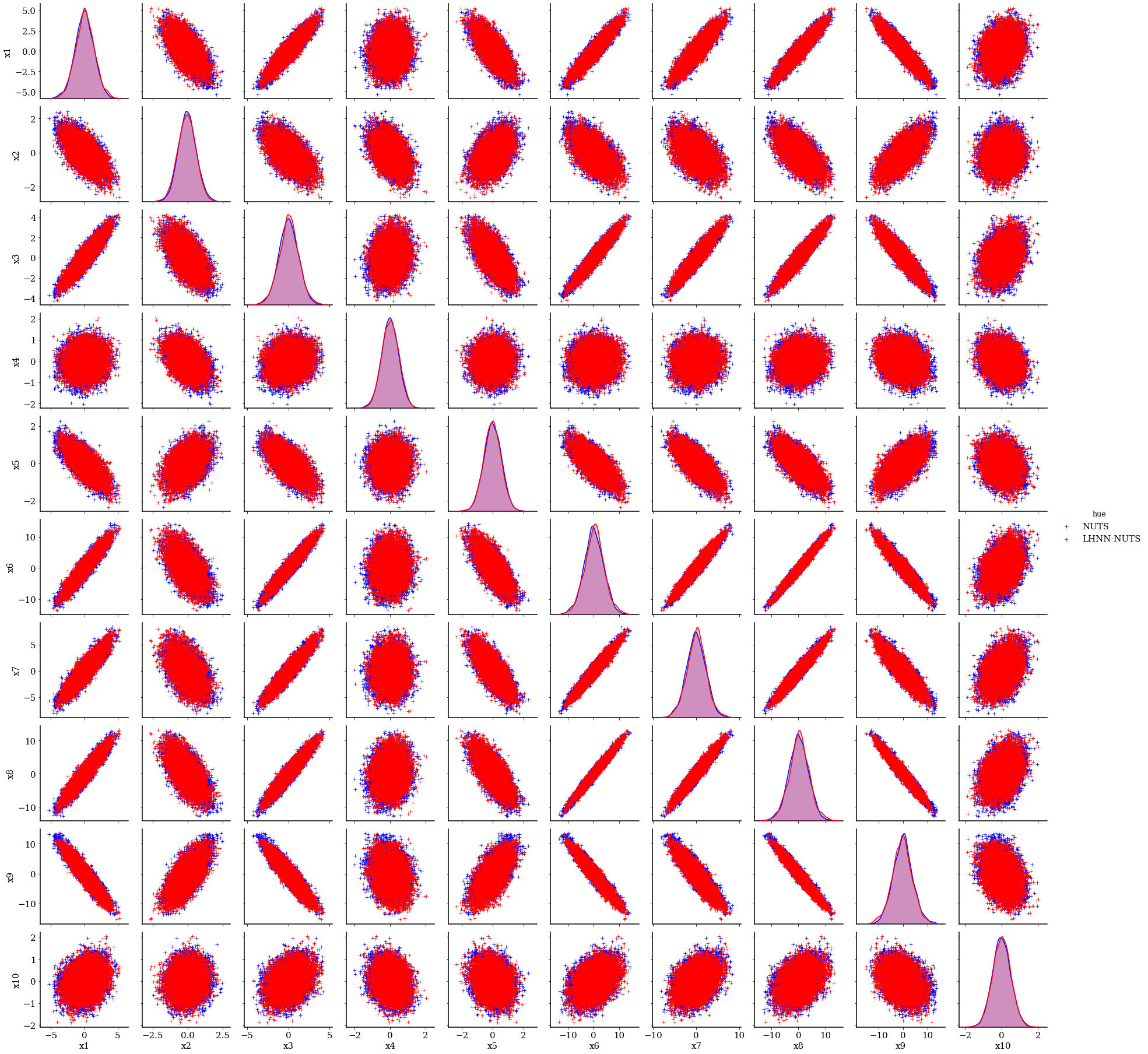} 
\caption{Scatter plot comparison between NUTS and LHNN-NUTS with online error monitoring, considering a 100-D Gaussian distribution with Wishart inverse covariance. As an illustration, 10 of the 100 dimensions are shown. (Blue dots: traditional NUTS; Red dots: LHNN-NUTS with online error monitoring)}
\label{100D_Gauss_highres}
\end{figure}

\end{document}